\documentclass[twoside,11pt]{article}
\usepackage{multirow,tikz,amsmath,comment}
\usepackage{algorithm}
\usepackage{algorithmic}
\usepackage{subfig}
\usepackage{enumitem}
\usepackage{epstopdf}
\usepackage{tabularx,booktabs}
\usepackage{lastpage}
\usepackage{caption}
\usepackage{microtype}
\usepackage{jmlr2e}

\newcommand\independent{\protect\mathpalette{\protect\independenT}{\perp}}
\def\independenT#1#2{\mathrel{\rlap{$#1#2$}\mkern2mu{#1#2}}}

\newcommand{\dependent}{ \not\!\perp\!\!\!\perp}

\newcommand*\samethanks[1][\value{footnote}]{\footnotemark[#1]}

\newtheorem{Theorem}{Theorem}
\newtheorem{Proposition}{Proposition}

\newtheorem{Definition}{Definition}
\newtheorem{Corollary}{Corollary}
\newtheorem{Assumption}{Assumption}

\jmlrheading{21}{2020}{1-\pageref{LastPage}}{3/19; Revised 2/20}{5/20}{19-232}{Biwei Huang, Kun Zhang, Jiji Zhang, Joseph Ramsey, Ruben
Sanchez-Romero, Clark Glymour, and Bernhard Sch{\"o}lkopf}

\ShortHeadings{Causal Discovery from Heterogeneous/Nonstationary Data with Independent Changes}{Huang, Zhang, Zhang, Ramsey,
Sanchez-Romero, Glymour, and Sch{\"o}lkopf}
\firstpageno{1}

\begin{document}
	
	\title{Causal Discovery from Heterogeneous/Nonstationary Data with Independent Changes}
	
	   \author{\name Biwei Huang \thanks{Equal contribution} \email     biweih@andrew.cmu.edu\\ 
        \name Kun Zhang \samethanks \email kunz1@cmu.edu\\ 
	    \addr Department of Philosophy, 
	    Carnegie Mellon University,
	    Pittsburgh, PA 15213
		\AND 
		\name Jiji Zhang \email jijizhang@ln.edu.hk\\ 
		\addr Department of Philosophy,
		Lingnan University, Tuen Mun,
		Hong Kong
		\AND 
		\name Joseph Ramsey \email jdramsey@andrew.cmu.edu \\ 
		\addr Department of Philosophy, 
		Carnegie Mellon University,
		Pittsburgh, PA 15213
	    \AND 
		\name Ruben Sanchez-Romero \email ruben.saro@rutgers.edu \\ 
		\addr Center for Molecular and Behavioral Neuroscience, 
		Rutgers University,
		Newark, NJ 07102
		\AND 
		\name Clark Glymour \email cg09@andrew.cmu.edu\\ 
		\addr Department of Philosophy,
		Carnegie Mellon University, 
		Pittsburgh, PA 15213
		\AND 
		\name Bernhard Sch{\"o}lkopf \email bernhard.schoelkopf@tuebingen.mpg.de\\ 
		\addr Max Planck Institute for Intelligent Systems, 
		T{\"u}bingen 72076,
		Germany
	   }
	
	\editor{Jon McAuliffe}

	\maketitle
	
	\begin{abstract}
        It is commonplace to encounter heterogeneous or nonstationary data, of which the underlying generating process changes across domains or over time. Such a distribution shift feature presents both challenges and opportunities for causal discovery. In this paper, we develop a framework for causal discovery from such data, called Constraint-based causal Discovery from heterogeneous/NOnstationary Data (CD-NOD), to find causal skeleton and directions and estimate the properties of mechanism changes. First, we propose an enhanced constraint-based procedure to detect variables whose local mechanisms change and recover the skeleton of the causal structure over observed variables. Second, we present a method to determine causal orientations by making use of independent changes in the data distribution implied by the underlying causal model, benefiting from information carried by changing distributions. After learning the causal structure, next, we investigate how to efficiently estimate the ``driving force'' of the nonstationarity of a causal mechanism. That is, we aim to extract from data a low-dimensional representation of changes. The proposed methods are nonparametric, with no hard restrictions on data distributions and causal mechanisms, and do not rely on window segmentation. Furthermore, we find that data heterogeneity benefits causal structure identification even with particular types of confounders. Finally, we show the connection between heterogeneity/nonstationarity and soft intervention in causal discovery. Experimental results on various synthetic and real-world data sets (task-fMRI and stock market data) are presented to demonstrate the efficacy of the proposed methods.
	\end{abstract}	
	
	\section{Introduction} \label{Sec: Intro}
	   Many tasks across several disciplines of empirical sciences and engineering rely on the underlying causal information. As it is often difficult, if not impossible, to carry out randomized experiments, inferring causal relations from purely observational data, known as the task of causal discovery, has drawn much attention in machine learning, philosophy, statistics, and computer science. Traditionally, for causal discovery from observational data, under appropriate assumptions, so-called constraint-based approaches recover some information of the underlying causal structure based on conditional independence relationships of the variables~\citep{SGS93}. Alternatively, approaches based on functional causal models infer the causal structure by exploiting the fact that under certain conditions, the independence between the noise and the hypothetical cause only holds for the correct causal direction but not for the wrong direction~\citep{Shimizu06,Hoyer09,Zhang09_additive,Zhang_UAI09}.
	   
	   Over the last few years, with the rapid accumulation of huge volumes of data of various types, causal discovery is facing exciting opportunities but also great challenges. 
	   One feature such data often exhibit is distribution shift. Distribution shift may occur across data sets, which may be obtained under different interventions or with different data collection conditions, or over time, as featured by nonstationary data. For an example of the former kind,  consider remote sensing imagery data. 
	   The data collected in different areas and at different times usually have different distributions due to varying physical factors related to ground, vegetation, illumination conditions, etc. As an example of the latter kind, fMRI recordings are usually nonstationary: information flows in the brain may change with stimuli, tasks, attention of the subject, etc.  More specifically, it is believed that one of the basic properties of neural connections  is their time-dependence \citep{DynamicC1}. In these situations many existing approaches to causal discovery may fail, as they assume a fixed causal model and hence a fixed joint distribution underlying the observed data. For example, if changes in local mechanisms of some variables are related, one can model the situation as if there exists some unobserved quantity which influences all those variables and, as a consequence, the conditional independence relationships in the distribution-shifted data will be different from those implied by the true causal structure.
	   
	   In this paper, we assume that mechanisms or parameters, associated with the causal model, may change across data sets or over time (we allow mechanisms to change in such a way that some causal links in the structure may vanish or appear in some domains or over some time periods). We aim to develop a principled framework to model such situations as well as practical methods, called Constraint-based causal Discovery from heterogeneous/NOnstationary Data (CD-NOD), to address the following questions:
	   
	   \begin{itemize}[itemsep=0.2pt,topsep=0.2pt]
	   	\item [1.] {How can we efficiently identify variables with changing local mechanisms and reliably recover the skeleton of the causal structure over observed variables?}
	   	
	   	\item [2.] {How can we take advantage of the information carried by distribution shifts for the purpose of identifying causal directions?}
	   \end{itemize}
      
       After identifying the causal structure, it is then appealing to ask how causal mechanisms change across domains or over time, which raises the question:
       	\begin{itemize}[itemsep=0.2pt,topsep=0.2pt]
       	\item [3.] {How can we extract from data a low-dimensional and potentially interpretable representation of changes, the so-called ``driving force" of changing causal mechanisms?}
       \end{itemize}  
               
        Furthermore, we extend our approach to deal with more general scenarios, e.g., dynamic systems which involve both time-varying instantaneous and lagged causal relations and the case of stationary confounders.
        
      In answering these questions, we make use of the following properties of causal systems. (i) Causal models and distribution shifts are heavily coupled: causal models provide a compact description of how data-generating processes, as well as data distributions, change, and distribution shifts exhibit such changes. (ii) From a causal perspective, the distribution shift in heterogeneous/nonstationary data is usually constrained---it may be due to the changes in the data-generating processes (i.e., the local causal mechanisms) of {\it a small number} of variables. (iii) From a latent variable modeling perspective, heterogeneous/nonstationary data are generated by some quantities that change across domains or over time, which gives hints as to how to understand distribution shift and estimate its underlying driving forces. (\textcolor{black}{iv}) Suppose that there are no confounders for $\texttt{cause}$ and $\texttt{effect}$. Then $P(\texttt{cause})$ and $P(\texttt{effect}\,|\,\texttt{cause})$ are either fixed or change independently, also known as the modularity property of causal systems \citep{Pearl00}. Such an independence helps identify causal directions in the presence of distribution shifts.
        
       To reliably estimate the skeleton of the causal structure and detect changing causal mechanisms from heterogeneous/nonstationary data (Problem 1), we make use of property (i), (ii), and (iii) listed above. Specifically, we introduce a surrogate variable $C$ into the causal system to characterize hidden quantities that lead to the changes across domains or over time. The variable $C$ can be a domain or time index. Including $C$ in the causal system provides a convenient way to unpack distribution shifts to causal representations.
       We show that given $C$, (conditional) independence relationships between observed variables are the same as those implied by the true causal structure. We, additionally, show that variables that are adjacent to the surrogate variable $C$ have changing causal mechanisms. We make the assumption of faithfulness on the graph involving  $C$, and it is known that faithfulness implies a minimality condition on the edges \citep{Jiji_Faithfulness16}; as a consequence, the graphical representation produced by our procedure naturally {\it enjoys a ``minimal change" principle}---the representation explains the conditional independence relations and changeability of the distribution with a minimal number of changing conditional distributions.
       
       Regarding Problem 2, as a sub-problem of causal discovery, we show that distribution shift provides additional information for causal direction identification. it is known that with functional causal model-based approaches, there are cases where causal directions are not identifiable, e.g., the linear-Gaussian case and the case with a general functional class \citep{Hyvarinen99,Zhang15_TIST}. This restricts the causal direction identification to certain functional classes, e.g., additive \citep{Shimizu06,Hoyer09,Zhang09_additive} or post-nonlinear models \citep{Zhang06_iconip, Zhang_UAI09}. We show that using information carried by distribution shifts does not suffer from these restrictions---the method applies to general causal mechanisms. 
       
       Specifically, we take advantage of  property (iv) for causal direction determination: if there is no confounder for $V_i$ and $V_j$, then the causal mechanisms, represented by the conditional distributions $P(V_i\,|\,\mathrm{PA}^i)$ and $P(V_j\,|\,\mathrm{PA}^j)$, change independently across data sets or over time. However, independence typically no longer holds for wrong directions. This gives rise to causal asymmetry. To exploit this asymmetry, we develop a kernel embedding of nonstationary conditional distributions to represent changing causal mechanisms and accordingly propose a dependence measure to determine causal directions.  Furthermore, it is worth noting that although our method can been seen as an extension of constraint-based methods such as PC \citep{Spirtes00}, unlike the original ones,  it can {\it find the causal direction between even two variables}, thanks to the surrogate variable $C$ in the system or more generally, the independent change property implied by a causal system.
       
	   Regarding Problem 3, traditionally, one may use Bayesian change point detection to detect change points of observed time series \citep{Adam07} or use sliding windows and then estimate the causal model within each segment separately. However, Bayesian change point detection can only be applied to detect changes in marginal or joint distributions, whereas causal mechanisms are represented by conditional distributions. Moreover, neither of them is appropriate when causal mechanisms change continuously over time. More recently, a window-free method has been proposed, by extending Gaussian process regression \citep{Huang15}. However, it requires the assumption of linearity, and it fails to handle the case when nonstationarity results from the change of noise distributions. In this paper, by leveraging property (iii), we propose a nonparametric method to {\it recover a low-dimensional and interpretable representation of mechanism changes}, which does not rely on window segmentation.	  	   

	   This paper is organized as follows.\footnote{This paper is built on the arXiv paper by \citet{CDNOD_arxiv} and the conference papers by \citet{Zhang17_IJCAI} and \citet{Huang17_ICDM} but is significantly extended in several aspects. We reformulate assumptions in Section \ref{Sec:assumptions}.  We extend Section \ref{Sec: skeleton} to show how we detect pseudo confounders for nonadjacent variables. In Section \ref{Sec: generalization_invar}, we add Algorithm \ref{Alg: invariance} which uses generalization of invariance to identify causal directions. In Section \ref{Sec:general}, we propose a new approach to efficiently identify causal directions using independent changes between causal mechanisms and detect pseudo confounders behind adjacent variables (Algorithm \ref{Alg: direction}). In Section \ref{Sec: identifiablity}, we give identifiability conditions of CD-NOD and define the equivalence class that CD-NOD can achieve if those conditions do not hold. Furthermore, we extend CD-NOD to the case when there exist both time-varying instantaneous and lagged causal relationships (Section \ref{Sec: Lagged}). Accordingly, we propose Algorithm \ref{rs_causal_discovery_lagged} to efficiently recover both instantaneous and time-lagged causal relationships. In Section \ref{Sec: confounder}, we further discuss whether distribution shifts also help for causal discovery when there exist stationary confounders. With CD-NOD, some causal directions may not be identifiable, if the identifiability conditions are not satisfied (Theorem \ref{Theorem: identifiability}). To make the method more applicable, we combine our framework with approaches based on constrained functional causal models (Section \ref{Sec: FCM}).  In Section \ref{Sec: Soft_Inter}, we show that heterogeneity/nonstationarity and soft intervention are related in causal discovery, and we find that our proposed method is even more effective.} In Section~\ref{Sec:problem} we define and motivate the problem in more detail and review related work.  Section~\ref{Sec:method} proposes an enhanced constraint-based method to recover the causal skeleton over observed variables and identify variables with changing causal mechanisms.  Section~\ref{Sec:unify} develops a method for determining causal directions by exploiting distribution shifts. It makes use of the property that in a causal system, causal modules change independently if there are no confounders. Section \ref{Sec: driving force} proposes a method, termed Kernel Nonstationary Visualization (KNV), to visualize a low-dimensional and interpretable representation of changing mechanisms, the so-called ``driving force". In Section \ref{Sec: extension}, we extend CD-NOD to the case that allows both time-varying lagged and instantaneous causal relationships, and we discuss whether distribution shifts also help for causal discovery when there exist stationary confounders. In addition, we give a procedure to leverage both CD-NOD and approaches based on constrained functional causal models. In Section \ref{Sec: Soft_Inter}, we show the connection between heterogeneity/nonstationarity and soft intervention in causal discovery. Section \ref{Sec:experiments} reports experimental results tested on both synthetic and real-world data sets, including task-fMRI data, Hong Kong stock data, and US stock data.
	   
	   \section{Problem Definition and Related Work}\label{Sec:problem}
	   In this section, we first review causal discovery approaches with fixed causal models. Then we give examples to show that if the underlying causal model changes, directly applying approaches with fixed causal models may result in spurious edges or wrong causal directions, which motivates our work in causal discovery with changing causal models.
	   
	   \subsection{Causal Discovery of Fixed Causal Models}
	   Most causal discovery methods assume that there is a fixed causal model underlying the observed data and aim to estimate it from the data. Classic approaches to causal discovery divide roughly into two types.  In the late 1980's and early 1990's, it was noted that under appropriate assumptions, one could recover a Markov equivalence class of the underlying causal structure based on conditional independence relationships among the variables~\citep{SGS93}. This gave rise to the constraint-based approach to causal discovery, and the resulting equivalence class may contain multiple DAGs (or other graphical objects to represent causal structures), which entail the same conditional independence relationships.  The required assumptions include the causal Markov condition and the faithfulness assumption, which entail a correspondence between d-separation properties in the underlying causal structure and statistical independence properties in the data.  The so-called score-based approach \citep[see, e.g.,][]{Chickering02,Heckerman95} searches for the equivalence class which gives the highest score under some scoring criteria, such as the Bayesian Information Criterion (BIC), the posterior of the graph given the data \citep{Heckerman97}, and the generalized score functions \citep{Huang18_KDD}.
	   
	   Another set of approaches is based on constrained functional causal models, which represent the effect as a function of the direct causes together with an independent noise term. The causal direction implied by the constrained functional causal model is generically identifiable, in that the model assumptions, such as the independence between the noise and cause, hold only for the true causal direction and are violated for the wrong direction. Examples of such constrained functional causal models include the linear non-Gaussian acyclic model \citep[LiNGAM,][]{Shimizu06}, the additive noise model~\citep{Hoyer09,Zhang09_additive}, and the post-nonlinear causal model~\citep{Zhang06_iconip, Zhang_UAI09}.
	   
	   \subsection{With Changing Causal Models}
	   Suppose that we are given a set of observed variables $\mathbf{V} = \{V_i\}_{i=1}^{m}$ whose causal structure is represented by a DAG $G$. For each $V_i$, let $PA^i$ denote the set of parents of $V_i$ in $G$. Suppose that at each time point or in each domain, the joint probability distribution of $\mathbf{V}$ factorizes according to $G$:
	   \begin{equation} \label{Eq:decomp}
	   P(\mathbf{V}) = \prod_{i=1}^m P(V_i \,|\, \mathrm{PA}^i).
	   \end{equation}
	   We call each $P(V_i\,|\,\mathrm{PA}^i)$ a {\it causal module} (the same meaning with ``causal mechanism" in previous sections). If there are distribution shifts (i.e., $P(\mathbf{V})$ changes across domains or over time), at least some causal modules $P(V_k\,|\,\mathrm{PA}^k)$, $k\in \mathcal{N}$ must change. We call those causal modules {\it changing causal modules}. Their changes may be due to changes of the involved functional models, causal strengths, noise levels, etc. We assume that those quantities that change across domains or over time can be written as functions of a domain or time index, and denote by $C$ such an index.

\begin{figure} 
	\begin{center}
		\begin{center}
			\begin{tikzpicture}[scale=.6, line width=0.5pt, inner sep=0.2mm, shorten >=.1pt, shorten <=.1pt]
			\draw (0, 0) node(1) [circle, draw] {{\footnotesize\,$V_1$\,}};
			\draw (1.8, 0) node(2) [circle, draw] {{\footnotesize\,$V_2$\,}};
			\draw (3.6, 0) node(3) [circle, draw] {{\footnotesize\,$V_3$\,}};
			\draw (5.4, 0) node(4) [circle, draw] {{\footnotesize\,$V_4$\,}};
			\draw (3.6, 1.3) node(5) {{\footnotesize\,{\small$g(C)$}\,}};
			\draw[-latex] (1) -- (2); 
			\draw[-latex] (2) -- (3); 
			\draw[-latex] (3) -- (4); 
			\draw[-latex] (5) -- (2); 
			\draw[-latex] (5) -- (4); 
			\end{tikzpicture} ~~~~~~~~~~~~~~~~~~~~~~
			\begin{tikzpicture}[scale=.6, line width=0.5pt, inner sep=0.2mm, shorten >=.1pt, shorten <=.1pt]
			\draw (0, 0) node(1) [circle, draw] {{\footnotesize\,$V_1$\,}};
			\draw (1.8, 0) node(2) [circle, draw] {{\footnotesize\,$V_2$\,}};
			\draw (3.6, 0) node(3) [circle, draw] {{\footnotesize\,$V_3$\,}};
			\draw (5.4, 0) node(4) [circle, draw] {{\footnotesize\,$V_4$\,}};
			\draw[-] (1) -- (2); 
			\draw[-] (2) -- (3); 
			\draw[-] (3) -- (4); 
			\draw[-] (1) to[out=30,in=150] (4); 
			\draw[-] (2) to[out=-30,in=-150] (4); 
			\end{tikzpicture} \\
			~~(a) ~~~~~~~~~~~~~~~~~~~~~~~~~~~~~~~~~~~~~~~~~~~~~~~~~~(b)
		\end{center}
		\caption{An illustration on how ignoring changes in the causal model may lead to spurious edges by constraint-based methods. (a) The true causal graph (including confounder $g(C)$, which is hidden). (b) The estimated causal skeleton on the observed data in the asymptotic case given by constraint-based methods, e.g., PC or FCI. }
		\label{fig:illust_C}
	\end{center}
\end{figure}
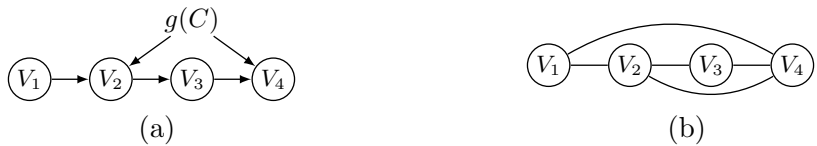

	   If the changes in some modules are related, one can treat the situation as if there exists some unobserved quantity (confounder) which influences those modules and, as a consequence, the conditional independence relationships in the distribution-shifted data will be different from those implied by the true causal structure. Therefore, standard constraint-based algorithms such as PC or FCI~\citep{SGS93} may not be able to reveal the true causal structure.  
	   As an illustration, suppose that the observed data were generated according to Fig.~\ref{fig:illust_C}(a), where $g(C)$, a function of $C$, is involved in the generating processes for both $V_2$ and $V_4$;  the causal skeleton over the observed data then contains spurious edges $V_1 - V_4$ and $V_2 - V_4$,  as shown in Fig.~\ref{fig:illust_C}(b), because there is only one conditional independence relationship, $V_3 \independent V_1 \,|\, V_2$. 
	   
	   \begin{figure*}[htp]
	   	\includegraphics[width = 1.0\textwidth]{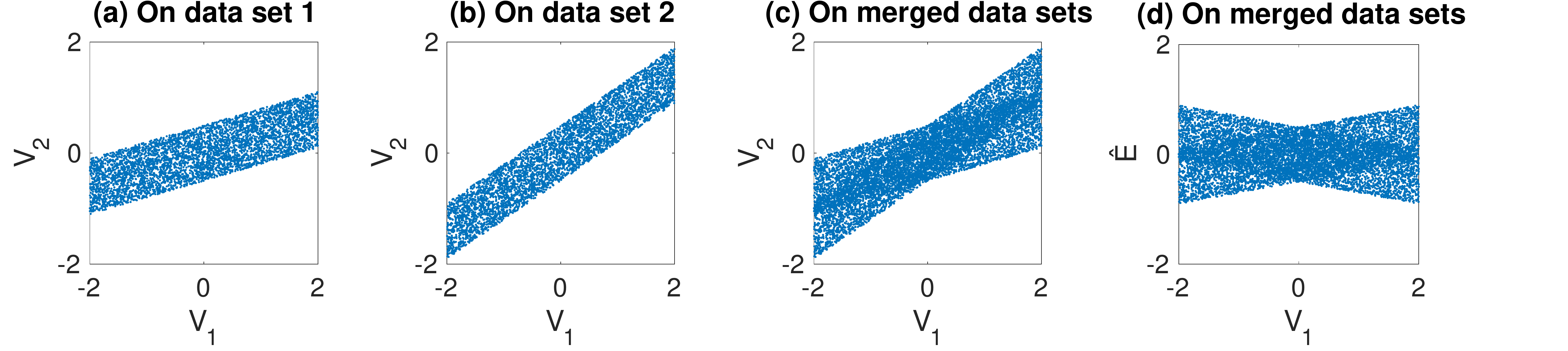} 
	   	\caption{An illustration of a failure of using the approach based on functional causal models for causal direction determination when the causal model changes. (a) Scatter plot of $V_1$ and $V_2$ on data set 1. (b) \textcolor{black}{Scatter plot of $V_1$ and $V_2$} on data set 2. (c)  \textcolor{black}{Scatter plot of $V_1$ and $V_2$} on merged data (both data sets). (d) Scatter plot of $V_1$ and the estimated regression residual $\hat{E}$ on merged data by regressing $V_2$ on $V_1$.}
	   	\label{fig:illust_SEM}
	   \end{figure*}
	   
	   Moreover, when one fits a fixed functional causal model ~\citep[e.g., a linear, non-Gaussian model,][]{Shimizu06} to distribution-shifted data, the estimated noise may not be independent of the cause. Consequently, the approach based on constrained functional causal models, in general, cannot infer the correct causal structure either.  Figure~\ref{fig:illust_SEM} gives an illustration of this point. Suppose that we have two data sets for variables $V_1$ and $V_2$: $V_2$ is generated from $V_1$ according to $V_2 = 0.3V_1 + E$ in the first data set and according to $V_2 = 0.7V_1 + E$ in the second one, and in both data sets $V_1$ and $E$ are mutually independent and follow a uniform distribution. Figure~\ref{fig:illust_SEM}(a-c) show the scatter plots of $V_1$ and $V_2$ on data set 1, on data set 2, and on merged data, respectively. Figure~\ref{fig:illust_SEM}(d) shows the scatter plot of $V_1$, the cause, and the estimated regression residual on the merged data set by regressing $V_2$ on $V_1$; they are not independent anymore, although on either data set the regression residual is independent of $V_1$. Thus, we cannot correctly determine the causal direction.
	   
	 To tackle the issue of changing causal models, one may try to find causal models in each domain, for data from multiple domains, or in each sliding window, for nonstationary data, separately, and then compare and merge them.  \textcolor{black}{For instance, regarding the former type of data from multiple domains, in particular, multiple data sets obtained by external interventions where it is unknown what variables are manipulated, \cite{Geng_16} considered two settings, depending on the sample size of each data set--in the setting with a  large sample size for each data set, they proposed a graph-merging method after learning a causal network in each domain separately; in the setting with a relative small sample size, they proposed to pool together the data to learn a network structure and then use a re-sampling approach to evaluate the edges of the learned network.} 
	 Regarding nonstationary data, improved versions include the online change point detection method~\citep{Adam07}, the online undirected graph learning~\citep{Talih05}, and the locally stationary structure tracker algorithm~\citep{Kummerfeld13}.  Such methods may suffer from high estimation variance due to sample scarcity, large type II errors, or multiple testing problems from a large number of statistical tests. Some methods aim to estimate the time-varying causal model by making use of certain types of smoothness of the change~\citep{Huang15}, but they do not explicitly locate the nonstationary causal modules. Several methods aim to model time-varying time-delayed causal relations~\citep{Xing10,Song09_DBN}, which can be reduced to online parameter learning because the direction of causal relations is given (i.e., the past influences the future). Compared to them, learning changing instantaneous causal relations, with which we are concerned in this paper, is generally more difficult. Recently, several methods have been proposed to tackle time-varying or domain-varying instantaneous causal relations \citep{Ghassami18_NIPS, SSM_Huang19, Huang19_NIPS, Huang20_AAAI}. However, they assume linear causal models, limiting their applicability to complex problems with nonlinear causal relations.
	   
	 In contrast, we develop a nonparametric and computationally efficient method that can identify changing causal modules and reliably recover the causal structure. We show that distribution shifts actually contain useful information for the purpose of determining causal directions and develop practical algorithms accordingly. After identifying the causal structure, we propose a method to estimate a low-dimensional and interpretable representation of changes.
	   
	   \section{CD-NOD Phase I: Changing Causal Module Detection and Causal Skeleton Estimation} \label{Sec:method}
	   In this section, we first formalize the assumptions that will be used in CD-NOD. Specifically, we allow a particular type of confounders---pseudo confounders, and we do not put hard restrictions on functional forms of causal mechanisms and data distributions. Accordingly, we propose an approach to efficiently detect changing causal modules and identify the causal skeleton; we call this step CD-NOD phase I. We show that the proposed approach is guaranteed to asymptotically recover the true graph as if unobserved changing factors were known.
	   
	   \subsection{Assumptions} \label{Sec:assumptions}
	   In this paper, we allow changes in causal modules and some of the changes to be related; the related changes can be explained by positing particular types of confounders. Intuitively, such confounders may refer to some high-level background variables. For instance, for fMRI data, they may be the subject's attention or some unmeasured background stimuli; for the stock market, they may be related to economic policies. Thus, we do not assume causal sufficiency for the set of observed variables.  Instead, we assume {\it pseudo causal sufficiency} as stated below. 
	   
	   \begin{Assumption}[Pseudo Causal Sufficiency]
	   	  We assume that the confounders, if any, can be written as functions of the domain index or smooth functions of time \footnote{More specifically, for data with multiple domains, we require that the confounders can be written as a function of the domain index (i.e., it does not change within a domain); for nonstationary time series, we require that the confounder is a smooth function of the time index. Roughly speaking, the smoothness constraint requires the gradient of the function to not change rapidly. In practice, one may specify the level of smoothness in advance (say, by assuming the function follows a Gaussian process prior and properly setting the kernel width to some range) or learn it from data by maximizing marginal likelihood or cross validation.}. It follows that in each domain or at each time instance, the values of these confounders are fixed. 
	   	  \label{Ass: pseudo sufficiency}
	   \end{Assumption}
	   
	   To clearly express our basic idea in the presence of distribution shift, we focus on DAGs and assume {\it pseudo causal sufficiency}. Note that our approach is flexible enough to be extended to cover other types of graphs, e.g., graphs with confounders and graphs with cycles. 
	   Later in Section \ref{Sec: confounder}, we will discuss how nonstationarity helps when there exist stationary confounders. In table \ref{Table: confounder}, we summarize descriptions of different types of confounders (latent common causes) that will be used in this paper, including pseudo confounders, stationary confounders, and nonstationary confounders.
	   
	   \begin{table}
	   	\begin{tabular}{ll}
	   		\hline
	   		\textbf{Confounder type}          & \textbf{Description}                                                                                                                                    \\ \hline
	   				\hline
	   		Pseudo confounder        & \begin{tabular}[c]{@{}l@{}} It can be represented as functions of domain index or \\ smooth functions of time index.    \end{tabular}                                                                                                 \\ \hline
	   		Stationary confounder    & \begin{tabular}[c]{@{}l@{}}Its distribution is fixed and it cannot be represented as \\functions of domain index or smooth functions of time.                                    \end{tabular}                                                                                                      \\ \hline
	   		Nonstationary confounder & \begin{tabular}[c]{@{}l@{}}Its distribution changes across domains or over time and \\it cannot be represented as functions of domain index or \\smooth functions of time index, respectively.\end{tabular} \\ \hline
	   	\end{tabular}
   	   	\caption{Descriptions of different types of confounders (latent common causes). }
        \label{Table: confounder}
	   \end{table}
	   
	   We start with contemporaneous causal relations; the mechanisms and parameters associated with the causal model are allowed to change across data sets or over time, or even vanish or appear in some domains or over some time periods. However, it is natural to generalize our framework to incorporate time-delayed causal relations (Section \ref{Sec: Lagged}).
	   
	   Denote by $\{g_l(C)\}_{l=1}^L$ the set of pseudo confounders (which may be empty). We further assume that for each $V_i$, its local causal process can be represented by the following structural equation model (SEM):
	   \begin{equation} \label{Eq:FCM}
	   V_i = f_i\big(\mathrm{PA}^i, \mathbf{g}^i(C), \theta_i(C), \epsilon_i\big),
	   \end{equation}
	   where $\mathbf{g}^i(C)\subseteq \{g_l(C)\}_{l=1}^L$ denotes the set of confounders that influence $V_i$ (it is an empty set if there is no confounder behind $V_i$ and any other
	   variable), $\theta_i(C)$ denotes the effective parameters in the model that are also assumed to be functions of $C$, and $\epsilon_i$ is a disturbance term that is independent of $C$ and $\mathrm{PA}^i$ and has a non-zero variance (i.e., the model is not deterministic). We also assume that the $\epsilon_i$'s are mutually independent. Note that $\theta_i(C)$ is specific to $V_i$ and is independent of $\theta_j(C)$ for $i\neq j$. The variable $C$ can be the domain or time index. In special cases, e.g., the case with multiple domains and all of which have nonstationarity, $C$ has two dimensions: one is the domain index and the other is the time index. The SEM given in Eq. \ref{Eq:FCM} does not have any restrictions on data distributions or functional classes. 
	   
	   In this paper we treat $C$ as a random variable, and so there is a joint distribution over $\mathbf{V}\cup\{g_l(C)\}_{l=1}^L\cup\{\theta_i(C)\}_{i=1}^m$. We assume that this distribution is Markov and faithful to the graph resulting from the following additions to $G$ (which, recall, is the causal structure over $\mathbf{V}$): add $\{g_l(C)\}_{l=1}^L\cup\{\theta_i(C)\}_{i=1}^m$ to $G$, and for each $i$, add an arrow from each variable in $\mathbf{g}^i(C)$ to $V_i$ and add an arrow from $\theta_i(C)$ to $V_i$. We  refer to this \textit{augmented graph} as $G^{aug}$. Obviously, $G$ is simply the induced subgraph of $G^{aug}$ over $\mathbf{V}$. 
	   Specifically, the assumption is summarized below.
	   \begin{Assumption}
	   	 The joint distribution over $\mathbf{V}\cup\{g_l(C)\}_{l=1}^L\cup\{\theta_i(C)\}_{i=1}^m$ is Markov and faithful to the augmented graph $G^{aug}$. In addition, there is no selection bias; i.e., the observed data are perfect random samples from the populations implied by the causal model. 
	   	 \label{Ass: bias}
	   \end{Assumption}
	   
	   The distribution change across domains or over time can be considered in the following way.
	   In the case when $C$ is the domain index, $C$ follows a uniform distribution over all possible values, and we have a particular way to generate its value: all possible values are generated once, resulting in domain indices. 
	   In the case when $C$ is the time index, we take time to be a special random variable which follows a uniform distribution over the considered time period, with the corresponding data points evenly sampled at a certain sampling frequency. 
	   Correspondingly, the generating process of nonstationary data can be considered as follows: we generate random values from $C$, and then we generate data points over $\mathbf{V}$ according to the SEM in (\ref{Eq:FCM}). The generated data points are then sorted in ascending order according to the values of $C$. In other words, we observe the distribution $P(\mathbf{V}|C)$, where $P(\mathbf{V}|C)$ may change across different values of $C$, resulting in non-identical distributions of data.

       \subsection{Detection of Changing Modules and Recovery of Causal Skeleton} \label{Sec: skeleton}
       In this section, we propose a method to detect variables whose causal modules change and infer the skeleton of $G$. The basic idea is simple: we use the (observed) variable $C$ as a surrogate for the unobserved $\{g_l(C)\}_{l=1}^L\cup\{\theta_i(C)\}_{i=1}^m$, or in other words, we take $C$ to capture \textit{C-specific} information. We now show that given the assumptions in Section \ref{Sec:assumptions}, we can apply conditional independence tests to $\mathbf{V}\cup C$ to detect variables with changing modules and recover the skeleton of $G$. We consider $C$ as a surrogate variable (it itself is not a causal variable, it is always available, and confounders and changing parameters are its functions): by adding only $C$ to the variable set $\mathbf{V}$,  the skeleton of $G$ and the changing causal modules can be {estimated}  {\it as if}
       $\{g_l(C)\}_{l=1}^L\cup\{\theta_i(C)\}_{i=1}^m$ were known. This is achieved by Algorithm \ref{rs_causal_discovery} and supported by Theorem 1.
       
       \begin{algorithm}
       	\caption{Detection of Changing  Modules and Recovery of Causal Skeleton}
       	\begin{enumerate}
       		\item Build a complete undirected graph $U_\mathcal{G}$ on the variable set $ \mathbf{V} \cup C$.
       		\item (\textit{Detection of changing  modules}) For each $i$, test for the marginal and conditional independence between $V_{i}$ and $C$. If they are independent given a subset of $\{V_k\,|\,k\neq i\}$, remove the edge between $V_{i}$ and $C$ in $ U_\mathcal{G}$.
       		\item (\textit{Recovery of causal skeleton}) For every $i\neq j$, test for the marginal and conditional independence between $V_{i}$ and $V_{j}$. If they are independent given a subset of $\{V_k\,|\,k\neq i, k\neq j\}\cup \{ C\}$, remove the edge between $V_{i}$ and $V_{j}$ in $ U_\mathcal{G}$. 
       	\end{enumerate}
       	\label{rs_causal_discovery}
       \end{algorithm}

       The procedure given in Algorithm \ref{rs_causal_discovery} outputs an undirected graph $U_\mathcal{G}$ that contains $C$ as well as $\mathbf{V}$. In Step 2, whether a variable $V_i$ has a changing module is decided by whether $V_i$ and $C$ are independent conditional on some subset of other variables. The justification for one side of this decision is trivial. If $V_i$'s module does not change, that means $P(V_i \,|\, \mathrm{PA}^i)$ remains the same for every value of $C$, and so $V_i \independent C \, | \, \mathrm{PA}^i$. Thus, if $V_i$ and $C$ are not independent conditional on any subset of other variables, $V_i$'s module changes with $C$, which is represented by an edge between $V_i$ and $C$. Conversely, we assume that if $V_i$'s module changes, which entails that $V_i$ and $C$ are not independent given $\mathrm{PA}^i$, then $V_i$ and $C$ are not independent given any other subset of $\mathbf{V}\backslash \{ V_i\}$. If this assumption does not hold, then we only claim to detect some (but not necessarily all) variables with changing modules.
       
       Step 3 aims to discover the skeleton of the causal structure over $\mathbf{V}$. \textcolor{black}{It leverages the results from Step 2: if neither $V_{i}$ nor $V_{j}$ is adjacent to $C$, then $C$ does not need to be involved in the conditioning set.} In practice, one may apply any constraint-based search procedures on $ \mathbf{V} \cup C$, e.g., SGS and PC \citep{SGS93}. Its (asymptotic) correctness is justified by the following theorem:
       \begin{Theorem} \label{theo1}
       	Given Assumptions \ref{Ass: pseudo sufficiency} and \ref{Ass: bias}, for every $V_i, V_j\in \mathbf{V}$, $V_i$ and $V_j$ are not adjacent in $G$ if and only if they are independent conditional on some subset of $\{V_k\,|\,k\neq i, k\neq j\}\cup \{ C\}$.
       \end{Theorem}
       {\it Basic idea of the proof. }  For a complete proof see Appendix A. The ``only if" direction is proved by making use of the weak union property of conditional independence repeatedly,  the fact that all $g_l(C)$ and $\theta_i(C)$ are deterministic functions of $C$, some implications of the SEMs Eq.~\ref{Eq:FCM}, the assumptions in Section~\ref{Sec:assumptions}, and the properties of mutual information given in~\citet{Madiman08}. The ``if" direction is shown based on the faithfulness assumption on $G^{aug}$ and the fact that $\{g_l(C)\}_{l=1}^L \cup \{\theta_i(C)\}_{i=1}^m$ is a deterministic function of $C$.  \hfill\(\Box\)
       
       Furthermore, for any pair of nonadjacent variables $V_i$ and $V_j$ with $V_i - C$ and $V_j - C$, we can easily detect whether there are pseudo confounders behind $V_i$ and $V_j$ from the independence test results derived from Algorithm \ref{rs_causal_discovery}:
       \begin{itemize}[itemsep=0.2pt,topsep=0.2pt]
        \item [1.] If $\forall \mathbf{V}_k \subseteq \mathbf{V} \backslash \{V_i, V_j\}$, $V_i \dependent V_j | \mathbf{V}_k$, and $\exists \mathbf{V}_{k'} \subseteq \mathbf{V} \backslash \{V_i, V_j\}$, so that $V_i \independent V_j | \{\mathbf{V}_{k'}, C\}$, then there exist pseudo confounders behind $V_i$ and $V_j$.
       	\item [2.] If $\exists \mathbf{V}_k \subseteq \mathbf{V} \backslash \{V_i, V_j\}$, so that $V_i \independent V_j | \mathbf{V}_k$, then there is no pseudo confounder behind $V_i$ and $V_j$.
       \end{itemize}
	   
	   Note that in Algorithm \ref{rs_causal_discovery}, it is crucial to use a general, nonparametric conditional independence test, for how variables depending on $C$ is unknown and usually very nonlinear. In this work, we use the kernel-based conditional independence test \citep[KCI-test,][]{Zhang11_KCI} to capture the dependence on $ C $ in a nonparametric way. By contrast, if we use, for example, tests of vanishing partial correlations, as is widely used in the neuroscience community, the proposed method may not work well.
	   
	   Moreover, it is worth noting that the estimated graphical representation by Algorithm \ref{rs_causal_discovery} naturally follows the principle of minimal changes, which was explicitly formulated by \cite{Ghassami18_NIPS}. This is because faithfulness on the augmented graph implies the edge minimality condition in the graphical representation. Any variable adjacent to $C$ in the augmented graph has a changing mechanism, and the estimated graph by Algorithm \ref{rs_causal_discovery} has as few edges involving $C$ as possible; hence it has the smallest number of changing causal mechanisms (or conditional distributions).

	   \section{CD-NOD Phase II: Distribution Shifts Benefit Causal Direction Determination} \label{Sec:unify}
	   We now show that introducing the additional variable $C$ as a surrogate not only allows us to infer the skeleton of the causal structure but also facilitates the determination of some causal directions. Let us call those variables that are adjacent to $C$ in the output of Algorithm 1 ``$C$-specific variables'', which are actually the effects of changing causal modules. For each $C$-specific variable $V_k$, it is possible to determine the direction of every edge which has an endpoint on $V_k$. Let $V_l$ be any variable adjacent to $V_k$ in the output of Algorithm 1. Then there are two possible scenarios to consider:
	   \begin{enumerate}
	   	\item  [S$_1$.] $V_l$ is not adjacent to $C$. Then $C- V_k - V_l$ forms an unshielded triple. For practical purposes, we take the direction between $C$ and $V_k$ as $C \rightarrow V_k$ (though we do not claim $C$ to be a cause in any substantial sense). Then we can use standard orientation rules for unshielded triples to orient the edge between $V_k$ and $V_l$~\citep{SGS93,Pearl00}. There are two possible situations:
	   	\\{\bf 1.a}~If $V_l$ and $C$ are independent given a set of variables excluding $V_k$, then the triple is a V-structure, and we have $V_k \leftarrow V_l$.
	   	\\{\bf 1.b}~Otherwise, if $V_l$ and $C$ are independent given a set of variables including $V_k$, then the triple is not a V-structure, and we have $V_k \rightarrow V_l$.
	   	
	   	\item [S$_2$.] $V_l$ is also adjacent to $C$. This case is more complex than S$_1$, but it is still possible to identify the causal direction between $V_k$ and $V_l$, based on the principle that $P(\texttt{cause})$ and $P(\texttt{effect}\,|\,\texttt{cause})$ change independently; a heuristic method is given in Section~\ref{Sec:general}. 
	   \end{enumerate}
	   
	   The procedure in S$_1$, which will be further discussed in Section \ref{Sec: generalization_invar}, contains the methods proposed in~\citet{Hoover90,Tian01,invariantPred} for causal discovery from changes as special cases. It may also be interpreted as special cases of the principle underlying the method for S$_2$:  
	   if one of $P(\texttt{cause})$ and $P(\texttt{effect}\,|\,\texttt{cause})$ changes while the other remains invariant, they are clearly independent.

	  \subsection{Causal Direction Identification by Generalization of Invariance} \label{Sec: generalization_invar}
	  There exist methods for causal discovery from differences among multiple data sets~\citep{Hoover90,Tian01,invariantPred} that explore {\it invariance} of causal mechanisms. They used linear models to represent causal mechanisms and, as a consequence, the invariance of causal mechanisms can be assessed by checking whether the involved parameters change across data sets or not.  Actually, S$_{1.b}$ above provides a nonparametric way to achieve this in light of nonparametric conditional independence tests.
	  For any variable $V_i$ and a set of variables $\mathbf{S}$, the conditional distribution $P(V_i \,|\, \mathbf{S})$ is invariant across different values of $C$ if and only if
	  $$P(V_i \,|\, \mathbf{S}, C=c_1) = P(V_i \,|\, \mathbf{S}, C=c_2), ~\forall ~c_1 \textrm{~and~} c_2.$$
	  This is exactly the condition under which $V_i \,\independent \, C \,|\, \mathbf{S}$. In other words, testing for invariance (or homogeneity) of the conditional distribution is naturally achieved by performing a conditional independence test on $V_i$ and $C$ given the set of variables $\mathbf{S}$, for which there exist off-the-shelf algorithms and implementations.  When $\mathbf{S}$ is the empty set, this reduces to a test of marginal independence between $V_i$ and $C$, or a test of homogeneity of $P(V_i)$. 
	  
	  In S$_{1.a}$, we have the invariance of \textcolor{black}{$P(V_l)$ (i.e., $P(\texttt{cause})$)} when the causal mechanism, represented by \textcolor{black}{$P(V_k|V_l)$ (i.e., $P(\texttt{effect}\,|\,\texttt{cause})$)}, changes, which is complementary to the invariance of causal mechanisms in S$_{1.b}$. The (conditional) independence test results between $V_i$ and $C$ are readily available from Algorithm \ref{rs_causal_discovery} and can be applied to determine causal directions between variables which satisfy S$_1$. The procedure is summarized in Algorithm \ref{Alg: invariance}.
	  
	   \begin{algorithm}[htp]
	  	\caption{Causal Direction Identification by Generalization of Invariance}
	  	\begin{enumerate}
	  		\item \textbf{Input:} causal skeleton $U_{\mathcal{G}}$ from Algorithm \ref{rs_causal_discovery}.
	  		\item Orient $C \rightarrow V_k$, for any variable which is adjacent to $C$.
	  		\item For any unshielded triple with $C \rightarrow V_k - V_l$, where $V_l$ is not adjacent to $C$,
	  		  \begin{itemize}
	  		  	\item [a.] if $V_l \independent C |\mathbf{S}$, with $\mathbf{S} \subseteq \mathbf{V}$ and $\mathbf{S} \land V_k = \ \emptyset$, orient $V_k \leftarrow V_l$;
	  		  	\item [b.] if $V_l \independent C |\mathbf{S}$, with $\mathbf{S} \subseteq \mathbf{V}$ and $V_k \in \mathbf{S}$, orient $V_k \rightarrow V_l$.
	  		  \end{itemize} 
  		  \item \textbf{Output:} partially oriented graph $U_{\mathcal{G}}$ using the property of generalization of invariance.
 
	  	\end{enumerate}
	  	\label{Alg: invariance}
	  \end{algorithm}
	  
	  Naturally, both invariance properties above are particular cases of the principle of {\it independent changes} of causal modules underlying the method for S$_2$:  
	  if one of $P(\texttt{cause})$ and $P(\texttt{effect}\,|\,\texttt{cause})$ changes while the other remains invariant, they are clearly independent. Usually, there is no reason why only one of them could change, so the above invariance properties are rather restrictive. The property of {\it independent changes} holds in rather generic situations, e.g., when there is no confounder behind $\texttt{cause}$ and $\texttt{effect}$.  Below we will propose an algorithm for causal direction determination based on independent changes of causal modules.
	  
	  \subsection{Causal Direction Identification by Independently Changing Modules} \label{Sec:general}
	  
	  We now develop a method to handle S$_2$ above. To clearly express the idea, let us start with a two-variable case: suppose $V_1$ and $V_2$ are adjacent and are both adjacent to $C$. We aim to identify the causal direction between them, which, without loss of generality, we assume to be $V_1\rightarrow V_2$. 

	  Figure \ref{fig:illust_direct} shows the case where the involved changing parameters, $\theta_1(C)$ and $\theta_2(C)$, are independent, i.e., $P(V_1;\theta_1)$ and $P(V_2\,|\,V_1;\theta_2)$ change independently (we dropped the argument $C$ in $\theta_1$ and $\theta_2$ to simplify notation).
	  
	 \begin{figure} 
	  	\begin{center}
	  		\begin{tikzpicture}[scale=.7, line width=0.5pt, inner sep=0.2mm, shorten >=.1pt, shorten <=.1pt]
	  		\draw (.5, 0) node(1) [circle, draw] {{\footnotesize\,$V_1$\,}};
	  		\draw (2.8, 0) node(2) [circle, draw] {{\footnotesize\,$V_2$\,}};
	  		\draw (.5, 1.9) node(3)  {{\footnotesize\,$\theta_1(C)$\,}};
	  		\draw (2.8, 1.9) node(4)  {{\footnotesize\,$\theta_2(C)$\,}};
	  		\draw[-latex] (1) -- (2); 
	  		\draw[-latex] (3) -- (1); 
	  		\draw[-latex] (4) -- (2); 
	  		\end{tikzpicture}
	  	\end{center} 
	  	\caption{An illustration of a two-variable case: $V_1 \rightarrow V_2$ with corresponding parameters $\theta_1(C)$ and $\theta_2(C)$ changing independently. } \label{fig:illust_direct}
	  \end{figure}
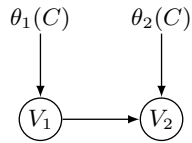
	  
	  For the reverse direction, one can decompose the joint distribution of $(V_1, V_2)$ according to
	  	\begin{equation} \label{Eq:reverse}
	  	P(V_1,V_2;\theta_1', \theta_2') = P(V_2;\theta_2')P(V_1\,|\,V_2;\theta_1'),
	  	\end{equation}
	  where $\theta_1'$ and $\theta_2'$ are assumed to be sufficient for the corresponding distribution modules $P(V_2)$ and $P(V_1|V_2)$. Generally speaking, $\theta_1'$ and $\theta_2'$ are not independent, because they are determined jointly by $\theta_1$ and $\theta_2$. 
	  
	  Now we face the problem of how to compare the dependence between $\theta_1$ and $\theta_2$ with that between $\theta'_1$ and $\theta'_2$. Since $\theta$ is assumed to be sufficient for the corresponding distribution module, it is equivalent to compare the dependence between $P(V_1)$ and $P(V_2|V_1)$ with that between $P(V_2)$ and $P(V_1|V_2)$. 
	  
	  The idea that causal modules are independent is not new \citep{Pearl00}, but note that in a stationary situation where each module is fixed, such independence is very difficult, if not impossible, to test. By contrast, in the situation we are considering presently, both $P(V_1)$ and $P(V_2|V_1)$ are changing, and we can try to measure the extent to which variation in $P(V_1)$ and variation in $P(V_2)$ are dependent (and similarly for $P(V_2)$ and $P(V_1|V_2)$). This is the sense in which distribution change actually helps in the identification of causal directions, and as far as we know, this is the first time that such an advantage is exploited in the case where both $P(\texttt{cause})$ and $P(\texttt{effect}\,|\,\texttt{cause})$ change.
	  
	  We extend the Hilbert Schmidt Independence Criterion \citep[HSIC,][]{Gretton08} to measure the dependence between causal modules. To do so, we first develop a novel kernel embedding of nonstationary conditional distributions which does not rely on sliding windows and estimate their corresponding Gram matrices in Section \ref{Sec: Kernel Embedding}, which will be used in the extended HSIC and the causal direction determination rule in Section \ref{Sec: HSIC}. In Section \ref{Sec: HSIC_multi}, we propose an algorithm for causal direction determination in multi-variable cases by taking advantage of independent changes. 
	
	 \subsubsection{Kernel Embedding of Constructed Joint Distributions} \label{Sec: Kernel Embedding}
	 \paragraph*{Notation}
	 Throughout this section, we use the following notation. 
	 Let $X$ be a random variable on domain $\mathcal{X}$, and $(\mathcal{H},k)$ be a Reproducing Kernel Hilbert Space (RKHS) with a measurable kernel on $\mathcal{X}$. Let $\phi(x) \in \mathcal{H}$ represent the feature map for each $x \in \mathcal{X}$, with $\phi: \mathcal{X} \rightarrow \mathcal{H}$. We assume integrability: $E_X[k(X,X)] \leq \infty$. Similar notations are for variables $Y$ and $C$. The cross-covariance operator $C_{YX}: \mathcal{H} \rightarrow \mathcal{G}$ is defined as $C_{YX} := E_{YX}[\phi(X) \otimes \psi(Y)]$, where $\mathcal{G}$ is the RKHS associated with $Y$.
	  
	 We represent causal modules $P(V_i|\mathrm{PA}^i, C)$ by kernel embedding. Intuitively, to represent the kernel embedding of changing causal modules, we need to consider $P(V_i|\mathrm{PA}^i, C)$ for each value of $C$ separately. If $C$ is a domain index, for each value of $C$ we have a dataset of $(V_i , \mathrm{PA}^i)$. If $C$ is a time index, one may use a sliding window to use the data of $(V_i,\mathrm{PA}^i)$ in the window of length $L$ centered at $C=c$. However, in some cases, it might be hard to find an appropriate window length $L$, especially when the causal module changes fast. In the following, we propose a way to estimate the kernel embedding of changing causal modules on the whole dataset, avoiding window segmentation. For the sake of conciseness, below we use $Y$ and $X$ to denote $V_i$ and $\mathrm{PA}^i$, respectively.
	 
	 Suppose that there are $N$ samples for each variable. Instead of working with $P(Y|X,C=c_n)$ $(n = 1,\cdots,N)$ directly, we ``virtually" construct a particular  distribution $\tilde{P}(\underline{Y}, X \,|\,C=c_n)$ as follows:\footnote{Here we use $\underline{Y}$ instead of ${Y}$ to emphasize that in this constructed distribution $Y$ and $X$ are not symmetric.}
	 \begin{equation} 
	 \tilde{P}(\underline{Y}, X | C=c_n) = P(Y | X,C=c_n)P(X).
	 \end{equation}
	 \textcolor{black}{The embedding of this ``joint distribution" of $X$ and $Y$ is simpler than that of the conditional of $Y$ given $X$.} Since $P(X)$ does not depend on $C$ and its support is rich enough to contain that of $P(X | C=c_n)$, one can see that whenever there are changes in $P(Y | X, C=c_n)$ across different values of $c_n$, there must be changes in $\tilde{P}(\underline{Y}, X | C=c_n)$, and vice versa.  In other words, the constructed distribution $\tilde{P}(\underline{Y}, X | C=c_n)$ captures changes in $P(Y | X, C=c_n)$ across different $c_n$.  We let $\tilde{P}(\underline{Y}, X,C=c_n) = P(Y | X,C=c_n)P(X)P(C=c_n)$.
	 
	 Proposition \ref{embedding} shows that the kernel embedding of the distribution $\tilde{P}(\underline{Y}, X | C=c_n)$ can be estimated on the whole data set, without window segmentation.
	 \begin{Proposition} \label{embedding}
	 	Let $X$ represent the direct causes of $Y$, and suppose that they have $N$ observations. The kernel embedding of distribution $\tilde{P}(\underline{Y}, X |C=c_n)$ can be represented as 
	 	\begin{flalign} \nonumber
	 	\hat{\tilde{\mu}}_{\underline{Y},X|{C=c_n}} 
	 	&= \frac{1}{n} \boldsymbol{\Phi}_\mathbf{y} (\mathbf{K}_\mathbf{x}\odot \mathbf{K}_\mathbf{c} + \lambda I)^{-1} \textrm{diag}(\mathbf{k}_{\mathbf{c},c_n}) \mathbf{K}_\mathbf{x} \boldsymbol{\Phi}_{\mathbf{x}}^\intercal,
	 	\end{flalign}
	 	where $\boldsymbol{\Phi}_\mathbf{y} := [\phi(y_1), ..., \phi(y_N)]$,  $\boldsymbol{\Phi}_\mathbf{x} := [\phi(x_1), ..., \phi(x_N)]$, 
	 	$\mathbf{k}_{\mathbf{c},c_n} := [k(c_1, c_n), ..., k(c_N, c_n)]^\intercal$, $\mathbf{K}_\mathbf{x}(x_t, x_{t'}) = \langle \phi(x_t), \phi(x_{t'}) \rangle$, $\mathbf{K}_\mathbf{c}(c_t, c_{t'}) = \langle \phi(c_t), \phi(c_{t'}) \rangle$, and $\odot$ represents point-wise product.
	 \end{Proposition}
     The detailed proof of proposition \ref{embedding} is given in Appendix B. 
	 Next we estimate the $N \times N$ Gram matrix of $\hat{\tilde{\mu}}_{\underline{Y},X|{C=c}}$. We consider different kernels for the estimation of Gram matrix. Let $M_{\underline{Y}X}^l$ represent the Gram matrix of $\hat{\tilde{\mu}}_{\underline{Y},X|{C=c}}$ with a linear kernel, and $M^\mathcal{G}_{\underline{Y}X}$ the Gram matrix of $\hat{\tilde{\mu}}_{\underline{Y},X|{C=c}}$ with a Gaussian kernel. 
	 
	 If we use a linear kernel, the $(c, c')$th entry of the Gram matrix $M_{\underline{Y}X}^l$ is the  inner product between $ \hat{\tilde{\mu}}_{\underline{Y},X|{C=c}}$ and $ \hat{\tilde{\mu}}_{\underline{Y},X|{C=c'}}$:
	 	\begin{flalign} \nonumber
	 	 M_{\underline{Y}X}^l(c,c') 
	 	\triangleq & \textrm{tr}( \hat{\tilde{\mu}}_{\underline{Y},X|{C=c}}^\intercal \hat{\tilde{\mu}}_{\underline{Y},X|{C=c'}}) \\ 
	 	= & \frac{1}{N^2}
	 	\mathbf{k}_{\mathbf{c},c}^\intercal 
	 	\big[ \mathbf{K}_\mathbf{x}^3 \odot \big(
	 	(\mathbf{K}_\mathbf{x}\odot \mathbf{K}_\mathbf{c} + \lambda I)^{-1} \mathbf{K}_\mathbf{y} (\mathbf{K}_\mathbf{x}\odot \mathbf{K}_\mathbf{c} + \lambda I)^{-1}
	 	\big) \big] 
	 	\mathbf{k}_{\mathbf{c},c'},
	 	\end{flalign}
	 which is the $(c,c')$th entry of the matrix 
	 	\begin{flalign}  
	 	 \mathbf{M}^l_{\underline{Y}X} 
	 	= \frac{1}{N^2}
	 	\mathbf{K}_{\mathbf{c}}
	 	\big[ \mathbf{K}_\mathbf{x}^3 \odot \big(
	 	(\mathbf{K}_\mathbf{x}\odot \mathbf{K}_\mathbf{c} + \lambda I)^{-1} \mathbf{K}_\mathbf{y} (\mathbf{K}_\mathbf{x}\odot \mathbf{K}_\mathbf{c} + \lambda I)^{-1}
	 	\big) \big] 
	 	\mathbf{K}_{\mathbf{c}}. \label{Eq:Linear2}
	 	\end{flalign}
	 
	 If we use a Gaussian kernel with kernel width $\sigma_2$, the Gram matrix is given by
	 \begin{flalign}  \nonumber 
	  M^\mathcal{G}_{\underline{Y}X}(c,c') 
	 = & \textrm{exp}\big(- \frac{ ||\tilde{\mu}_{\underline{Y},X|{C=c}} -  \tilde{\mu}_{\underline{Y},X|{C=c'}} ||_F^2}{2\sigma_2^2} \big) \\  \label{Eq:Gg}
	 = & \textrm{exp}\big(-\frac{
	 	M_{\underline{Y}X}^l(c,c) +M_{\underline{Y}X}^l(c',c') - 2M_{\underline{Y}X}^l(c',c)}{2\sigma_2^2}
	 \big),
	 \end{flalign}
	 where $||\cdot||_F$ denotes the Frobenius norm. This can be represented in matrix notation as
	 \begin{flalign}  
	  \mathbf{M}^\mathcal{G}_{\underline{Y}X} 
	  =  \textrm{exp}\big(-\frac{
	 	diag(\mathbf{M}_{\underline{Y}X}^l) \cdot \mathbf{1}_{N}+ \mathbf{1}_{N} \cdot diag(\mathbf{M}_{\underline{Y}X}^l) - 2\mathbf{M}_{\underline{Y}X}^l}{2\sigma_2^2}
	 \big), \label{Eq:Gg2}
	 \end{flalign}
	 where $diag(\cdot)$ sets all off-diagonal entries zero, and $\mathbf{1}_{N}$ is an $N \times N$ matrix with all entries being 1.
		 
	 Excitingly, we can see that with our methods, we do not need to explicitly learn the high-dimensional kernel embedding $\tilde{\mu}_{\underline{Y},X|{C=c}}$ for each $c$. With the kernel trick, the final Gram matrix can be represented by $N \times N$ kernel matrices directly. 
	 
	 There are several hyperparameters to set. The hyperparameters associated with $\mathbf{K_x}$, $\mathbf{K_c}$, and the regularization parameter $\lambda$ in equation (\ref{Eq:Linear2}) are learned through a Gaussian process regression framework: they are learned by maximizing the marginal likelihood of $Y$. For the hyperparameters associated with $\mathbf{K_y}$ and the kernel with $\sigma_2$ in equation (\ref{Eq:Gg}), we set them with empirical values; please refer to \cite{Zhang11_KCI} for details.
	 
	 \textit{Change in marginal distributions.}
	 As a special case, when we are concerned with how the marginal distribution of $Y$ changes with $C$, i.e., when $X = \emptyset$, we directly make use of 
	 \begin{equation}
	   \mu_{Y|C=c_n} =  \mathcal{C}_{YC} \mathcal{C}_{CC}^{-1}\phi (c_n).
	 \end{equation}
	 This can also be obtained by constraining $X$ in $\tilde{\mu}_{\underline{Y},X|C=c_n}$ to take a fixed value. 
	 Its empirical estimate is 
	 \begin{flalign} \nonumber
	 \hat{\mu}_{Y|C=c_n} &=  \frac{1}{N}\boldsymbol{\Phi}_\mathbf{y} \boldsymbol{\Phi}_\mathbf{c}^\intercal  ( \frac{1}{n} \boldsymbol{\Phi}_\mathbf{c} \boldsymbol{\Phi}_\mathbf{c}^\intercal + \lambda I)^{-1}  \phi_{c_n} \\  
	 &=  \boldsymbol{\Phi}_\mathbf{y} ( \mathbf{K}_\mathbf{c} + \lambda I)^{-1} \mathbf{k}_{\mathbf{c},c_n}.
	 \end{flalign}
	 Then $(c,c')$ entry of the Gram matrix with a linear kernel is:
	 \begin{flalign} \nonumber
	 M_Y^l(c,c')& \triangleq  \hat{\mu}_{Y|{C=c}}^\intercal \hat{\mu}_{Y|{C=c'}}\\ \nonumber
	 & =  \mathbf{k}_{\mathbf{c},c}^\intercal ( \mathbf{K}_\mathbf{c} + \lambda I)^{-1} \boldsymbol{\Phi}_\mathbf{y}^\intercal \boldsymbol{\Phi}_\mathbf{y}  ( \mathbf{K}_\mathbf{c} + \lambda I)^{-1} \mathbf{k}_{\mathbf{c},c'}\\ 
	 & =  \mathbf{k}_{\mathbf{c},c}^\intercal ( \mathbf{K}_\mathbf{c} + \lambda I)^{-1} \mathbf{K}_\mathbf{y}  ( \mathbf{K}_\mathbf{c} + \lambda I)^{-1} \mathbf{k}_{\mathbf{c},c'},
	 \end{flalign}
	 which is the $(c,c')$th entry of 
	 \begin{equation}
	  \mathbf{M}^l_{Y} = \mathbf{K}_{\mathbf{c}} ( \mathbf{K}_\mathbf{c} + \lambda I)^{-1} \mathbf{K}_\mathbf{y}  ( \mathbf{K}_\mathbf{c} + \lambda I)^{-1} \mathbf{K}_{\mathbf{c}}.
	 \end{equation}
	 For a Gaussian kernel with kernel with $\sigma_2$, the Gram matrix is
	 \begin{equation}
	   \mathbf{M}^\mathcal{G}_{Y} = \textrm{exp}\big(-\frac{
	   	diag(\mathbf{M}_{\underline{Y}X}^l) \cdot \mathbf{1}_{N}+ \mathbf{1}_{N} \cdot diag(\mathbf{M}_{\underline{Y}X}^l) - 2\mathbf{M}_{\underline{Y}X}^l}{2\sigma_2^2} \big).
	 \end{equation}
	
	  \subsubsection{Two-Variable Case} \label{Sec: HSIC}
	  In this section, we extend HSIC to measure the dependence between causal modules, based on which we determine causal directions.
	  
	  For simplicity, let us start with the two-variable case: suppose that $X$ and $Y$ are adjacent and both are adjacent to $C$. We aim to identify the causal direction between them, which, without loss of generality, we assume to be $X \rightarrow Y$. The guiding idea is that distribution shift may carry information that confirms the independence of causal modules, which, in the simple case we are considering, is the independence between $P(X)$ and $P(Y|X)$. If $P(X)$ and $P(Y|X)$ are independent but $P(Y)$ and $P(X|Y)$ are not, then the causal direction is inferred to be from $X$ to $Y$. 
	  
	  The dependence between $P(X)$ and $P(Y|X)$ can be measured by extending the HSIC \citep{Gretton08}.
	  
	  \paragraph{HSIC}
	  
	  Given a set of observations $\{(u_1, v_1), (u_2, v_2), ... ,  (u_N, v_N) \}$ for variables $U$ and $V$, HSIC provides a measure of dependence and a statistic for testing their statistical independence. Roughly speaking, it measures the squared covariances between feature maps of $U$ and feature maps of $V$. Let $\mathbf{M}_U$ and $\mathbf{M}_V$ be the Gram matrices for $U$ and $V$ calculated on the sample,  respectively. An estimator of HSIC is given by~\citet{Gretton08}:
	  \begin{equation} \label{Eq:HSIC}
	  \mathrm{HSIC}_{UV} = \frac{1}{(N-1)^2}\mathtt{tr}(\mathbf{M}_U  H \mathbf{M}_V H),
	  \end{equation}
	  where $H$ is used to center the features, with entries $H_{ij} := \delta_{ij} - N^{-1}$. 
	  
	  In what follows, we will use a normalized version of the estimated HSIC, which is invariant to the scale in $\mathbf{M}_U$ and $\mathbf{M}_V$:
	  \begin{flalign} \nonumber
	  \mathrm{HSIC}_{UV}^\mathcal{N} &= \frac{\mathrm{HSIC}_{UV}}{\frac{1}{N-1}\mathtt{tr}(\mathbf{M}_U H)\cdot \frac{1}{N-1}\mathtt{tr}(\mathbf{M}_VH) } \\ \label{Eq:norm_HSIC}
	  &= \frac{\mathtt{tr}(\mathbf{M}_U  H \mathbf{M}_V H)}{\mathtt{tr}(\mathbf{M}_U  H)\mathtt{tr}(\mathbf{M}_V  H)}.
	  \end{flalign}

	  \paragraph{Dependence between Changing Modules}
	  
	  In our case, we aim to check whether $P(Y|X)$ and $P(X)$ change independently along with $C$. We work with the estimate of their embeddings. Then we can think of $\big\{ \big( 
	  \hat{{\mu}}_{X|{C=c_n}}, \hat{\tilde{\mu}}_{\underline{Y},X|{C=c_n}} \big)\big\}_{n=1}^{N}$ as the observed data pairs and measure their dependence from the data pairs. 
	  
	  This can be done by applying (the normalized version of) the estimate of HSIC given in Eq. \ref{Eq:norm_HSIC} to the above data pairs. The expression then involves $\mathbf{M}_X$, the Gram matrix of $\hat{\mu}_{X|C}$ at $C = c_1, c_2, ..., c_N$, and $\mathbf{M}_{\underline{Y}X}$, the Gram matrix of $\hat{\tilde{\mu}}_{\underline{Y}X|C}$  at $C = c_1, c_2, ..., c_N$.
	  In particular, the dependence between $P(Y|X)$ and $P(X)$ on the given data can be estimated by
	  \begin{equation} 
	  \hat{\Delta}_{X \rightarrow Y} =  \frac{\mathtt{tr}(\mathbf{M}_X  H \mathbf{M}_{\underline{Y}X} H)}{\mathtt{tr}(\mathbf{M}_X  H)\mathtt{tr}(\mathbf{M}_{\underline{Y}X}  H)}.
	  \label{Dir_criterion}
	  \end{equation}
	  Similarly, for the hypothetical direction $Y \rightarrow X$ the dependence between $P(X|Y)$ and $P(Y)$ on the data is estimated by
	  \begin{equation} 
	  \hat{\Delta}_{Y \rightarrow X} =  \frac{\mathtt{tr}(\mathbf{M}_Y  H \mathbf{M}_{\underline{X}Y} H)}{\mathtt{tr}(\mathbf{M}_Y  H)\mathtt{tr}(\mathbf{M}_{\underline{X}Y}  H)}.
	  \label{Dir_criterion2}
	  \end{equation}
	  
	  We then have the following rule to determine the causal direction between $X$ and $Y$.
	  \paragraph*{Causal Direction Determination Rule}
	  Suppose that $X$ and $Y$ are two random variables with $N$ observations. We assume that $X$ and $Y$ are adjacent and both are adjacent to $C$ and assume no pseudo confounders behind them. The causal direction between $X$ and $Y$ is then determined according to the following rule: 
	  \begin{itemize}
	  	\item if $\hat{\Delta}_{X \rightarrow Y} < \hat{\Delta}_{Y \rightarrow X}$, output $X \rightarrow Y$;
	  	\item if $\hat{\Delta}_{X \rightarrow Y} > \hat{\Delta}_{Y \rightarrow X}$, output $X \leftarrow Y$.
	  \end{itemize}
	  In practice, there may exist pseudo confounders. In such a case, we set a threshold $\alpha$ on $\hat{\Delta}$. If $\hat{\Delta}_{X \rightarrow Y} > \alpha$ and $\hat{\Delta}_{Y \rightarrow X} > \alpha$, we conclude that there are pseudo confounders which influence both $X$ and $Y$ and leave the direction undetermined.

	  \subsubsection{With More Than Two Variables} \label{Sec: HSIC_multi}
	  
	  The causal direction determination rule in the two-variable case can be extended to learn causal directions in multi-variable cases. Suppose that we have $m$ observed random variables $\{V_i\}_{i=1}^m$ and a partially oriented graph $U_{\mathcal{G}}$ derived from Algorithms \ref{rs_causal_discovery} and \ref{Alg: invariance}. Let $\mathbf{V}_S$ be the subset of $\{V_i\}_{i=1}^m$, such that $V_i\in \mathbf{V}_S$ if and only if $V_i$'s causal module changes. 

     \textcolor{black}{Note that generally speaking, different from the unconfounded two-variable case,  when identifying the causal direction between an unoriented pair of adjacent variables, we need to remove the effect from their common causes. Thus,} before moving forward, we first define \textit{deconfounding set} and \textit{potential deconfounding set} of a pair of adjacent variables in $U_{\mathcal{G}}$. 
     \begin{Definition}[Deconfounding Set]
     	A set of variables $\mathbf{Z} \subseteq \mathbf{V} \backslash \{V_l,V_k\}$ is the deconfounding set of a pair of adjacent variables $(V_l,V_k)$, if
     	\begin{enumerate}
     		\item [(i)] no node in $\mathbf{Z}$ is a descendant of $V_l$ or $V_k$,
     		\item [(ii)] and $\mathbf{Z}$ blocks every path between $V_l$ and $V_k$ that contains arrows into $V_l$ and $V_k$.   		
     	\end{enumerate}
     \end{Definition}
     Furthermore, a set of variables $\mathbf{Z}$ is the \textit{minimal deconfounding set} of a pair of adjacent variables $(V_l,V_k)$, if any $\mathbf{Z}_s \subset \mathbf{Z}$ is not a deconfounding set.

        \begin{Definition}[Potential Deconfounding Set]
     	A set of variables $\mathbf{Z} \subseteq \mathbf{V} \backslash \{V_l,V_k\}$ is the potential deconfounding set of a pair of adjacent variables $(V_l,V_k)$, if
     	\begin{enumerate}
     		\item [(i)] no node in $\mathbf{Z}$ is a descendant of $V_l$ or $V_k$,
     		\item [(ii)] $\mathbf{Z}$ blocks every path between $V_l$ and $V_k$ that does not contain an arrow out of $V_l$ or $V_k$, 
     		\item [(iii)] and any $Z \in \mathbf{Z}$ is not in the deconfounding set.
     	\end{enumerate}    
      \end{Definition}
       Similarly, a set of variables $\mathbf{Z}$ is the \textit{minimal potential deconfounding set} of a pair of adjacent variables $(V_l,V_k)$, if any subset $\mathbf{Z}_s \subset \mathbf{Z}$ is not a potential deconfounding set.
       
      In Figure \ref{Exp: confouding}, for example, the set $\mathbf{Z} = \{V_3\}$ is a minimal deconfounding set of $(V_1, V_2)$, and the set $\mathbf{Z} = \{V_5\}$ is a minimal potential deconfounding set of $(V_1, V_2)$.
      
      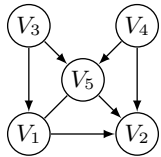
\begin{figure} 
      	\begin{center}
      		\begin{tikzpicture}[scale=.8, line width=0.5pt, inner sep=0.2mm, shorten >=.1pt, shorten <=.1pt]
      		\draw (0, 0) node(1) [circle, draw] {{\footnotesize\,$V_1$\,}};
      		\draw (1.8, 0) node(2) [circle, draw] {{\footnotesize\,$V_2$\,}};
      		\draw (0, 1.8) node(3) [circle, draw] {{\footnotesize\,$V_3$\,}};
      		\draw (1.8, 1.8) node(4) [circle, draw] {{\footnotesize\,$V_4$\,}};
            \draw (0.9, 0.9) node(5) [circle, draw] {{\footnotesize\,$V_5$\,}};
      		\draw[-latex] (1) -- (2); 
      		\draw[-latex] (3) -- (1); 
      		\draw[-latex] (4) -- (2); 
      		\draw[-] (5) -- (1); 
      		\draw[-latex] (5) -- (2); 
      		\draw[-latex] (3) -- (5); 
      		\draw[-latex] (4) -- (5); 
      		\end{tikzpicture} 
      		 \caption{An illustration of the definitions of minimal deconfounding set and minimal potential deconfounding set on a partially oriented graph.}
      	\label{Exp: confouding}
      	\end{center}
      \end{figure}
      
     We take advantage of the independence between causal modules to identify directions. 
     To efficiently identify causal directions using independent changes when there are multiple variables, we propose Algorithm \ref{Alg: direction}, with the main procedure as follows. For each undirected pair $V_l, V_k \in \mathbf{V}_S$, we denote their minimal deconfounding set by $\mathbf{Z}_{lk}^{(1)}$ and their minimal potential deconfounding set by $\mathbf{Z}_{lk}^{(2)}$. Note that there may be multiple minimal deconfounding sets and multiple minimal potential deconfounding sets; to search for such sets more efficiently, in our implementation, we only consider those ones with every variable in $\mathbf{Z}_{lk}^{(1)}$ and  $\mathbf{Z}_{lk}^{(2)}$ being adjacent to $V_l$. Let $\mathbf{Z}_{lk}^{(2,n)}$ be a subset of $\mathbf{Z}_{lk}^{(2)}$, where $n$ is the total cardinality of $\mathbf{Z}_{lk}^{(1)}$ and $\mathbf{Z}_{lk}^{(2,n)}$; i.e., $|\mathbf{Z}_{lk}^{(1)}| + |\mathbf{Z}_{lk}^{(2,n)}| = n$.
     We evaluate the dependence between $P(V_k, \mathbf{Z}_{lk}^{(1)},\mathbf{Z}_{lk}^{(2,n)})$ and $P(V_l|V_k,\mathbf{Z}_{lk}^{(1)},\mathbf{Z}_{lk}^{(2,n)})$, and that between $P(V_l, \mathbf{Z}_{lk}^{(1)},\mathbf{Z}_{lk}^{(2,n)})$ and $P(V_k|V_l,\mathbf{Z}_{lk}^{(1)},\mathbf{Z}_{lk}^{(2,n)})$. If we find that $P(V_l,\mathbf{Z}_{lk}^{(1)},\mathbf{Z}_{lk}^{(2,n)}) \independent P(V_k|V_l,\mathbf{Z}_{lk}^{(1)},\mathbf{Z}_{lk}^{(2,n)})$ and that $P(V_k,\mathbf{Z}_{lk}^{(1)},\mathbf{Z}_{lk}^{(2,n)}) \dependent P(V_l|V_k,\mathbf{Z}_{lk}^{(1)},\mathbf{Z}_{lk}^{(2,n)})$, we output $V_l \rightarrow V_k$, and if there are unoriented edges from variables in $\mathbf{Z}_{lk}^{(2,n)}$ to $V_k$ or $V_l$, then we consider those variables as parents. Similarly, if we find that $P(V_k,\mathbf{Z}_{lk}^{(1)},\mathbf{Z}_{lk}^{(2,n)}) \independent P(V_l|V_k,\mathbf{Z}_{lk}^{(1)},\mathbf{Z}_{lk}^{(2,n)})$ and that $P(V_l,\mathbf{Z}_{lk}^{(1)},\mathbf{Z}_{lk}^{(2,n)}) \dependent P(V_k|V_l,\mathbf{Z}_{lk}^{(1)},\mathbf{Z}_{lk}^{(2,n)})$,  we output $V_k \rightarrow V_l$, instead. Similar to the search procedure of PC, we start from $n=0$, evaluate the dependence between corresponding modules for each undirected pair, and then let $n = n+1$ and repeat the procedure until no unoriented pairs have the total cardinality of minimal deconfounding set and minimal potential deconfounding set greater than or equal to $n$.

     Note that in Algorithm \ref{Alg: direction}, we use dependence measures in (\ref{Dir_criterion}) and (\ref{Dir_criterion2}) to determine the independence between causal modules. Particularly, we set a threshold $\alpha$. If the dependence measure $\hat{\Delta} \leq \alpha$, then the corresponding modules are independent; otherwise, they are dependent. For a pair of adjacent variables $V_k, V_l \in \mathbf{V}_S$, if their direction is undetermined by Algorithm \ref{Alg: direction} and all their measured modules are dependent, then there exist pseudo confounders behind $V_k$ and $V_l$. 
     
     \begin{algorithm}[htp]
     	\caption{Causal Direction Identification by Independent Changes of Causal Modules}
     	\begin{enumerate}
     		\item  \textbf{Input}: observations of $\{V_i\}_{i=1}^m$, a subset $\mathbf{V}_S \subseteq \mathbf{V}$ ($\forall V \in \mathbf{V}_S$, $V$ has a changing causal module), a partially oriented causal graph $U_{\mathcal{G}}$ from Algorithms \ref{rs_causal_discovery} and \ref{Alg: invariance}.
     		
     		\item $n = 0$\\		
     		Repeat
     		\begin{enumerate}[nolistsep]
     			\item [] Repeat
     			\begin{enumerate}
     				\item  Select an unoriented pair of adjacent variables $(V_l, V_k)$, with $V_l, V_k\in \mathbf{V}_S$.
     				\begin{enumerate}
     					\item Let $\mathbf{Z}_{lk}^{(1)}$ be the set of minimal deconfounding set of $V_k$ and $V_l$; any $Z \in \mathbf{Z}_{lk}^{(1)}$ is adjacent to $V_l$.
     					\item Let $\mathbf{Z}_{lk}^{(2)}$ be the set of minimal potential deconfounding set; any $Z \in \mathbf{Z}_{lk}^{(2)}$ is adjacent to $V_l$.
     				\end{enumerate}
     				
     				\item  If $|\mathbf{Z}_{lk}^{(1)}| + |\mathbf{Z}_{lk}^{(2)}| \geq n$, repeat
     				\begin{enumerate}
     					\item Take $\mathbf{Z}_{lk}^{(2,n)} \subseteq \mathbf{Z}_{lk}^{(2)}$, with $|\mathbf{Z}_{lk}^{(2,n)}|=n- |\mathbf{Z}_{lk}^{(1)}| $.
     					\item If $P(V_k,\mathbf{Z}_{lk}^{(1)},\mathbf{Z}_{lk}^{(2,n)}) \independent P(V_l|V_k,\mathbf{Z}_{lk}^{(1)},\mathbf{Z}_{lk}^{(2,n)}) \text{ and } P(V_l,\mathbf{Z}_{lk}^{(1)},\mathbf{Z}_{lk}^{(2,n)}) \dependent P(V_k|V_l,\mathbf{Z}_{lk}^{(1)},\mathbf{Z}_{lk}^{(2,n)}) $, output $V_k \rightarrow V_l$; 
     					
     					for $Z \in \mathbf{Z}_{lk}^{(2,n)}$ with $Z - V_l$, output $Z \rightarrow V_l$;  
     					
     					for $Z \in \mathbf{Z}_{lk}^{(2,n)}$ with $Z - V_k$, output $Z \rightarrow V_k$. 
     					
     					\item If $P(V_k,\mathbf{Z}_{lk}^{(1)},\mathbf{Z}_{lk}^{(2,n)}) \dependent P(V_l|V_k,\mathbf{Z}_{lk}^{(1)},\mathbf{Z}_{lk}^{(2,n)}) \text{ and } P(V_l,\mathbf{Z}_{lk}^{(1)},\mathbf{Z}_{lk}^{(2,n)}) \independent P(V_k|V_l,\mathbf{Z}_{lk}^{(1)},\mathbf{Z}_{lk}^{(2,n)}) $, output $V_l \rightarrow V_k$; 
     					
     					for $Z \in \mathbf{Z}_{lk}^{(2,n)}$ with $Z - V_l$, output $Z \rightarrow V_l$; 
     					
     					for $Z \in \mathbf{Z}_{lk}^{(2,n)}$ with $Z - V_k$, output $Z \rightarrow V_k$. 
     					\item If one of the conditions in (B) and (C) holds, return to step (i); 
     				\end{enumerate}
     			\end{enumerate}
     			\item [] Until every unoriented pair of adjacent variables $(V_l, V_k)$, with $V_l, V_k\in \mathbf{V}_S$, has been selected.
     		\end{enumerate}
     		\item [] $n = n+1$.
     		\item \textbf{Output}: graph $U_{\mathcal{G}}$, with edges between variables in $\mathbf{V}_S$ oriented.
     	\end{enumerate}
     	\label{Alg: direction}
     \end{algorithm}

   Furthermore, for variables whose causal modules are stationary and which are adjacent only to variables with stationary modules, the causal directions between them cannot be determined by Algorithm \ref{Alg: invariance} or \ref{Alg: direction}. In such a case, one may further infer some causal directions by making use of Meek's orientation rules \citep{Meek95} used in PC. 
	  
	To better illustrate the algorithm, we go through a simple example, for which the true graph and brief orientation procedures are given in Figure \ref{Exp: direction}. Below is the precise orientation procedure.
	\begin{enumerate}[noitemsep,topsep=0pt]
		\item  In step 1, we have observed variables $\{V_i\}_{i=1}^4$, with $V_i \in \mathbf{V}_S$ for any $i$, and an unoriented graph $U_{\mathcal{G}}$ over $\{V_i\}_{i=1}^4$ after Algorithms \ref{rs_causal_discovery} and \ref{Alg: invariance}. Since all causal modules are changing, there is no invariance property that can be used for orientation determination in Algorithm \ref{Alg: invariance}.
		\item  In step 2, we start from $n=0$. For example, for the unoriented pair of adjacent variables $(V_1,V_2)$, $\mathbf{Z}_{12}^{(1)} = \{\emptyset\}$ and $\mathbf{Z}_{12}^{(2,n)} = \{\emptyset\}$. In this case, we have $P(V_1) \dependent P(V_2|V_1)$ and $P(V_2) \dependent P(V_1|V_2)$, and thus we cannot determine the direction between $V_1$ and $V_2$. Similarly, we cannot determine the direction between $V_2$ and $V_4$.  For the unoriented pair $(V_1,V_3)$, we have that $P(V_3) \independent P(V_1|V_3)$ and that $P(V_1) \dependent P(V_3|V_1)$, and thus we output $V_3 \rightarrow V_1$. Similarly, we can determine the direction between $V_3$ and $V_4$ and output $V_3 \rightarrow V_4$.
		
		Then let $n=1$. For the unoriented pair $(V_1,V_2)$, $\mathbf{Z}_{12}^{(1)} = \{\emptyset\}$ and $\mathbf{Z}_{12}^{(2,n)} = \{V_3\}$. We then have that $P(V_1, V_3) \independent P(V_2 | V_1,V_3)$ and that $P(V_2, V_3) \dependent P(V_1 | V_2,V_3)$, and thus we output $V_1 \rightarrow V_2$. For the unoriented pair $(V_2,V_4)$, $\mathbf{Z}_{24}^{(1)} = \{V_3\}$ and $\mathbf{Z}_{24}^{(2,n)} = \{\emptyset\}$. We then have that $P(V_4, V_3) \independent P(V_2 | V_4,V_3)$ and that $P(V_2, V_3) \dependent P(V_4 | V_2,V_3)$, and thus output $V_4 \rightarrow V_2$.
		\item We output the fully identified causal graph.
	\end{enumerate}

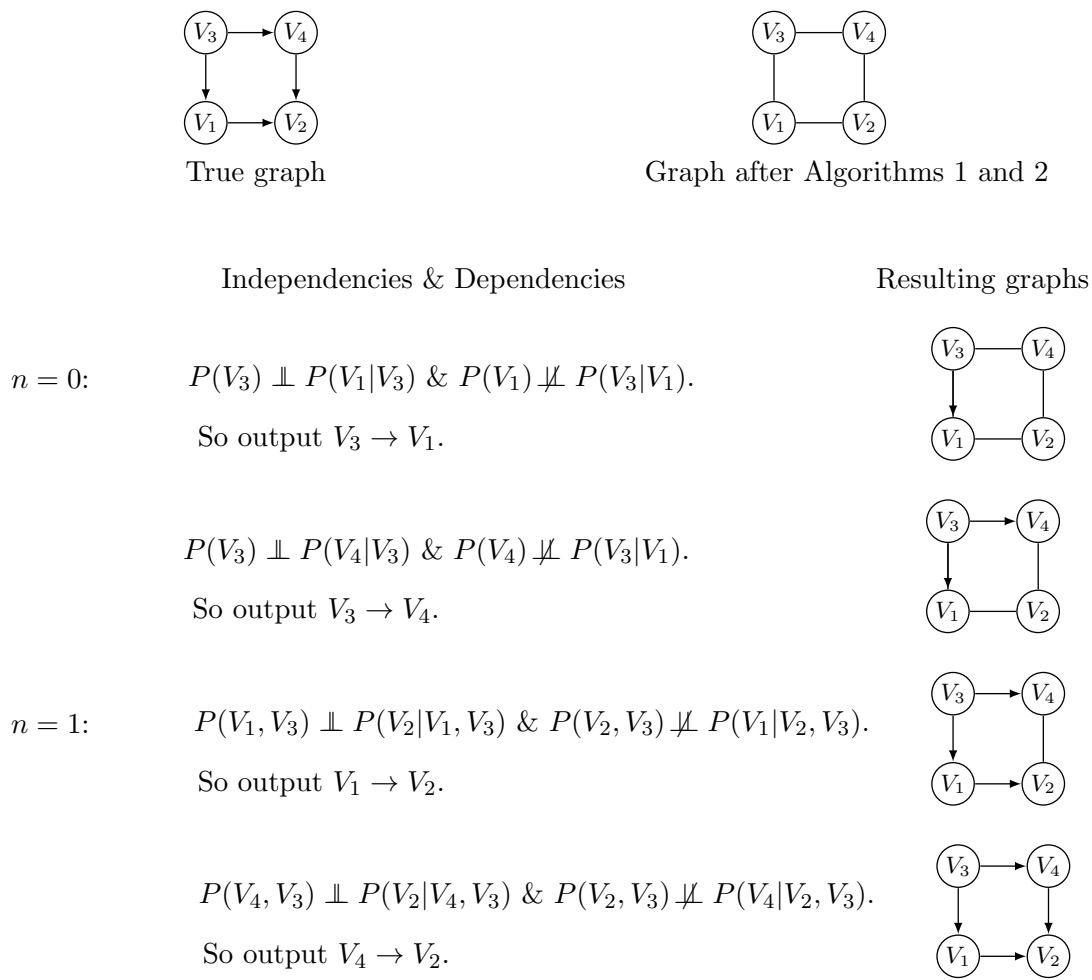
\begin{figure} 
	\begin{center}
		\begin{tikzpicture}[scale=.8, line width=0.5pt, inner sep=0.2mm, shorten >=.1pt, shorten <=.1pt]
		\draw (0, 0) node(1) [circle, draw] {{\footnotesize\,$V_1$\,}};
		\draw (1.5, 0) node(2) [circle, draw] {{\footnotesize\,$V_2$\,}};
		\draw (0, 1.5) node(3) [circle, draw] {{\footnotesize\,$V_3$\,}};
		\draw (1.5, 1.5) node(4) [circle, draw] {{\footnotesize\,$V_4$\,}};
		\draw[-latex] (3) -- (1); 
		\draw[-latex] (1) -- (2); 
		\draw[-latex] (4) -- (2);
		\draw[-latex] (3) -- (4); 
		\end{tikzpicture} ~~~~~~~~~~~~~~~~~~~~~~~~~~~~~~~~~~~~~~~~~~~
		\begin{tikzpicture}[scale=.8, line width=0.5pt, inner sep=0.2mm, shorten >=.1pt, shorten <=.1pt]
		\draw (5, 0) node(1) [circle, draw] {{\footnotesize\,$V_1$\,}};
		\draw (6.5, 0) node(2) [circle, draw] {{\footnotesize\,$V_2$\,}};
		\draw (5, 1.5) node(3) [circle, draw] {{\footnotesize\,$V_3$\,}};
		\draw (6.5, 1.5) node(4) [circle, draw] {{\footnotesize\,$V_4$\,}};
		\draw[-] (3) -- (1); 
		\draw[-] (1) -- (2); 
		\draw[-] (4) -- (2);
		\draw[-] (3) -- (4); 
		\end{tikzpicture} \\
		~~~~~~~~~~~~~~~~~True graph~~~~~~~~~~~~~~~~~~~~~~~~~~~~~~~~~Graph after Algorithms \ref{rs_causal_discovery} and \ref{Alg: invariance} \\
		~\\
		
		~\\
		~~~~~~~~~~~~~~~~~~~~~~~~~Independencies \& Dependencies~~~~~~~~~~~~~~~~~~~~~~~~~~Resulting graphs \\
		~\\		
		\begin{tikzpicture}[scale=.8, line width=0.5pt, inner sep=0.2mm, shorten >=.1pt, shorten <=.1pt]
		\draw (5, 0) node(1) [circle, draw] {{\footnotesize\,$V_1$\,}};
		\draw (6.5, 0) node(2) [circle, draw] {{\footnotesize\,$V_2$\,}};
		\draw (5, 1.5) node(3) [circle, draw] {{\footnotesize\,$V_3$\,}};
		\draw (6.5, 1.5) node(4) [circle, draw] {{\footnotesize\,$V_4$\,}};
		\draw (-10, 1) node(5) {{\,$n=0$:\,}};
		\draw (-3.5, 1) node(7) {{\,$P(V_3) \independent P(V_1|V_3) $ \& $P(V_1) \dependent P(V_3|V_1)$.\,}};
			\draw (-5.5, -0) node(8) {{\,So output $V_3 \rightarrow V_1$.\,}};  
		\draw[-] (3) -- (1); 
		\draw[-] (1) -- (2); 
		\draw[-] (4) -- (2);
		\draw[-] (3) -- (4);
		\draw[-latex] (3) -- (1); 
		\end{tikzpicture} \\
		~\\
		
		\begin{tikzpicture}[scale=.8, line width=0.5pt, inner sep=0.2mm, shorten >=.1pt, shorten <=.1pt]
      \draw (5, 0) node(1) [circle, draw] {{\footnotesize\,$V_1$\,}};
     \draw (6.5, 0) node(2) [circle, draw] {{\footnotesize\,$V_2$\,}};
     \draw (5, 1.5) node(3) [circle, draw] {{\footnotesize\,$V_3$\,}};
     \draw (6.5, 1.5) node(4) [circle, draw] {{\footnotesize\,$V_4$\,}};
     \draw (-10, 1) node(5) {{\,$ \qquad$\,}};
    \draw (-3.5, 1) node(7) {{\,$P(V_3) \independent P(V_4|V_3) $ \& $P(V_4) \dependent P(V_3|V_1)$.\,}};
    \draw (-5.5, -0) node(8) {{\,So output $V_3 \rightarrow V_4$.\,}};  
   \draw[-] (3) -- (1); 
    \draw[-] (1) -- (2); 
   \draw[-] (4) -- (2);
   \draw[-latex] (3) -- (1); 
   \draw[-latex] (3) -- (4); 
    \end{tikzpicture} \\
   ~\\
		
		\begin{tikzpicture}[scale=.8, line width=0.5pt, inner sep=0.2mm, shorten >=.1pt, shorten <=.1pt]
		\draw (5, 0) node(1) [circle, draw] {{\footnotesize\,$V_1$\,}};
		\draw (6.5, 0) node(2) [circle, draw] {{\footnotesize\,$V_2$\,}};
		\draw (5, 1.5) node(3) [circle, draw] {{\footnotesize\,$V_3$\,}};
		\draw (6.5, 1.5) node(4) [circle, draw] {{\footnotesize\,$V_4$\,}};
		\draw (-10, 1) node(5) {{\,$n=1$:\,}};
		\draw (-2, 1) node(6) {{\,$P(V_1,V_3) \independent P(V_2|V_1,V_3) $ \& $P(V_2,V_3) \dependent P(V_1|V_2,V_3)$.\,}};
		\draw (-5.5, 0) node(6) {{\,So output $V_1 \rightarrow V_2$.\,}};  
		\draw[-latex] (3) -- (1); 
		\draw[-latex] (1) -- (2); 
		\draw[-] (4) -- (2);
		\draw[-latex] (3) -- (4); 
		\end{tikzpicture} \\
		
		~\\

		~~~~~~~~~~~~~~~~~~~
		\begin{tikzpicture}[scale=.8, line width=0.5pt, inner sep=0.2mm, shorten >=.1pt, shorten <=.1pt]
		\draw (5, 0) node(1) [circle, draw] {{\footnotesize\,$V_1$\,}};
		\draw (6.5, 0) node(2) [circle, draw] {{\footnotesize\,$V_2$\,}};
		\draw (5, 1.5) node(3) [circle, draw] {{\footnotesize\,$V_3$\,}};
		\draw (6.5, 1.5) node(4) [circle, draw] {{\footnotesize\,$V_4$\,}};
		\draw (-2, 1) node(6) {{\,$P(V_4,V_3) \independent P(V_2|V_4,V_3) $ \& $P(V_2,V_3) \dependent P(V_4|V_2,V_3)$.\,}};
		\draw (-5.5, 0) node(6) {{\,So output $V_4 \rightarrow V_2$.\,}};  
		\draw[-latex] (3) -- (1); 
		\draw[-latex] (1) -- (2); 
		\draw[-latex] (4) -- (2);
		\draw[-latex] (3) -- (4); 
		\end{tikzpicture} \\

	\end{center}
	\caption{An example to illustrate causal direction identification with Algorithm \ref{Alg: direction}.}
   \label{Exp: direction}
\end{figure}

	\subsection{Identifiability Conditions of CD-NOD} \label{Sec: identifiablity}
	
	In this section, we first give identifiability conditions of CD-NOD, and then we define the equivalence class that CD-NOD can identify if corresponding conditions do not hold.
	
	We make use of independent changes and orientation rules, including discovering V-structures and using orientation propagation, to determine causal directions. To fully identify the DAG, we ought to make some assumption on the change property of the distribution in wrong directions. In particular, we have the following assumption.
	\begin{Assumption}
		 If $V_i \rightarrow V_j$, with $V_i, V_j \in \mathbf{V}$, and at least one of them has a changing mechanism, then $P(V_i|V_j,\mathbf{Z})$ and $P(V_j,\mathbf{Z})$ are dependent, where $\mathbf{Z} \subseteq \mathbf{V} \backslash \{V_l,V_k\}$ is a minimal deconfounding set of $(V_i, V_j)$.
		 \label{Ass: module independence}
	\end{Assumption}
    This assumption can be viewed as faithfulness on another level. Specifically, if $V_i \rightarrow V_j$, $P(V_i|V_j,\mathbf{Z})$ and $P(V_j,\mathbf{Z})$ involve the changing parameters both in $V_i$ and $V_j$, so they are expected to be dependent.
	Now we are ready to give identifiability conditions of causal directions, which are stated in Theorem \ref{Theorem: identifiability}.	
	   \begin{Theorem}
		Under Assumptions \ref{Ass: pseudo sufficiency}, \ref{Ass: bias}, and \ref{Ass: module independence}, the causal direction between adjacent variables $V_i, V_j \in \mathbf{V}$ is identifiable with CD-NOD, if at least one of the following three conditions is satisfied:
		\begin{itemize}
			\item [(1)] the edge between $V_i$ and $V_j$ is involved in a V-structure; 
			\item [(2)] at least one of $V_i's$ causal module and $V_j's$ causal module changes, and if both, the causal modules of $V_i$ and $V_j$ are independent;
			\item [(3)] there exists an edge incident to one and only one of the variables in $\{V_i, V_j\}$.
		\end{itemize}		
		\label{Theorem: identifiability}
	\end{Theorem}
  
   Please refer to Appendix C for detailed proofs. Note that CD-NOD does not require hard restrictions on the functional class of causal mechanisms and data distributions. Combined with Theorem \ref{theo1}, we know that for any pair of adjacent variables in graph $G$, if at least one of the three conditions in Theorem \ref{Theorem: identifiability} holds, then the whole causal structure is identifiable. It is given in the following corollary.
   
    \begin{Corollary}
 	Under Assumptions \ref{Ass: pseudo sufficiency}, \ref{Ass: bias}, and \ref{Ass: module independence}, the whole causal graph is identifiable, if any pair of adjacent variables $V_i$ and $V_j$ in graph $G$ satisfies at least one of the following three conditions:
 	\begin{itemize}
 		\item [(1)] the edge between $V_i$ and $V_j$ is involved in a V-structure; 
 		\item [(2)] at least one of $V_i's$ causal module and $V_j's$ causal module changes, and if both, the causal modules of $V_i$ and $V_j$ are independent;
 		\item [(3)] there exists an edge incident to one and only one of the variables in $\{V_i, V_j\}$.
 	\end{itemize}
 \end{Corollary}   
\begin{proof}
 From Theorem \ref{theo1}, we know that the causal skeleton is identifiable. In addition, from Theorem \ref{Theorem: identifiability}, we know that for any pair of adjacent variables, if at least one of the three conditions is satisfied, then the direction is identifiable. Therefore, the whole causal structure is identifiable.
\end{proof}	
    For a pair of adjacent variables $V_i$ and $V_j$, if none of the conditions in Theorem  \ref{Theorem: identifiability} holds, we may not be able to determine the causal direction between them. Thus, in this case, we cannot derive a fully identified causal graph but an equivalence class. Note that the causal skeleton is identifiable, as shown by Theorem \ref{theo1}, regardless of these conditions. Below we give a formal definition of equivalence class with CD-NOD.
    \begin{Definition}(CD-NOD Equivalence Class)
    	Let $G = (\mathbf{V}, \mathbf{E})$ and $G' = (\mathbf{V}, \mathbf{E}')$ be two DAGs over the same set of variables $\mathbf{V}$. $G$ and $G'$ are called CD-NOD equivalent, if and only if they satisfy the following properties.
    	\begin{itemize}
    		\item [(1)] $G$ and $G' $ have the same causal skeleton.
    		\item[(2)] For any pair of adjacent variables $V_i, V_j \in \mathbf{V}$, if it satisfies at least one of the conditions given in Theorem \ref{Theorem: identifiability}, then the causal direction between $V_i$ and $V_j$ is the same in $G$ and $G'$.
    	\end{itemize} 	
    \end{Definition}

    \paragraph{Remark:} Another important factor that needs to be considered in practical problems is computational complexity. In CD-NOD phase I, the computational complexity of each kernel-based (conditional) independence test is $O(N^3)$, where $N$ is the number of samples of each variable. The computational complexity of the search procedure with PC is bounded by $\frac{(m+1)^2 m^{k-1}}{(k-1)!}$, where $m$ is the number of observed variables, and $k$ is the maximal degree of variables in $\mathbf{V} \cup C$.  In CD-NOD phase II, the computational complexity of each dependence measure, with normalized HSIC, is $O(N^3)$.
	  
	 \section{CD-NOD Phase III: Nonstationary Driving Force Estimation} \label{Sec: driving force}
	  In this section, we focus on the visualization of \emph{how} causal module $P({V_i\,|\,\mathrm{PA}^i}, C)$ changes, i.e., where the changes occur, how fast it changes, and how to visualize the changes. We assume that we already know the causal structure and know which causal modules are changing (see Algorithms \ref{rs_causal_discovery}, \ref{Alg: invariance}, and \ref{Alg: direction}). 
	  
	  In the parametric case, if we know which parameters of the causal model $\mathrm{PA}^i \rightarrow V_i$ are changing, e.g., the mean of a direct cause, the coefficients in a linear SEM, then we can estimate such parameters for different values of $C$ and see how they change. However, such knowledge is usually not available, and for the sake of flexibility, it is generally better to model causal processes nonparametrically. Therefore, it is desirable to develop a general nonparametric procedure to capture the change of causal modules.
	  
	  We aim to find a low-dimensional and interpretable mapping of $P({V_i\,|\,\mathrm{PA}^i},C)$ which captures its nonstationarity in a nonparametric way:
	  \begin{equation} \label{Eq:def_lambda}
	  \lambda_i(C) = h_i(P({V_i\,|\,\mathrm{PA}^i},C)),
	  \end{equation}
	  where $h_i$ is a nonlinear function, mapping the conditioning distribution $P({V_i\,|\,\mathrm{PA}^i},C)$ to a low dimensional representation $\lambda_i(C)$. We call $\lambda_i(C)$ the nonstationary driving force of $P({V_i\,|\,\mathrm{PA}^i},C)$. This formulation is rather general: any identifiable parameters in $P({V_i\,|\,\mathrm{PA}^i},C)$ can be expressed this way, and in the nonparametric case, $\lambda_i(C)$ can be seen as a statistic to summarize changes in $P({V_i\,|\,\mathrm{PA}^i},C)$ along with different values of $C$. If $P({V_i\,|\,\mathrm{PA}^i},C)$ does not change along with $C$, then $\lambda_i(C)$ remains constant; otherwise, $\lambda_i(C)$ is intended to capture the variability of $P({V_i\,|\,\mathrm{PA}^i},C)$ across different values of $C$.
	  
	  Now there are two problems to solve. One is that given only observed data, how we represent the conditional distributions conveniently. The other is what method to use to enable $\lambda_i(C)$ to capture the variability in the conditional distribution along with $C$.
	  
	  For the former problem, we represent changing conditional distributions by kernel embedding as that given in Proposition \ref{embedding}, which is readily achievable. For the latter one, we use kernel principle component analysis (KPCA) to capture its variability along with $C$ and accordingly propose a method called Kernel Nonstationary Visualization (KNV) of causal modules. In the following, for conciseness, we will use $X$ and $Y$ to denote $V_i$ and $\mathrm{PA}^i$, respectively.
	  
	  \paragraph{Nonstationary Driving Force Estimation as Eigenvalue Decomposition Problems}
	  We use the estimated kernel embedding of distributions, $\hat{\tilde{\mu}}_{\underline{Y},X|{C=c_n}}$ $(n = 1,\cdots,N)$, derived from Proposition \ref{embedding} as the input, and aim to find $\hat{\lambda}(C)$ as a (linear or nonlinear) transformation of $\tilde{\mu}_{\underline{Y},X|{C=c_n}}$, to capture its variability across different values of $C$. This can be achieved by exploiting KPCA techniques~\citep{Scholkopf98}, which computes principal components in kernel spaces of the input. 
	  
	  To perform KPCA, we need to know the Gram matrix of $\hat{\tilde{\mu}}_{\underline{Y},X|{C=c_n}}$, which has already been estimated in Section \ref{Sec: Kernel Embedding}. We use  $M_{\underline{Y}X}^l$ to represent the Gram matrix of $\hat{\tilde{\mu}}_{\underline{Y},X|{C}}$ with a linear kernel, and $M^\mathcal{G}_{\underline{Y}X}$ the Gram matrix of $\hat{\tilde{\mu}}_{\underline{Y},X|{C}}$ with a Gaussian kernel. 
	  Then $\hat{\lambda}(C)$ can be found by performing eigenvalue decomposition on the above Gram matrix, $M_{\underline{Y}X}^l$ or $M_{\underline{Y}X}^g$; for details please see~\cite{Scholkopf98}. In practice, one may take the first few eigenvectors which capture most of the variance. 
	  
	  We can see that with our methods, we do not need to explicitly learn the high-dimensional kernel embedding $\tilde{\mu}_{\underline{Y},X|{C=c}}$ for each $c$. With the kernel trick, the final Gram matrix can be represented by $N \times N$ kernel matrices directly. Then the nonstationary driving force $\hat{\lambda}(C)$ can be estimated by performing eigenvalue decomposition on the Gram matrix.	  
	  Algorithm \ref{KNV} summarizes the proposed KNV method. 
	  
	  \begin{algorithm}
	  	\caption{KNV of Causal Modules $P(Y|X,C)$}
	  	\begin{enumerate}
	  		\item \textbf{Input:} $N$ observations of $X$ and $Y$.
	  		
	  		\item Calculate Gram matrix $M_{\underline{Y}X}$ (see Eq.~\ref{Eq:Linear2} for linear kernels and Eq.~\ref{Eq:Gg2} for Gaussian kernels). 
	  		
	  		\item Find $\hat{\lambda}(C)$ by directly feeding Gram matrix $M_{\underline{Y}X}$ to KPCA. That is, perform eigenvalue decomposition on $M_{\underline{Y}X}$ to find the nonlinear principal components $\hat{\lambda}(C)$, as in Section 4.1 of~\cite{Scholkopf98}. In practice, we may take the first few eigenvectors which capture most of the variance. 
	  		
	  		\item \textbf{Output:} the estimation of nonstationary driving force $\hat{\lambda}(C)$ of $P(Y|X,C)$.
	  	\end{enumerate}
	  	\label{KNV}
	  \end{algorithm}

\section{Extensions of CD-NOD} \label{Sec: extension}
In this section, we show how to extend CD-NOD to deal with more general scenarios. For example, we extend CD-NOD to allow both time-varying instantaneous and lagged causal relationships. We additionally discuss whether and how distribution shifts help for causal discovery in the presence of stationary confounders.  We then give a procedure, which leverages both CD-NOD and approaches based on constrained functional causal models, to extend the generality of the proposed method.

    \subsection{With Time-Varying Instantaneous and Lagged Causal Relationships} \label{Sec: Lagged}
     In many scenarios, e.g., dynamic systems with insufficient time resolution, there may exist both instantaneous and time-lagged causal relationships over the measured data \citep{Aapo10_VAR,Subsampling_ICML15}. In this section, we extend CD-NOD proposed in Section \ref{Sec:method} to recover both time-varying instantaneous and lagged causal relationships and identify changing causal modules.
     
     Suppose that there are $m$ observed processes $\mathbf{V}(t) = (V_1(t),\cdots,V_m(t))^{\text{T}}$ with $t = 1,\cdots,T$, where there exist both time-varying instantaneous and lagged causal relations over these $m$ processes. Further denote the largest time lag by $P$. To efficiently recover both instantaneous and time-lagged causal relations, we reorganize the set of variables $\mathbf{V}$ into $\tilde{\mathbf{V}} = \tilde{\mathbf{V}}^{(1)}  \cup \cdots \cup \tilde{\mathbf{V}}^{(P+1)}$, with 
     \begin{equation*}
     \begin{array}{lllll}
     \tilde{\mathbf{V}}^{(1)} & = \big \{  V_1^{(1)}, & V_2^{(1)}, & \cdots , & V_m^{(1)} \big \},\\
     \tilde{\mathbf{V}}^{(2)} & = \big \{  V_1^{(2)}, & V_2^{(2)}, & \cdots , & V_m^{(2)} \big \},\\
     \cdots & \cdots\\
     \tilde{\mathbf{V}}^{(P+1)} & = \big \{  V_1^{(P+1)}, & V_2^{(P+1)}, & \cdots , & V_m^{(P+1)} \big \},\\
     \end{array}
     \end{equation*}
     where $V_i^{(k)} = \big ( V_i(k), V_i(k+1),\cdots, V_i(T-P+k-1) \big)$, indicating the $i$th process from the $k$th time point to the $(T-P+k-1)$th time point, for $i = 1,\cdots,m$ and $k = 1,\cdots,P+1$.
     After reorganization, in total there are $m*(P+1)$ variables. We constrain that the future cannot cause the past; i.e., variables in $ \tilde{\mathbf{V}}^{(i)}$ cannot cause variables in $ \tilde{\mathbf{V}}^{(j)}$ when $i>j$.
     
     Figure \ref{Fig: ins_lagged} gives an illustration of a two-variable case with time lag $P=1$. Figure \ref{Fig: ins_lagged}(a) gives the repetitive causal graph over two processes $\mathbf{V}(t) = (V_1(t),V_2(t))^{\text{T}}$ with $t = 1,\cdots,T$. Figure \ref{Fig: ins_lagged}(b) gives the unit causal graph over the reorganized set of variables $\tilde{\mathbf{V}} = \tilde{\mathbf{V}}^{(1)} \cup \tilde{\mathbf{V}}^{(2)}$. In this case, we aim to recover the (time-varying) instantaneous causal relations between $V_1^{(2)}$ and $V_2^{(2)}$, and the (time-varying) lagged causal relations from $V_1^{(1)}$ to $V_2^{(2)}$, $V_1^{(1)}$ to $V_1^{(2)}$, and $V_2^{(1)}$ to $V_2^{(2)}$. Note that in this example, when inferring the instantaneous causal relation between $V_1^{(2)}$ and $V_2^{(2)}$, we should consider the influence of the lagged common cause $V_1^{(1)}$.
     
     We also add the surrogate variable $C$ into the causal system to characterize distribution shifts. Algorithm \ref{rs_causal_discovery_lagged} extends Algorithm \ref{rs_causal_discovery} to recover both instantaneous and lagged causal skeletons and detect changing causal modules. 
     The main procedure is as follows. We first construct a complete undirected graph over the reconstructed variable set $\tilde{\mathbf{V}}$ and $C$. Then in step 2, we detect changing modules for variables in $\tilde{\mathbf{V}}^{(P+1)}$ by testing the independence between $V_i^{(P+1)}$ and $C$ given a subset of $\tilde{\mathbf{V}}^{(P+1)} \backslash V_i^{(P+1)}$, for $i = 1,\cdots,m$. If they are independent, we remove the edge between $V_i^{(k)}$ and $C$, for $k = 1,\cdots,P+1$. 
     In step 3, we recover lagged causal relations between variables in $\tilde{\mathbf{V}}^{(P+1)}$ and those in $\tilde{\mathbf{V}}^{(P-p+1)}$, for $p = 1,\cdots,P$. For example, if $V_i^{(P+1)}$ and $V_j^{(P-p+1)}$ are independent given a subset of $\tilde{\mathbf{V}} \backslash \{V_i^{(P+1)}, V_j^{(P-p+1)}\} \cup C$, then remove the edge between $V_i^{(P+1-k)}$ and $V_j^{(P-p+1-k)}$ for $k = 0,\cdots, P-p$.
     In step 4, we estimate the instantaneous causal skeleton between $V_i^{(P+1)}$ and $V_j^{(P+1)}$ ($i \neq j$). If $V_i^{(P+1)}$ and $V_j^{(P+1)}$ are independent given a subset of $\tilde{\mathbf{V}} \backslash \{V_i^{(P+1)}, V_j^{(P+1)}\} \cup C  \cup \mathbf{Z}_{ij}$, where $\mathbf{Z}_{ij} \subseteq \{\tilde{\mathbf{V}}^{(k)}\}_{k=1}^{P}$ are the lagged common causes of $V_i^{(P+1)}$ and $V_j^{(P+1)}$, then we remove the edge between $V_i^{(k)}$ and $V_j^{(k)}$ for $k = 1,\cdots,P+1$. Note that it is important to consider lagged common causes $\mathbf{Z}_{ij}$ in this step, e.g., $V_1^{(1)}$ is the lagged common cause of $V_1^{(2)}$ and $V_2^{(2)}$ in Figure \ref{Fig: ins_lagged} (b).

     For the recovery of causal directions, lagged causal relations obey the rule that past causes future, which is obvious, while instantaneous causal directions can be inferred in the same way as that described in Algorithm \ref{Alg: invariance} and \ref{Alg: direction}.
     
     	\begin{figure} 
     	\begin{center}
     		\begin{tikzpicture}[scale=.5, line width=0.5pt, inner sep=0mm, shorten >=.1pt, shorten <=.1pt,minimum size=4mm]
     		\draw (.2, 3) node(1) [] {{\footnotesize\,$V_1(1)$\,}};
     		\draw (3.2, 3) node(2) [] {{\footnotesize\,$V_1(2)$\,}};
     		\draw (6.2, 3) node(3) [] {{\footnotesize\,$\cdots$\,}};
     		\draw (9.2, 3) node(4) [] {{\footnotesize\,$V_1(T-1)$\,}};
     		\draw (12.8, 3) node(5) [] {{\footnotesize\,$V_1(T)$\,}};
     		\draw[-latex] (1) -- (2); 
     		\draw[-latex] (2) -- (3); 
     		\draw[-latex] (3) -- (4); 
            \draw[-latex] (4) -- (5); 

            \draw (.2, 0) node(6) [] {{\footnotesize\,$V_2(1)$\,}};
            \draw (3.2, 0) node(7) [] {{\footnotesize\,$V_2(2)$\,}};
            \draw (6.2, 0) node(8) [] {{\footnotesize\,$\cdots$\,}};
            \draw (9.2, 0) node(9) [] {{\footnotesize\,$V_2(T-1)$\,}};
            \draw (12.8, 0) node(10) [] {{\footnotesize\,$V_2(T)$\,}};
            \draw[-latex] (6) -- (7); 
            \draw[-latex] (7) -- (8); 
            \draw[-latex] (8) -- (9); 
            \draw[-latex] (9) -- (10); 
            
            \draw[-latex] (1) -- (6); 
            \draw[-latex] (2) -- (7); 
            \draw[-latex] (9) -- (10); 
            \draw[-latex] (1) -- (7); 
            \draw[-latex] (4) -- (9); 
            \draw[-latex] (5) -- (10); 
            \draw[-latex] (4) -- (10); 
  		\end{tikzpicture}~~~~~~~~~~~~~~
  		\begin{tikzpicture}[scale=.5, line width=0.5pt, inner sep=0mm, shorten >=.1pt, shorten <=.1pt,minimum size=4mm]
  		\draw (.2, 3) node(1) [] {{\footnotesize\,$V_1^{(1)}$\,}};
  		\draw (3.2, 3) node(2) [] {{\footnotesize\,$V_1^{(2)}$\,}};
  		\draw[-latex] (1) -- (2); 
  		
  		\draw (.2, 0) node(3) [] {{\footnotesize\,$V_2^{(1)}$\,}};
  		\draw (3.2, 0) node(4) [] {{\footnotesize\,$V_2^{(2)}$\,}};
  		\draw[-latex] (3) -- (4); 
  		
  		\draw[-latex] (1) -- (3); 
  		\draw[-latex] (2) -- (4); 
  	    \draw[-latex] (1) -- (4); 
  		\end{tikzpicture}
  		\\~~~~~~~~~~~~~~(a) Repetitive causal graph. ~~~~~~~~~~~~~~~~~~~~~(b) Unit causal graph.
     	\end{center} 
    \caption{A two-variable case with both instantaneous and one-lagged ($P=1$) causal relationships. (a) Repetitive causal graph. (b) Unit causal graph.}
    \label{Fig: ins_lagged}
     \end{figure}
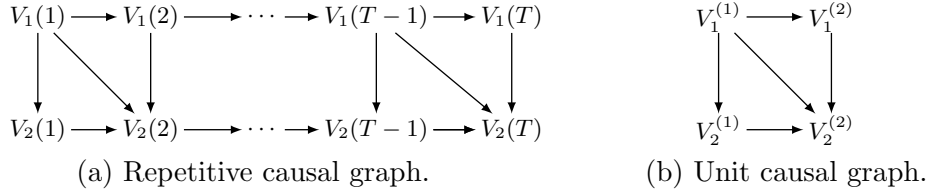

      \begin{algorithm}[htp!]
     \caption{Detection of Changing  Modules and Recovery of Causal Skeleton with both (Time-Varying) Instantaneous and Lagged Causal Relationships}
     \begin{enumerate}
     \item Build a complete undirected graph $U_{\mathcal{G}}$ over the variable set $ \tilde{\mathbf{V}} \cup C$.
     \item (\textit{Detection of changing  modules}) For each $i$ $ (i = 1,\cdots,m)$, test for the marginal and conditional independence between $V_i^{(P+1)}$ and $C$. If they are independent given a subset of $\tilde{\mathbf{V}}^{(P+1)} \backslash V_i^{(P+1)}$, remove the edge between $V_i^{(P+1)}$ and $C$ in $ U_{\mathcal{G}}$; at the same time, remove the edge between $V_i^{(k)}$ $(k = 1,\cdots,P)$ and $C$.
     \item (\textit{Recovery of lagged causal skeleton}) For the $p$th $(p = 1,\cdots,P)$ lagged causal relationships, test for the marginal and conditional independence between the variables in $\mathbf{\tilde{V}}^{(P+1)}$ and those in $\mathbf{\tilde{V}}^{(P-p+1)}$. Particularly, for $V_i^{(P+1)}$ and $V_j^{(P-p+1)}$, if they are independent given a subset of $\tilde{\mathbf{V}} \backslash \{V_i^{(P+1)},V_j^{(P-p+1)}\} \cup C$, remove the edge between $V_i^{(P+1)}$ and $V_j^{(P-p+1)}$; at the same time, remove the edge between $V_i^{(P-k+1)}$ and $V_j^{(P-p+1-k)}$, for $k =0,\cdots, P-p$.
     \item (\textit{Recovery of instantaneous causal skeleton}) For every $i\neq j$, test for the marginal and conditional independence between $V_{i}^{(P+1)}$ and $V_{j}^{(P+1)}$. Let $\mathbf{Z}_{ij} \subseteq \{\tilde{\mathbf{V}}^{(k)}\}_{k=1}^{P}$ be the lagged common causes of $V_i^{(P+1)}$ and $V_j^{(P+1)}$, which can be derived from the output of Step 3 above.
     If $V_i^{(P+1)}$ and $V_j^{(P+1)}$ are independent given a subset of $ \tilde{\mathbf{V}}^{(P+1)} \backslash \{V_i^{(P+1)},V_j^{(P+1)}\} \cup  C \cup \mathbf{Z}_{ij}$, remove the edge between $V_{i}^{(P+1)}$ and $V_{j}^{(P+1)}$ in $ U_{\mathcal{G}}$; at the same time, remove the edge between $V_{i}^{(k)}$ and $V_{j}^{(k)}$, for $k = 1,\cdots,P$.
     \end{enumerate}
     \label{rs_causal_discovery_lagged}
     \end{algorithm}

    \subsection{With Stationary Confounders} \label{Sec: confounder}
    In this section, we discuss the case when there exist stationary confounders; that is, the distribution of the confounder is fixed. We find that distribution shifts over observed variables may still help estimate causal directions in such a case. 
    
    To show how distribution shifts help, we analyze two-variable cases in Figure \ref{fig:illust_confounder}, where $V_1$ and $V_2$ are observed variables, and $Z$ is a hidden variable which influences both $V_1$ and $V_2$. In  Figure \ref{fig:illust_confounder}(a) \& (b), the causal module of $V_1$ changes, with changing parameter $\theta_1(C)$, while causal modules of $V_2$ and $Z$ are fixed. In such a case, we have $V_1 \dependent V_2$ and $V_1 \dependent V_2 | C$ in both graphs. Furthermore, in (a) we have $C \dependent V_2$ and $C  \dependent V_2 | V_1$, while in (b) we have $C  \independent V_2$ and $C  \dependent V_2 | V_1$. We can see that distribution shift helps to distinguish between the two graphs when only one causal module changes.

    In Figure \ref{fig:illust_confounder}(c) \& (d), both the causal modules of $V_1$ and $V_2$ change, and they change independently, while the distribution of $Z$ is fixed. In these two graphs, we have $V_1 \dependent V_2$ and $V_1 \dependent V_2 | C$, and we have $P(V_1) \dependent P(V_2|V_1)$ and $P(V_2) \dependent P(V_1|V_2)$, for the reason given below. Hence, graphs in (c) and (d) have the same independence over $\mathbf{V} \cup C$ and between distribution modules, and thus we cannot distinguish between (c) and (d) without further information. 
    
    The reason that $P(V_1) \dependent P(V_2|V_1)$ in (c) is as follows. In (c), $P(V_2|V_1) $ can be represented as: $P(V_2|V_1) = \int P(V_2|V_1,Z) P(Z|V_1) \,dZ$, where $P(V_2|V_1,Z)$ contains the information of $\theta_2(C)$, and $P(Z|V_1)$ contains the information of $\theta_1(C)$, so $P(V_2|V_1)$ also contains the information of $\theta_1(C)$. Since both $P(V_2|V_1)$ and $P(V_1)$ contain the information of $\theta_1(C)$, generally speaking, $P(V_2|V_1)$ and $P(V_1)$ are not independent. Therefore, although, according to the causal module independence, $P(V_1|Z) \independent P(V_2|V_1,Z)$,  $P(V_1)$ and $P(V_2|V_1)$ are dependent. Similarly, we can derive that $P(V_2) \dependent P(V_1|V_2)$ in (d).

    	\begin{figure} 
    	\begin{center}
    		\begin{tikzpicture}[scale=.5, line width=0.5pt, inner sep=0.2mm, shorten >=.1pt, shorten <=.1pt]
    		\draw (.2, 0) node(1) [circle, draw] {{\footnotesize\,$V_1$\,}};
    		\draw (3.2, 0) node(2) [circle, draw] {{\footnotesize\,$V_2$\,}};
    		\draw (0, 1.9) node(3)  {{\footnotesize\,$\theta_1(C)$\,}};
    		\draw (1.7, 2.2) node(5) [circle, dashed,draw] {{\footnotesize\,$Z$\,}};
    		\draw[-latex] (1) -- (2); 
    		\draw[-latex] (3) -- (1); 
    		\draw[-latex] (5) -- (1);
    		\draw[-latex] (5) -- (2); 
    		\end{tikzpicture}~~~~~~~~~~~~~~~~~~~~
    		\begin{tikzpicture}[scale=.5, line width=0.5pt, inner sep=0.2mm, shorten >=.1pt, shorten <=.1pt]
    		\draw (.2, 0) node(1) [circle, draw] {{\footnotesize\,$V_1$\,}};
    		\draw (3.2, 0) node(2) [circle, draw] {{\footnotesize\,$V_2$\,}};
    		\draw (0, 1.9) node(3)  {{\footnotesize\,$\theta_1(C)$\,}};
    		\draw (1.7, 2.2) node(5) [circle, dashed,draw] {{\footnotesize\,$Z$\,}};
    		\draw[-latex] (2) -- (1); 
    		\draw[-latex] (3) -- (1); 
    		\draw[-latex] (5) -- (1); 
    		\draw[-latex] (5) -- (2); 
    		\end{tikzpicture}
    		\\~~~~~~~~~~~~~~~(a)~~~~~~~~~~~~~~~~~~~~~~~~~~~~~~~~~~~~(b)~~~~~~~~~~~~~\\  	
    		~\\	
    		~~~\begin{tikzpicture}[scale=.5, line width=0.5pt, inner sep=0.2mm, shorten >=.1pt, shorten <=.1pt]
    		\draw (.2, 0) node(1) [circle, draw] {{\footnotesize\,$V_1$\,}};
    		\draw (3.2, 0) node(2) [circle, draw] {{\footnotesize\,$V_2$\,}};
    		\draw (0, 1.9) node(3)  {{\footnotesize\,$\theta_1(C)$\,}};
    		\draw (3.4, 1.9) node(4)  {{\footnotesize\,$\theta_2(C)$\,}};
    		\draw (1.7, 2.2) node(5) [circle, dashed,draw] {{\footnotesize\,$Z$\,}};
    		\draw[-latex] (1) -- (2); 
    		\draw[-latex] (3) -- (1); 
    		\draw[-latex] (5) -- (1);
    		\draw[-latex] (5) -- (2);
    		\draw[-latex] (4) -- (2);  
    		\end{tikzpicture}~~~~~~~~~~~~~~~~~~
    		\begin{tikzpicture}[scale=.5, line width=0.5pt, inner sep=0.2mm, shorten >=.1pt, shorten <=.1pt]
    		\draw (.2, 0) node(1) [circle, draw] {{\footnotesize\,$V_1$\,}};
    		\draw (3.2, 0) node(2) [circle, draw] {{\footnotesize\,$V_2$\,}};
    		\draw (0, 1.9) node(3)  {{\footnotesize\,$\theta_1(C)$\,}};
    		\draw (3.4, 1.9) node(4)  {{\footnotesize\,$\theta_2(C)$\,}};
    		\draw (1.7, 2.2) node(5) [circle, dashed,draw] {{\footnotesize\,$Z$\,}};
    		\draw[-latex] (2) -- (1); 
    		\draw[-latex] (3) -- (1); 
    		\draw[-latex] (5) -- (1); 
    		\draw[-latex] (5) -- (2); 
    		\draw[-latex] (4) -- (2); 
    		\end{tikzpicture}
    		\\~~~~~~~~~~~~~~~(c)~~~~~~~~~~~~~~~~~~~~~~~~~~~~~~~~~~~~(d)~~~~~~~~~~~~~
    	\end{center} 
    	\caption{Two-variable cases with stationary confounders. (a) $V_1 \rightarrow V_2$ with the changing causal module of $V_1$. (b) $V_1 \leftarrow V_2$ with the changing causal module of $V_1$. (c) $V_1 \rightarrow V_2$ with the changing causal modules of both $V_1$ and $V_2$. (d) $V_1 \leftarrow V_2$ with the changing causal modules of both $V_1$ and $V_2$.} \label{fig:illust_confounder}
    \end{figure}
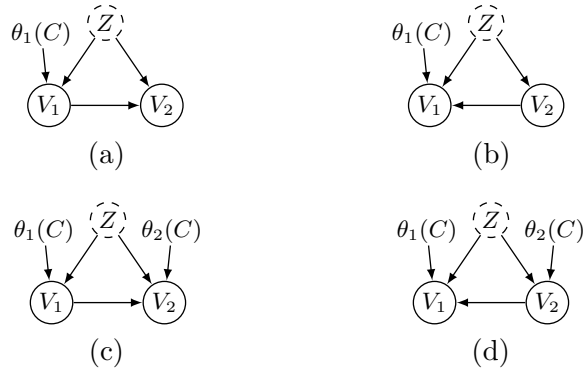

    \subsection{Combination with Approaches Based on Constrained Functional Causal Models} \label{Sec: FCM}
    For a pair of adjacent variables $V_i$ and $V_j$, if none of the identifiability conditions given in Theorem \ref{Theorem: identifiability} are satisfied, the causal direction between $V_i$ and $V_j$ is undetermined with CD-NOD. In such a case, to infer the causal direction between $V_i$ and $V_j$, one possible way is to further leverage the approaches based on constrained functional causal models, e.g., the linear non-Gaussian model \citep{Shimizu06}, the nonlinear additive noise model \citep{Hoyer09,Zhang09_additive}, or the post-nonlinear model \citep{Zhang06_iconip, Zhang_UAI09}. To determine the causal direction with functional causal model-based approaches, we have two scenarios to consider, depending on whether $P(V_i, V_j)$ changes or not.
    \begin{itemize}
    	\item [1.] The joint distribution $P(V_i, V_j)$ is fixed. In this scenario, we can fit a constrained functional causal model, e.g., the nonlinear additive noise model, as usual. Let $\mathbf{Z} \subseteq \mathbf{V} \backslash \{V_i,V_j\}$ be a minimal deconfounding set of $V_i$ and $V_j$. We 
    	  \begin{itemize}
    	  	\item first assume $V_i \rightarrow V_j$ and fit an additive noise model
    	  	   \begin{equation*}
    	  	     V_j = f_j(V_i, \mathbf{Z}) + E_j,
    	  	   \end{equation*}
    	  	   and test the independence between estimated residual $\hat{E}_j$ and hypothetical causes $(V_i, \mathbf{Z})$,
    	  	 \item and then assume $V_j \rightarrow V_i$ and fit an additive noise model in this direction
    	  	   \begin{equation*}
    	  	   V_i = f_i(V_j, \mathbf{Z}) + E_i,
    	  	  \end{equation*}
    	  	 and test the independence between estimated residual $\hat{E}_i$ and hypothetical causes $(V_j, \mathbf{Z})$.
    	  \end{itemize}
          We choose the direction in which the independence between estimated residual and hypothetical causes holds.  
          
    	\item [2.] The joint distribution $P(V_i, V_j)$ changes. The fact that for $V_i$ and $V_j$, none of the conditions in Theorem \ref{Theorem: identifiability} is satisfied implies that neither $V_i$ nor $V_j$ is adjacent to $C$. It further infers that there are some variables in $\mathbf{V}$ with changing distributions and influencing $V_i$ or $V_j$. Let $\mathbf{Z}_i \subseteq \mathbf{V} \backslash \{V_i, V_j\} $ represent a set of variables which directly influence $V_i$, and satisfies that $\forall Z_i \in \mathbf{Z}_i$, its distribution $P(Z_i)$ changes. Similarly, we denote $\mathbf{Z}_j$ as parents of $V_j$ that have changing distributions. Let $\mathbf{Z} \subseteq \mathbf{V} \backslash \{V_i,V_j\}$ be a minimal deconfounding set of $V_i$ and $V_j$, and let $\tilde{\mathbf{Z}} \subseteq \mathbf{Z}$, with $P(\tilde{\mathbf{Z}})$ fixed and $\forall Z \in \mathbf{Z} \backslash \tilde{\mathbf{Z}}$, $P(Z)$ changes. Note that it is easy to infer that $\mathbf{Z}_i \subseteq \mathbf{Z}$ and $\mathbf{Z}_j \subseteq \mathbf{Z}$. We
    	  \begin{itemize}
    	  	\item first fit an additive noise model by assuming $V_i \rightarrow V_j$
    	  	    \begin{equation*}
    	  	   V_j = f_j(V_i, \tilde{\mathbf{Z}}, \mathbf{Z}_j) + E_j,
    	  	   \end{equation*}
    	  	   and test the independence between estimated residual $\hat{E}_j$ and hypothetical causes $(V_i, \mathbf{Z}_j, \tilde{\mathbf{Z}})$,
    	  	\item and then assume $V_j \rightarrow V_i$ and fit an additive noise model accordingly
    	  	 \begin{equation*}
    	  	 V_i = f_i(V_j, \tilde{\mathbf{Z}},  \mathbf{Z}_i) + E_i,
    	  	 \end{equation*}
    	  	 and test the independence between estimated residual $\hat{E}_i$ and hypothetical causes $(V_j, \mathbf{Z}_i, \tilde{\mathbf{Z}})$.
      	 \end{itemize}
       We choose the direction in which the independence between estimated residual and hypothetical causes holds. Note that in case 2, we cannot ignore the changing factors $\mathbf{Z}_i$ and $\mathbf{Z}_j$ in their corresponding functional causal models, because functional causal model-based approaches assume a fixed causal model. If we do not consider $\mathbf{Z}_i$ or $\mathbf{Z}_j$, then the corresponding functions $f_i$ and $f_j$, respectively, may change across data sets or over time.

    \end{itemize}
 If there are unoriented edges to $V_i$ or $V_j$, the potential minimal deconfounding set $\tilde{\mathbf{Z}}$, and potential changing influences $\mathbf{Z}_i$ and $\mathbf{Z}_j$ are chosen in a heuristic way, similar to that in Algorithm \ref{Alg: direction}.

    \section{Relations between Heterogeneity/Nonstationarity and Soft Intervention in Causal Discovery} \label{Sec: Soft_Inter}
    
    In this section, we show that heterogeneity/nonstationarity and soft intervention are tightly related in causal discovery, and the former may be more effective for causal discovery than the latter. The definition of soft intervention is as follows.
    \begin{Definition}[Soft Intervention \citep{Frederick_softInter}]
    	Given a set of measured variables $\mathbf{V}$, a soft intervention $I_p$ on a variable $S \in \mathbf{V}$ is an intervention on $S$ that satisfies the following constraints:
    	when $I_p = 1$, $I_p$ does not make $S$ independent of their causes in $\mathbf{V}$; i.e., it does not break any edges that are incident to S. The causal module $P(S|\mathrm{PA}^S)$ is replaced by $P^*(S|\mathrm{PA}^S, I_p = 1)$, where 
    	\begin{equation}
    	  P^*(S|\mathrm{PA}^S, I_p = 1) \neq P(S|\mathrm{PA}^S, I_p = 0). 
    	  \label{Eq: SI}
    	\end{equation}
    	 Otherwise, all terms remain unchanged.
    	 
    	 Additionally, $I_p$ satisfies the following properties:
    	 \begin{itemize}[itemsep=0.2pt,topsep=0.2pt]
    	 	\item $I_p$ is a variable with two states, $I_p = 1$ or $I_p = 0$, whose state is known.
    	 	\item When $I_p = 0$, the passive observational distribution over $S$ obtains.
    	 	\item $I_p$ is a direct cause of $S$ and only $S$.
    	 	\item $I_p$ is exogenous, that is, uncaused.
    	 \end{itemize} 
    \end{Definition}
    Note that the soft intervention does not imply any structural changes over the variables in $\mathbf{V}$; instead, it influences the manipulated probability distribution.

  Heterogeneity/nonstationarity can be seen as a consequence of soft interventions done by nature. Recall that the heterogeneous/nonstationary causal modules are defined as 
   \begin{equation*}\label{non/heter}
    V_i = f_i\big(\mathrm{PA}^i, \mathbf{g}^i(C), \theta_i(C), \epsilon_i\big),
   \end{equation*}
   where $\mathbf{g}^i(C)\subseteq \{g_l(C)\}_{l=1}^L$ denotes the set of confounders that influence $V_i$ (it is an empty set if there is no confounder behind $V_i$ and any other variable), $\theta_i(C)$ denotes the effective parameters in the model that are also assumed to be functions of $C$, and $\epsilon_i$ is a disturbance term that is independent of $C$. The distribution of $P(V_i | \mathrm{PA}^i,C)$ changes when $C$ takes different values, i.e.,
   \begin{equation}
    P(V_i | \mathrm{PA}^i, C=c) \neq P(V_i | \mathrm{PA}^i, C=c').
   \end{equation}
   This means that across the values of $C$, $c$ and $c'$, there exists a soft intervention on $V_i$.
  
   Our developed framework in causal discovery from heterogeneous/nonstationary data is more applicable than that from \citet{Frederick_softInter} in the following aspects:
   \begin{enumerate}[itemsep=0.3pt,topsep=0.3pt]
   	\item In \citet{Frederick_softInter}, $I_p$ is known and only contains two states, while in our case we can detect where it happens in an automated way, and the changes can be discrete (across domains) or continuous (over time).
   	\item In \citet{Frederick_softInter}, $I_p$ is assumed to be a direct cause to a single variable, while in our case there is no such a restriction. For example, we allow pseudo confounders $g_l(C)$ which affects several variables.
   	\item Our framework also allows edges to vanish or appear in some domains or over some time periods; that is, the structure over $\mathbf{V}$ may change, for which hard intervention is a special case, while in \citet{Frederick_softInter}, the causal structure does not change. 
   \end{enumerate}

\section{Experimental Results} \label{Sec:experiments}
   
   We applied the proposed CD-NOD (both phase I and phase II) to both synthetic and real-world data sets to learn causal graphs and identify changing causal modules, and then we learned a low-dimensional representation of changing causal modules (phase III).
   
   \subsection{Synthetic Data}
   
   \subsubsection{Setting 1}
    \label{Sec:simul}
    To show the generality of the proposed methods in causal discovery when data distributions change, we considered two types of data: multi-domain heterogeneous data and nonstationary time series. Moreover, we considered changes in both causal strength and noise variances. 
    
    For variable $V_{i}$ whose causal module changes, the $t$th data point was generated according to the following functional causal model:
    \begin{equation*}
    V_{i,t} = \sum_{V_j \in \mathrm{PA}^i} b_{ij,t} f_i(V_{j,t}) + \sigma_{i,t} \epsilon_{i,t},
    \end{equation*}
    where $\mathrm{PA}^i$ denotes the set of $V_i$'s direct causes, $V_j \in \mathrm{PA}^i$ is the $j$th direct cause of $V_i$, $b_{ij,t}$ is the varying causal strength from variable $V_j$ to $V_i$, and $\sigma_{i,t}$ is the changing parameter applied to the noise term $ \epsilon_{i,t}$. The function $f_i$ was randomly chosen from linear, cubic, tanh, sinc functions, or random mixtures of those functions. The noise term $\epsilon_{i,t}$ was randomly chosen from a uniform distribution $\mathcal{U}(-0.5,0.5)$ or a standard normal distribution $\mathcal{N}(0,1)$. 
    
    The varying causal strength $b_{ij,t}$ and varying noise's parameter $\sigma_{i,t}$ were generated in different ways for heterogeneous and nonstationary data.
    \begin{itemize}
      \item  For heterogeneous data whose data distributions change across domains, in each domain, $b_{ij,t}$ and $\sigma^2_{i,t}$ were randomly generated from uniform distributions $\mathcal{U}(0.5,2.5)$ and $\mathcal{U}(1,3)$, respectively. Note that in the same domain both $b_{ij,t}$ and $\sigma^2_{i,t}$ remain the same. 
      \item  For nonstationary data whose data distributions smoothly change over time, we generated $b_{ij,t}$ and $\sigma_{i,t}$ by sampling from a Gaussian process (GP) prior $b_{ij}(\mathbf{t}) \sim \mathcal{GP}(\vec{0},K(\mathbf{t},\mathbf{t}))$ and $\sigma_i(\mathbf{t}) \sim \mathcal{GP}(\vec{0},K(\mathbf{t},\mathbf{t}))$, respectively. We used squared exponential kernel $k(t,t') = \exp (-\frac{1}{2} \frac{|t-t'|}{\lambda^2})$, where $\lambda$ is the kernel width.
    \end{itemize}

    For variable $V_i$ whose causal module is fixed, it was generated according to a fixed functional causal model:
    \begin{equation*}
    V_{i,t} = \sum_{V_j \in \mathrm{PA}^i} b_{ij} f_i(V_{j,t}) + \epsilon_{i,t},
    \end{equation*}
    where both the causal strength $b_{ij}$ and variance of the noise term $\epsilon_{i}$ are fixed over time.

    We randomly generated acyclic causal structures according to the Erdos-Renyi model \citep{Generate_graph} with the probability of each edge 0.3. Each generated graph has 6 variables. In each graph, we randomly picked variables with changing causal modules, and the number was randomly chosen from $\{4,5,6\}$. We also considered different sample sizes. Particularly, for heterogeneous data, the sample size of each domain was chosen between 50 and 100, and the total sample size was $N = $600, 900, 1200, and 1500. For nonstationary data, we fixed the kernel with of the Gaussian prior and generated the data with sample size $N = 600, 900, 1200, \text{ and } 1500$.  In each scenario (each data type and each sample size), we generated 100 random realizations.
    
    We applied the proposed CD-NOD to identify the underlying causal structure, both causal skeleton (phase I) and causal directions (phase II), and detect changing causal modules (phase I).
    
    \paragraph{Phase I: causal skeleton identification and changing causal module detection.}
    We applied CD-NOD phase I (Algorithm 1) to identify the causal skeleton and detect changing causal modules. Specifically, for heterogeneous data, we used the domain index as the surrogate variable $C$ to capture the distribution change, while for nonstationary data, we used the time index. We used PC \citep{SGS93} as the search procedure and the kernel-based conditional independence (KCI) test to test conditional independence relationships. 
    
    We compared our approach with the original constraint-based method, which does not take into account distribution changes. For the original constraint-based method, we applied the PC search over the observed variables, combined with the KCI test. We also compared with approaches based on the minimal change principle, the identical boundaries (IB) method and the minimal changes (MC) method \citep{Ghassami18_NIPS}. Both IB and MC methods are designed for multi-domain causal discovery in linear systems. 
        
    For the KCI test, the significance level was $0.05$. We used Gaussian kernels, and the hyperparameters, e.g. the kernel width, were set with empirical values; please refer to \cite{Zhang11_KCI} for details. Since both IB and MC methods need data from multiple domains, 
    for heterogeneous data, we segmented the data according to the domain index before applying IB and MC methods, and for nonstationary data, we segmented the data into non-overlapping domains, with sample size 100 in each domain.
        
    Figure \ref{Fig: Simu_skeleton} gives the accuracy of the recovered causal skeleton from both heterogeneous (upper row) and nonstationary (lower row) data. The x-axis indicates the sample size, and the y-axis shows the accuracy of the recovered causal skeleton. We considered three accuracy measures: F$_1$ score, precision, and recall, where precision indicates the rate of spurious edges, recall the rate of missing edges, and F$_1 = 2 \cdot \frac{\text{precision} \cdot \text{recall}}{\text{precision} + \text{recall}}$. We can see that the proposed CD-NOD (phase I) gives the best F$_1$ score and precision on both heterogeneous and nonstationary data for all sample sizes, and it gives similar recall with others. Moreover, there is a slight increase in F$_1$ score and recall along with the sample size. The original constraint-based method gives a much lower precision, and hence a lower F$_1$ score as well, which means that there are many spurious edges. The reason is that it does not consider distribution shifts, as we illustrated in Figure \ref{fig:illust_C}. Both IB and MC methods give a much lower precision, probably because they are designed for linear systems, and thus the indirect nonlinear influences cannot be totally captured by variables along the pathway with linear functions.
    
   \begin{figure} 
    	\centering
    	\includegraphics[width=1\textwidth, trim={1cm 0cm 1cm 0cm},clip]{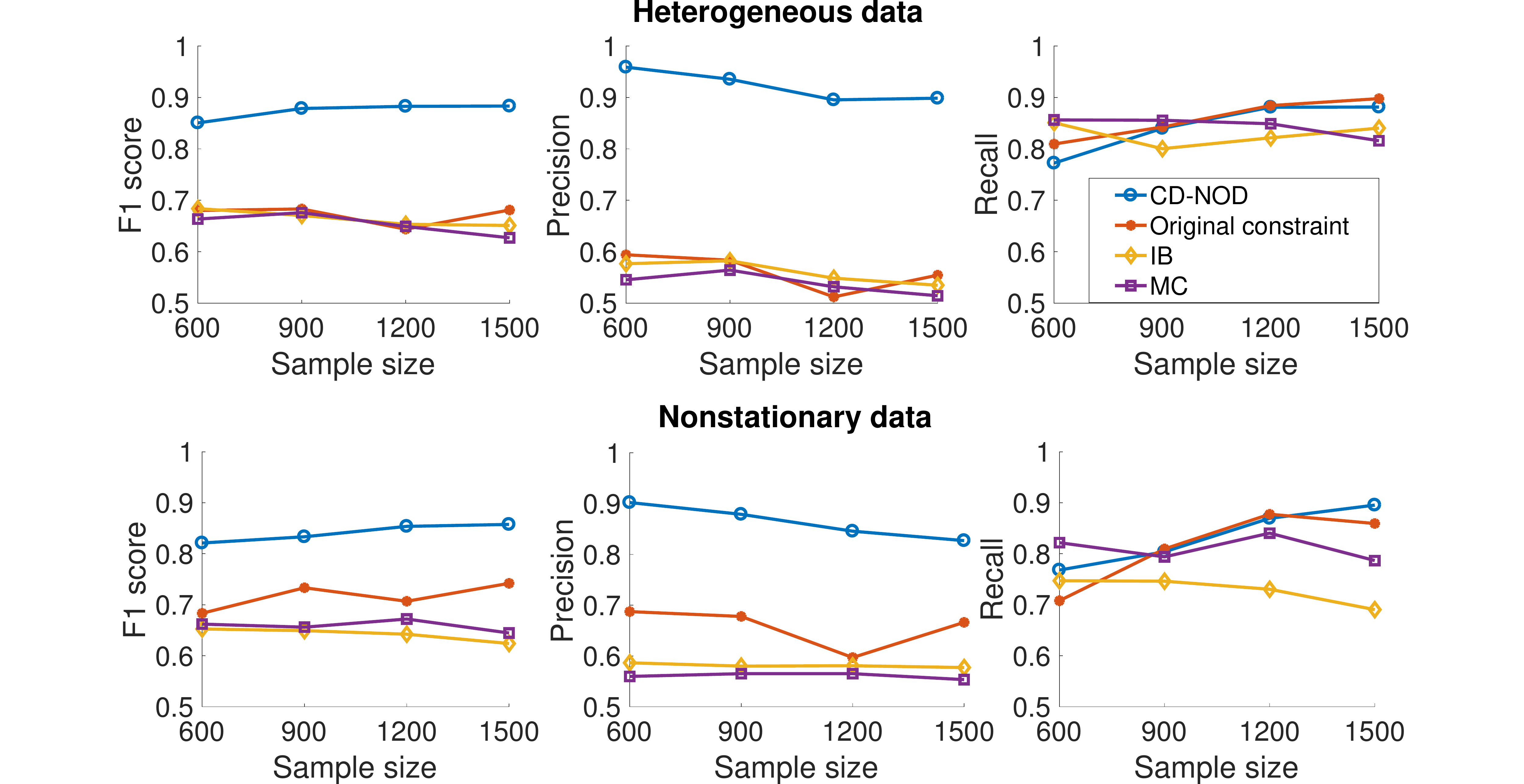}
    	\caption{Accuracy of the recovered causal skeleton from heterogeneous (upper row) and nonstationary (lower row) data. We compared three accuracy measures: F$_1$ score (left column), precision (middle column), and recall (right column). We compared our proposed CD-NOD (phase I) with original constraint-based method, IB method, and MC method.}
    	\label{Fig: Simu_skeleton}
    \end{figure}

   Table \ref{Table: Simu_table} shows the accuracy of the detected changing causal modules, measured by F$_1$ score, precision, and recall. We can see that the proposed CD-NOD performs well, and that the F$_1$ score tends to increase along with sample size in both scenarios.

    \begin{table}[htp!]
    	\centering
    	\subfloat[][Heterogeneous data]{
    		\begin{tabular}{p{1.5cm} p{1.5cm} p{1.5cm} p{1.5cm}}
    			N & F$_1$ \quad& Precision & Recall  \\ \hline
    			600  & 0.93  & 0.98  & 0.89\\ \hline
    			900 & 0.95 & 0.99   & 0.93\\ \hline
    			1200 & 0.96  & 1.00   & 0.93\\ \hline
    			1500 & 0.96  & 0.99   & 0.93\\ \hline
    		\end{tabular}
    	}
    	\hspace{.5cm}\subfloat[][Nonstationary data]{
    		\begin{tabular}{p{1.5cm} p{1.5cm} p{1.5cm} p{1.5cm}}
    			N & F$_1$ \quad& Precision & Recall \\ \hline
    			600  & 0.94  & 0.99 &  0.90 \\ \hline
    			900 & 0.96  & 0.99  & 0.94 \\ \hline
    			1200 & 0.97 & 0.99  & 0.96\\ \hline
    			1500 & 0.96  & 0.97  & 0.95\\ \hline
    		\end{tabular}
    	}
    	\caption{Accuracy of the identified changing causal modules from (a) heterogeneous data and (b) nonstationary data.}
    	\label{Table: Simu_table}
    \end{table}

    \paragraph{Phase II: causal direction identification.}
    We then applied three rules to identify causal directions, based on the skeleton learned in phase I. The three rules are
    \begin{itemize}[itemsep=0.2pt,topsep=0.2pt]    	
    	\item [(1)] generalization of invariance (Algorithm \ref{Alg: invariance}),
    	\item [(2)] independent changes between causal modules (Algorithm \ref{Alg: direction}),
    	 \item [(3)] and Meek's orientation rule \citep{Meek95}.
    \end{itemize}
    The three rules are applied in the above sequence to identify causal directions.  
    
    We compared with the window-based method \citep{Zhang17_IJCAI}, the IB method, and the MC method for direction determination. The window-based method identifies the causal direction by measuring the dependence between hypothetical causal modules; however, the kernel density estimation of modules were performed in each sliding window or in each domain, separately. For the window-based method, we used the causal skeleton derived from CD-NOD phase I. For heterogeneous data, the window size was the same as the sample size in each domain; for nonstationary data, the size of each sliding window was $100$. For the IB and MC method, the skeleton and orientations were derived simultaneously.
    
    Figure \ref{Fig: Simu_dir} shows the F$_1$ score of the recovered whole causal graph (both skeletons and directions) from heterogeneous and nonstationary data. We can see that the proposed CD-NOD gives the best F$_1$ score in all scenarios. The window-based method has second-best performance, and it performs slightly better on heterogeneous data than on nonstationary data. The reason may be that on nonstationary data, the sliding window-based method may lead to large estimation errors, especially when the causal influence varies quickly over time. The performance of IB and MC methods is worse; it is not surprising since the generated data are not linear.
    
     \begin{figure} 
    	\centering
    	\includegraphics[width=.8\textwidth, trim={0cm 0cm 0cm 0cm},clip]{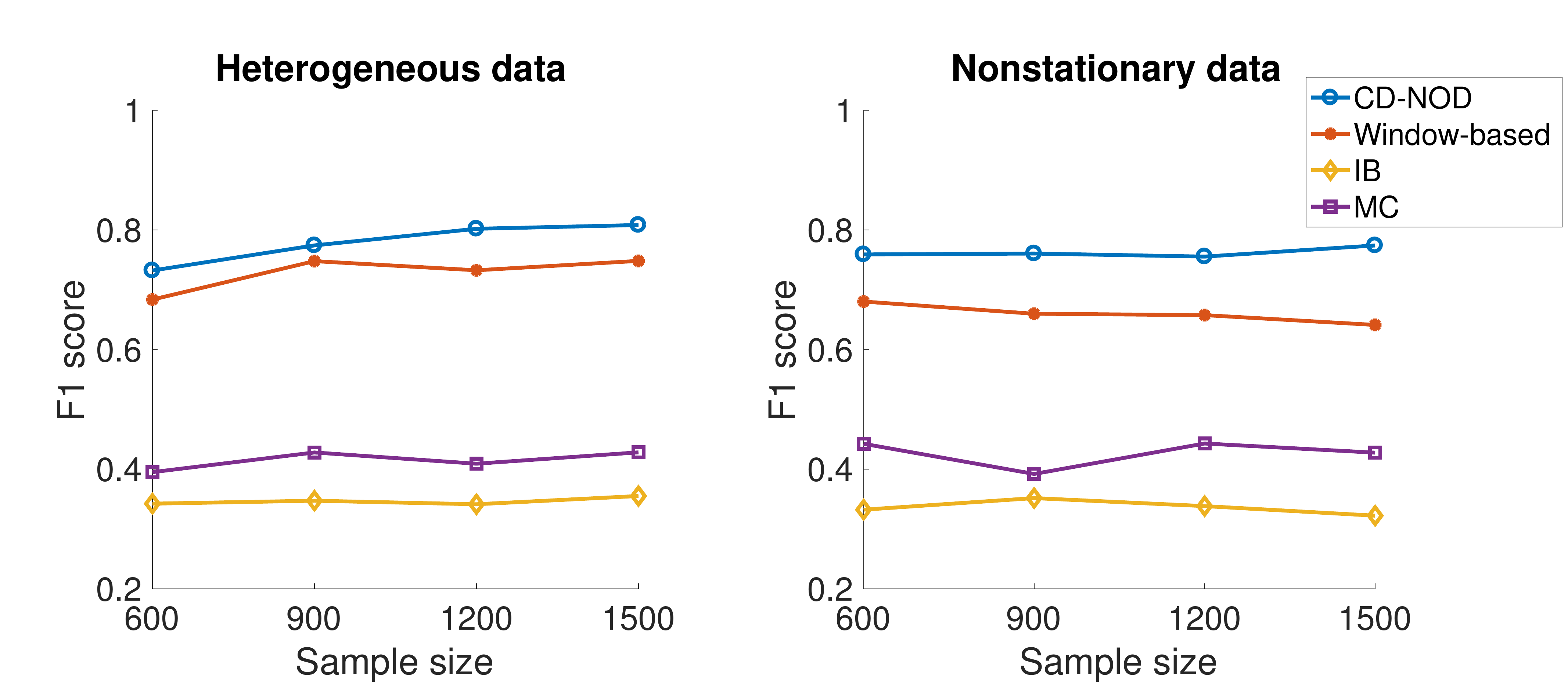}
    	\caption{Accuracy of the recovered whole causal graph (both skeletons and directions) from both heterogeneous and nonstationary data. We compared our proposed CD-NOD with window-based method, IB method, and MC method.}
    	\label{Fig: Simu_dir}
    \end{figure}

    \subsubsection{Setting 2}
    To clearly show the efficacy of Algorithm \ref{Alg: direction} in causal direction identification (without being affected by the accuracy from causal skeleton identification and the other two orientation rules), we generated another synthetic data set. Particularly, we generated fully connected acyclic graphs. For each variable $V_i$, all its causal strength $b_{ij}$ and noise's parameter $\sigma_{i}$ change, either across domains or over time. We considered different graph sizes $m = 2, 4, \text{ and } 6$ and different sample sizes $N = 600, 900, 1200, \text{ and } 1500$.
    
    To assess the performance of causal direction identification, we assumed that the causal skeleton is given (fully connected), and we inferred causal directions with Algorithm \ref{Alg: direction}, which exploits the independence between causal modules. We compared CD-NOD with the window-based method, the IB method, and the MC method. The setting of hyperparameters and window size was the same as that in Setting 1. 
    
    Figure \ref{Fig: Simu_dir2} gives the accuracy (F$_1$ score) of the inferred causal directions from heterogeneous and nonstationary data. We can see that CD-NOD (Algorithm \ref{Alg: direction}) gives the best F$_1$ score, and its accuracy slightly increase with sample size. The accuracy of window-based method is comparably lower than CD-NOD (Algorithm \ref{Alg: direction}), and its performance is better on heterogeneous data than that on nonstationary data. The accuracy of the IB and MC method is much lower, since they are only for linear systems.
     \begin{figure} 
    	\centering
    	\includegraphics[width=.9\textwidth, trim={1cm 0cm 1cm 0cm},clip]{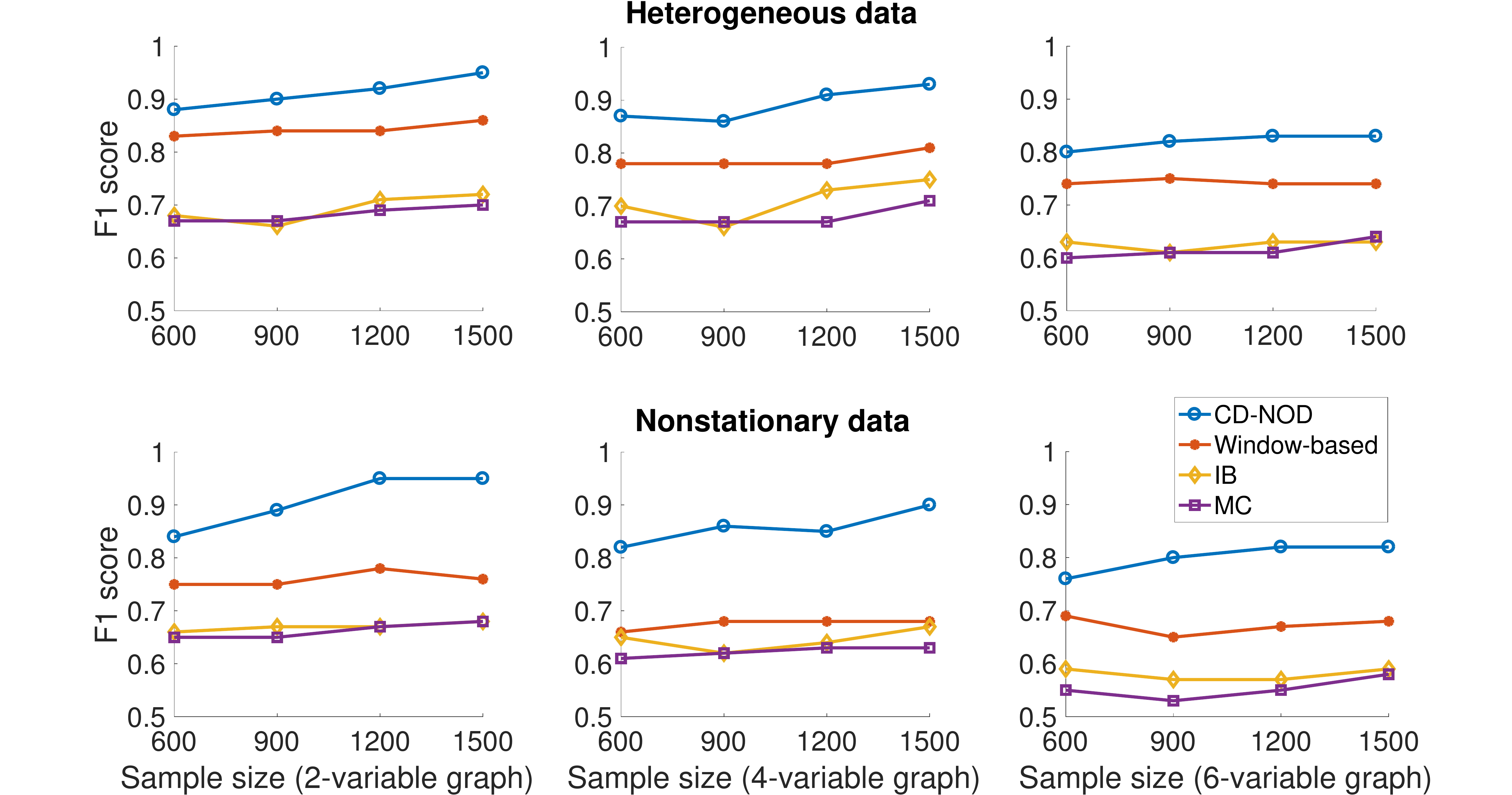}
    	\caption{Accuracy of inferred causal directions from heterogeneous (upper row) and nonstationary (lower row) data. We tested the accuracy on 2-variable (left column), 4-variable (middle column), 6-variable (right column) fully connected graphs. We compared our proposed CD-NOD with the window-based method, the IB method, and the MC method.}
    	\label{Fig: Simu_dir2}
    \end{figure}
     
     \subsubsection{Setting 3}
     \textcolor{black}{We further investigated the effect of sample size per domain in heterogeneous data of the final performance. We varied the number of samples $N_0$ in each domain, with $N_0 = 10,20,40,60,80$, and $100$, and there were in total 10 domains. Other settings of the data generating process are the same as that in Setting 1.}
     
     \textcolor{black}{Figure \ref{Fig: Simu_added}(left) gives the F$_1$ score of the recovered causal skeleton, along with the number of samples per domain. We compared CD-NOD with original constraint-based methods (without taking into account the distribution shift), the IB method, and the MC method. Figure \ref{Fig: Simu_added}(right) shows the F$_1$ score of the recovered whole causal graph, including both the skeleton and directions, produced by CD-NOD,  the window-based method, the IB method, and the MC method, respectively. 
     We can see that CD-NOD achieves the best accuracy in all cases, regarding both the estimated skeleton and the whole causal graph. Moreover, we found that not surprisingly, its performance improves with the number of samples in each domain, with a relatively large improvement (around 5\%) from $N_0 = 40$ to $N_0 = 60$.}
     \begin{figure} 
    	\centering
    	\includegraphics[width=.9\textwidth, trim={0cm 0cm 0cm 0cm},clip]{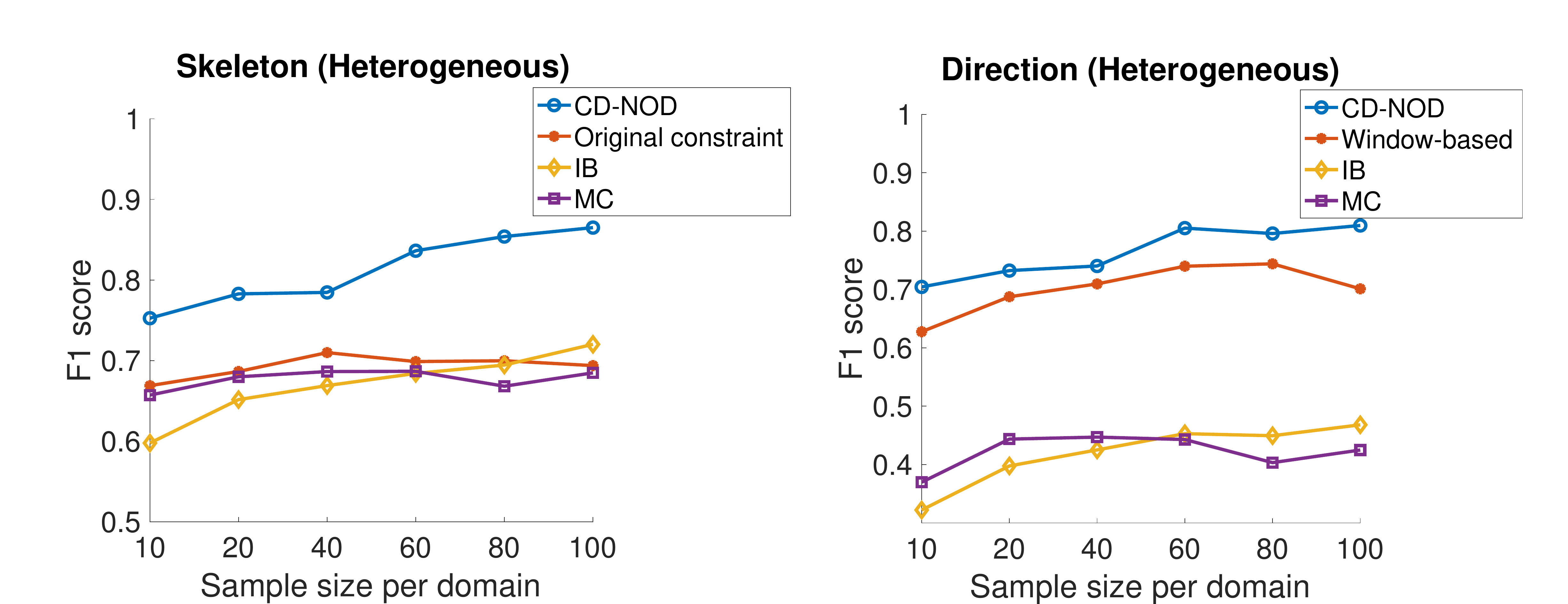}
    	\caption{\textcolor{black}{Accuracy of the estimated causal skeleton (left) and the whole causal graph (right) from heterogeneous data, along with the number of samples in each domain.}}
    	\label{Fig: Simu_added}
    \end{figure}

     \subsubsection{Setting 4}
     After identifying the causal structure, we then estimated the nonstationary driving force of changing causal modules by the procedure KNV given in Algorithm \ref{KNV}. For illustration purposes, we used a single driving force for each module, which changes both causal strength $b$ and noise parameter $\sigma$.
     
     \paragraph{Phase III: nonstationary driving force estimation} We used Gaussian kernels both in kernel embedding of constructed joint distributions and kernel PCA. We compared our approach with the linear time-dependent functional causal model \citep{Huang15}, which puts a GP prior on time-varying coefficients and uses the mean of the posterior to represent the nonstationary driving force. In addition, we compared our methods with Bayesian change point detection \citep{Adam07}, \footnote{We used the implemented Matlab code from http://hips.seas.harvard.edu/content/bayesian-online-changepoint-detection.} which is widely used in nonstationary data to detect change points.

     Figure \ref{Fig: Simu_recover}(a) and \ref{Fig: Simu_recover}(b) show the estimated low-dimensional driving force of changing causal modules for smooth changes (nonstationary data) and sudden changes (heterogeneous data), respectively, when $N=600$. In left panels, blue lines show the estimated driving force by KNV, and red lines show ground truth. Vertical black dashed lines indicate detected change points by Bayesian change point detection. Middle panels show the largest ten eigenvalues of Gram matrix $M^g$. In right panels, blue lines are the recovered changing components by the linear time-dependent GP, and red lines are ground truth. The scale of the y-axis has been adapted for clarity. We only showed the first principal component in the left panel, since the first eigenvector captures most of the variance, as shown in middle panels. We can see that KNV gives the best recovery of the changing components in all cases. Bayesian change point detection, which aims to estimate change points in the marginal or joint distributions (instead of causal mechanisms), fails to handle the case of smooth changes. 
     The linear time-dependent GP does not work well, and one reason is that it cannot account for influences from changing noise parameters, in cases of both smooth and sudden change.

      \begin{figure} 
    	\centering
    	\subfloat[][Smooth changes]{
    		\includegraphics[width=.8\textwidth,trim={1.7cm 0.5cm 2.1cm 0cm},clip]{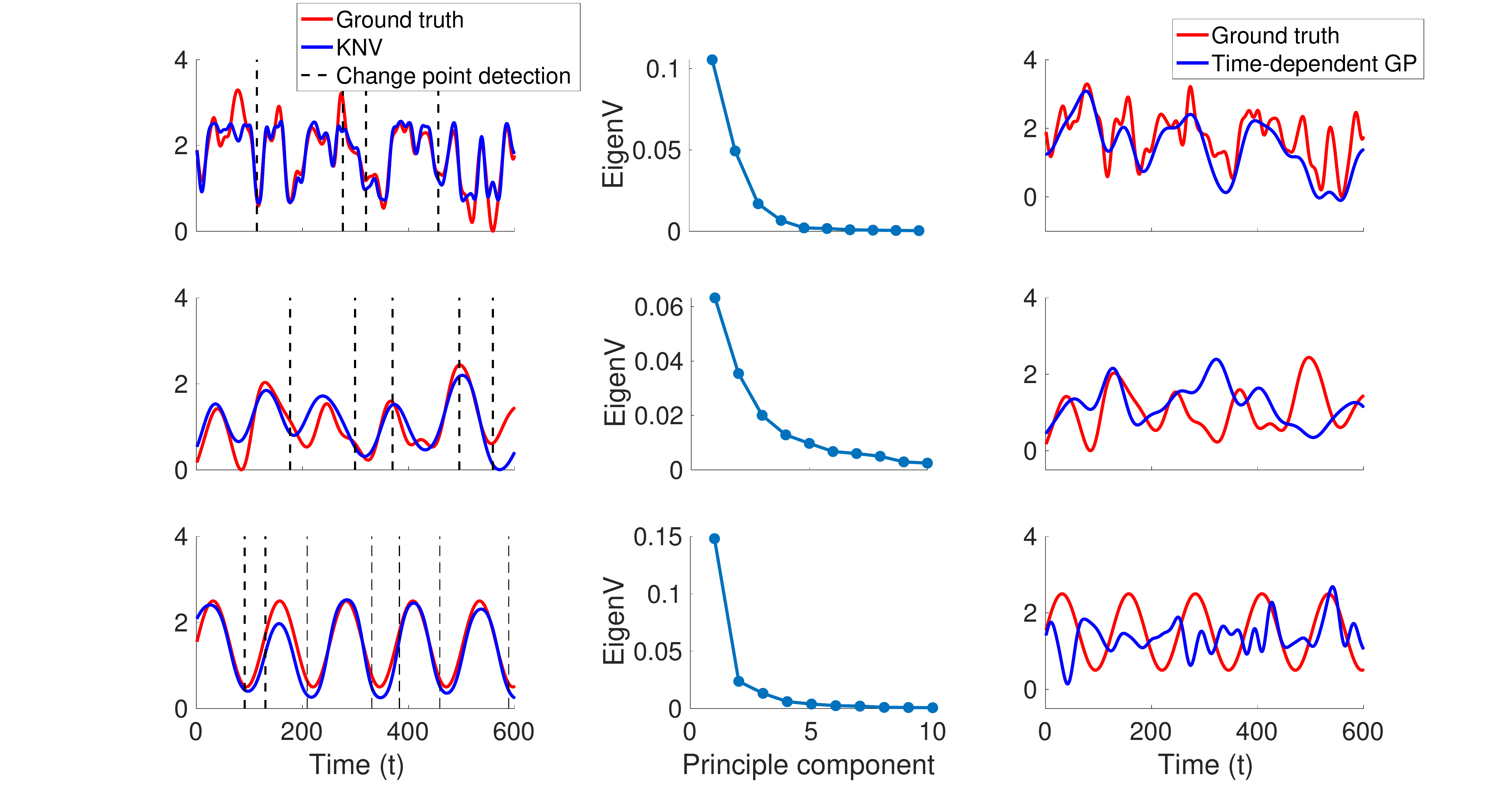}
    	}\\
    	\subfloat[][Sudden changes]{
    		\includegraphics[width=.8\textwidth,trim={1.7cm 0.5cm 2.1cm 0.5cm},clip]{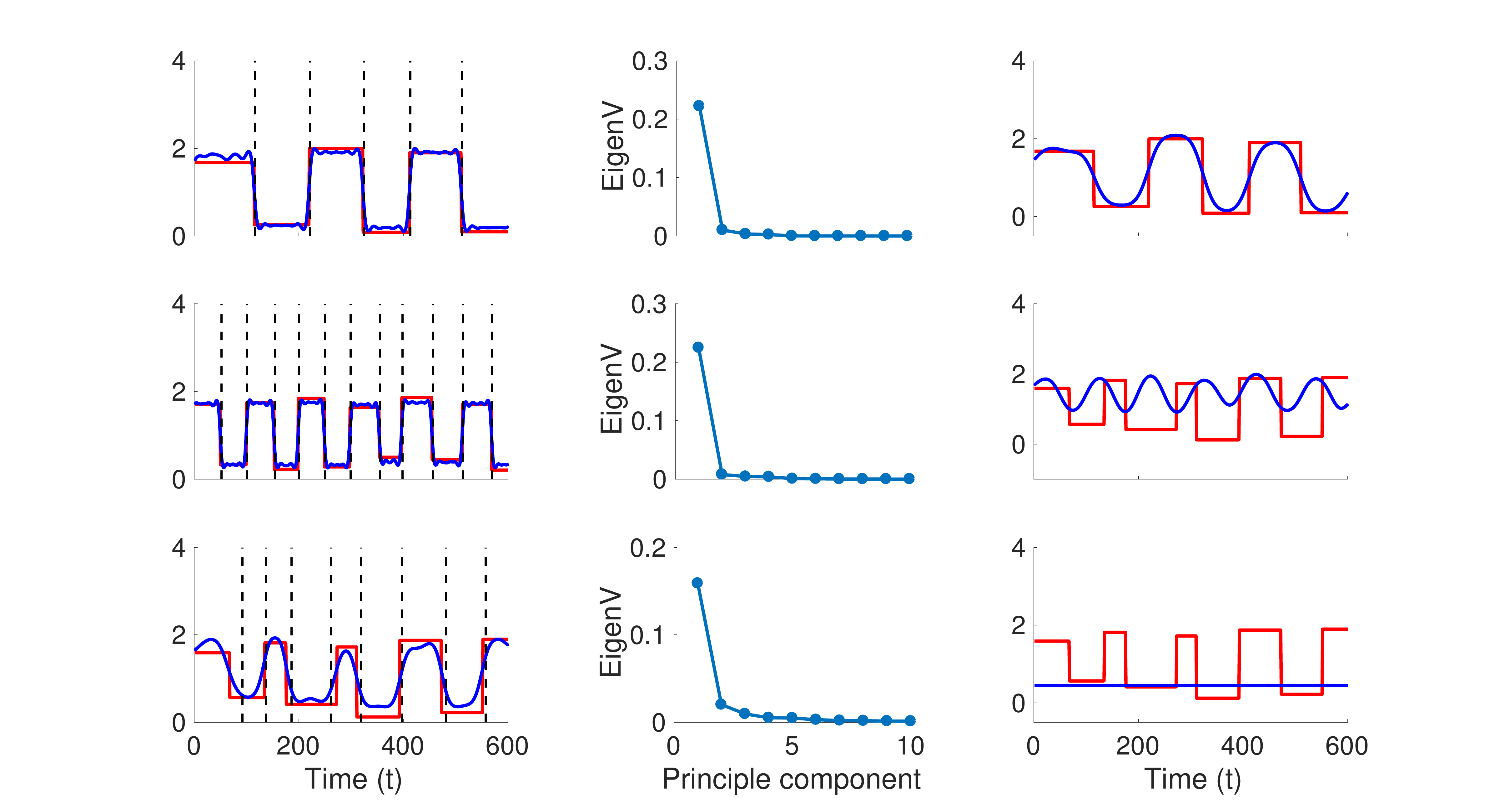}
    	}
    	\caption{Visualization of estimated driving forces of changing causal modules for (a) smooth changes and (b) sudden changes. Left panel: blue lines show the recovered nonstationary components by KNV. Red lines are ground truth. Vertical black dashed lines indicate detected change points by Bayesian change point detection. Middle Panel: the largest ten eigenvalues of Gram matrix $M^g$. Right Panel: blue lines are recovered nonstationary components by linear time-dependent GP. Red lines are ground truth. }\label{Fig: Simu_recover}
    	
    \end{figure}

    \subsection{Real-World Data Sets}
    We then applied the proposed CD-NOD to three real-world data sets. We performed the same three procedures as before: recover causal structure over variables, identify changing causal modules, and estimate low-dimensional driving forces of changing causal modules. Specifically, for causal skeleton determination and changing module detection, we applied Algorithm \ref{rs_causal_discovery}. For causal direction determination, we applied three orientation rules: Algorithm \ref{Alg: invariance}, Algorithm \ref{Alg: direction}, and Meek's orientation rule, in sequence. For driving force estimation, we applied KNV (Algorithm \ref{KNV}). Since our proposed methods outperform all the alternatives on synthetic data, we only applied our proposed methods to real-world data sets.
    
    \subsubsection{Task fMRI}
    We applied our methods to task fMRI data, which were recorded under a star/plus experiment \citep{starplus04}. There are three states in the experiment. (1) The subject rested or gazed at a fixation point on the screen (State 1). (2) The subject was shown a picture and a sentence that was not negated, and was instructed to press a button to indicate whether the sentence matched the picture (State 2). (3) The subject was shown a picture and a negated sentence, and was instructed to press a button indicating whether the sentence matched the picture (State 3). The time resolution of the recording is 500 ms. The fMRI voxel data are organized into regions of interests (ROIs), resulting in 25 ROIs. 
    
    Figure \ref{fMRI_graph} shows the recovered causal structure over the 25 ROIs, where red circles mean that corresponding brain areas have changing causal mechanisms. We can see that the causal edges from the ROIs in the right hemisphere to the corresponding left ones are robust, e.g., RIPL $\rightarrow$ LIPL, RIT $\rightarrow$ LIT. ROIs CALC, LSGA, LIPL, and LIT have large indegree centrality, where CALC and LIT are responsible for visual input processing, and LSGA and LIPL are for language processing. The causal modules of CALC, RIPL, LIPL, RSGA, LSGA, RDLPFC, LDLPFC, RTRIA, LSP, and RIT are time-varying, which are marked with red circles. The identified changing causal modules correspond to key areas for visual and language perception. By considering nonstationarity, 16 connections that were produced by the original PC algorithm are removed, e.g., LIPL - RSGA, LSPL - RDLPFC, and RDLPFC - RTRIA, indicating possible pseudo confounders behind those pairs of ROIs.
    
    \begin{figure}
    	\centering
    	\includegraphics[width=0.9\textwidth]{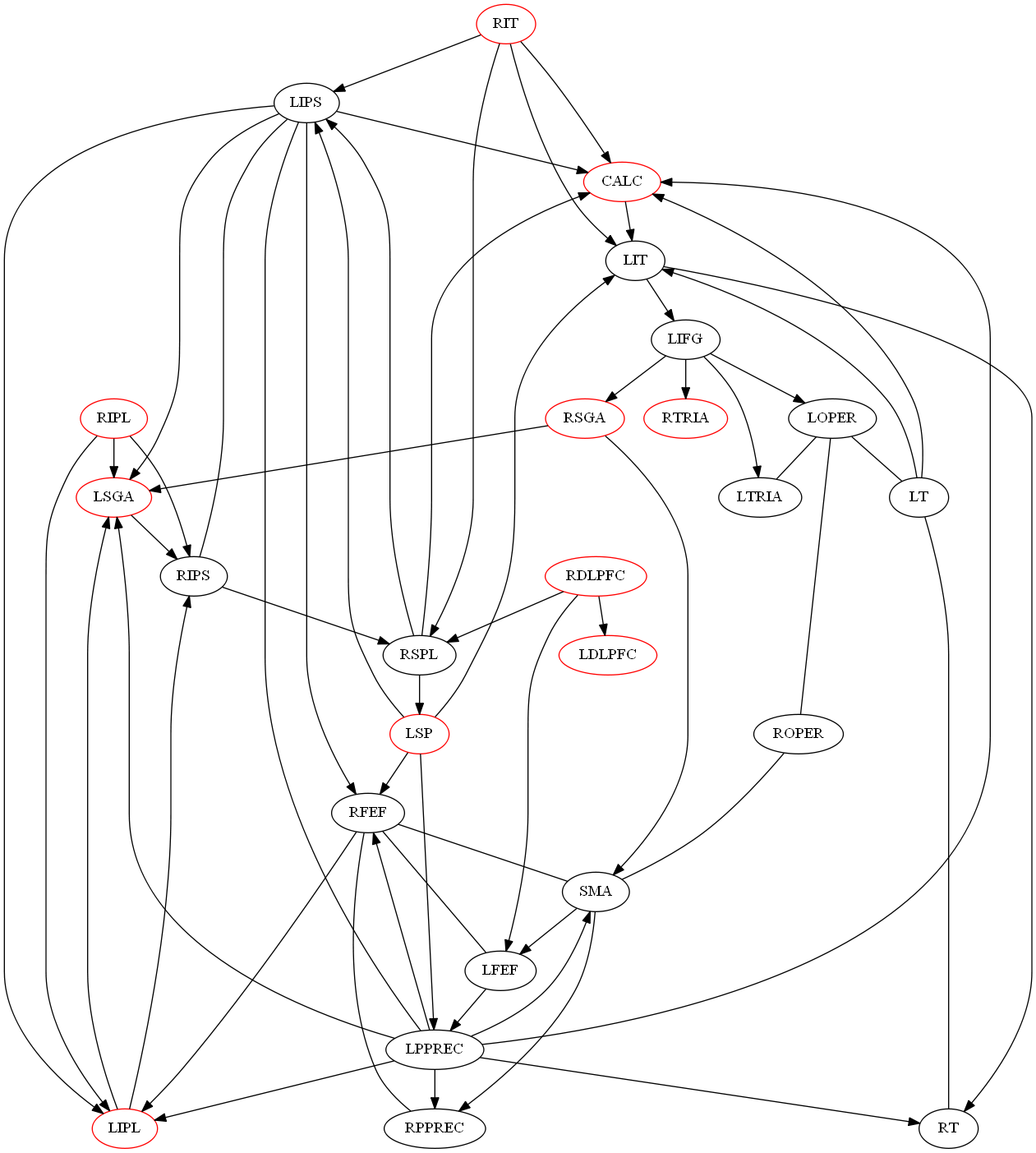}
    	\caption{Recovered causal graph over 25 ROIs. The red circles denote that causal modules of corresponding brain areas changing over states.}\label{fMRI_graph}
    \end{figure}
    
    Figure \ref{fMRI_recovery} visualizes the estimated nonstationary driving force by KNV. The top figure shows the state of input stimuli: State = 1 (State 1) means that the subject was in resting state; State = 2 (State 2) means that the subject was shown a picture and a non-negated sentence, and State = 3 (State 3) means that the shown sentence was negated. The lower four figures show the estimated nonstationary driving forces of the causal modules of several key ROIs: LDLPFC, LIPL, LSGA, and RIPL, respectively. The shaded areas are intervals which have obvious changes; they match the resting state (State 1) with no input stimuli. Quantitatively, the correlations between states and the recovered driving forces of LDLPFC, LIPL, LSGA, and RIPL are 0.22, 0.49, 0.41, 0.47, respectively. For LDLPFC, the second shaded interval is delayed, which may be due to the fact that LDLPFC is involved in cognitive processes, e.g., working memory, so that it may remain activated for some time when the input stimulus has been removed.  There are changes between State 2 and State 3, but they are less obvious. In short, the recovered nonstationary driving forces show that there are obvious changes between the resting state and the task states, while the changes between two task states are relatively less obvious.

    \begin{figure} 
    	\centering
    	\includegraphics[width=.4\textwidth,trim={.2cm 0cm 2.2cm 0cm},clip]{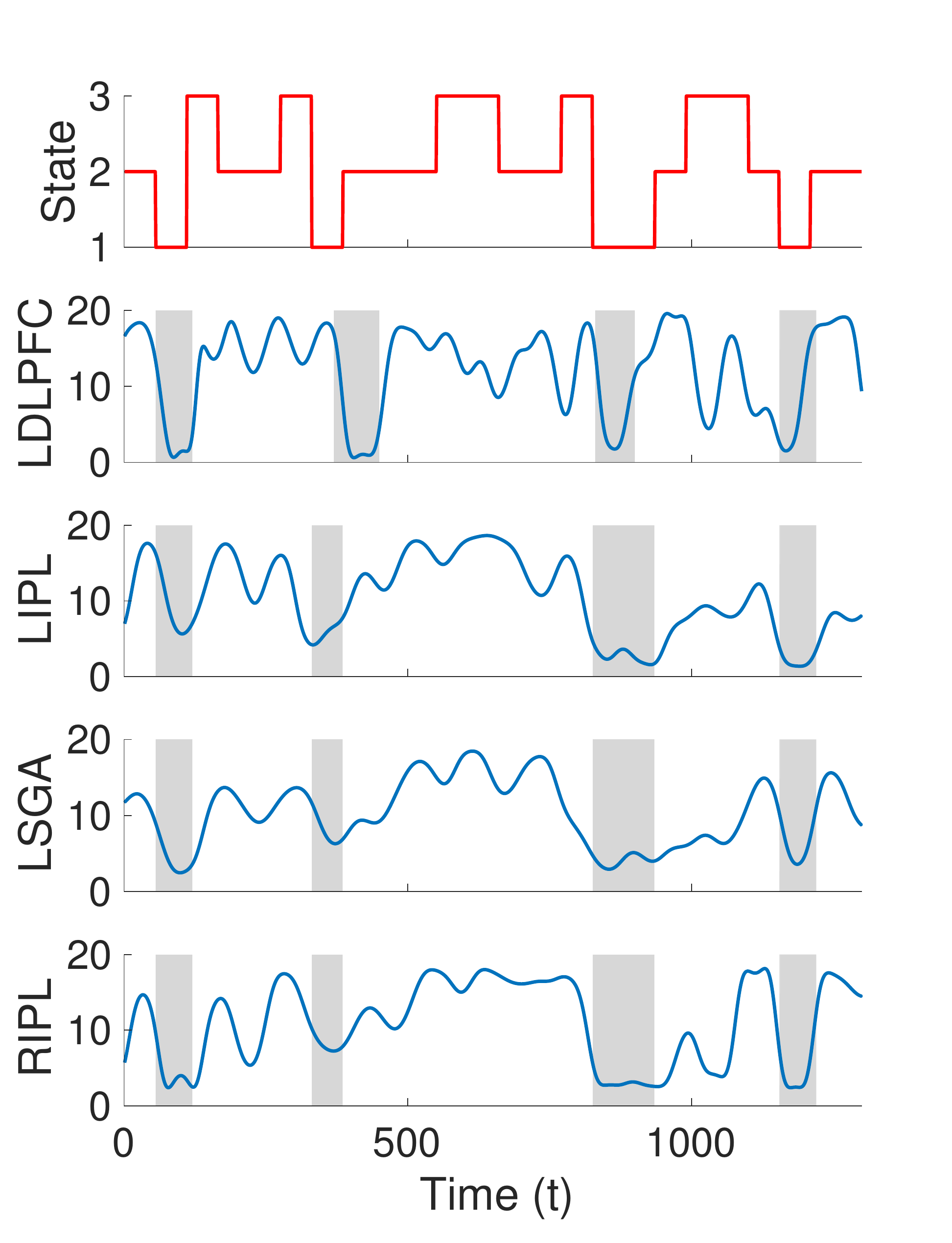}
    	\caption{Visualization of estimated nonstationary driving force from task fMRI data. The top figure shows the state of input stimuli: State = 1 means that the subject was in resting state; State = 2 means that the subject was shown a picture and a sentence, in which the sentence was not negated, while in State = 3 the sentence was negated. The other four figures show the estimated nonstationary driving forces of LDLPFC, LIPL, LSGA, and RIPL, respectively. The shaded areas are intervals which have obvious changes; they match the resting states (State = 1) with no input stimuli. For LDLPFC, the second shaded interval is delayed. There are changes between State 2 and State 3 but  are  less obvious, comparatively.}\label{fMRI_recovery}
    \end{figure}

    \subsubsection{Stock Returns}
    We applied our methods to daily returns of stocks from Hong Kong (HK) and the United States (US), downloaded from Yahoo Finance. 
    
    \paragraph{HK Stock Market} The HK stock dataset contains 10 major stocks, which are daily dividend adjusted closing prices from 10/09/2006 to 08/09/2010. For the few days when the stock price is not available, a simple linear interpolation is used to estimate the price. Denoting the closing price of the $i$th stock on day $t$ by $P_{i,t}$, the corresponding return is calculated by $V_{i,t} = \frac{P_{i,t}-P_{i,t-1}}{P_{i,t-1}} $.
    The 10 stocks are Cheung Kong Holdings (1), Wharf (Holdings) Limited (2), HSBC Holdings plc (3), Hong Kong Electric Holdings Limited (4), Hang Seng Bank Ltd (5), Henderson Land Development Co. Limited (6), Sun Hung Kai Properties Limited (7), Swire Group (8), Cathay Pacific Airways Ltd (9), and Bank of China Hong Kong (Holdings) Ltd (10). Among these stocks, 3, 5, and $10$ belong to Hang Seng Finance Sub-index (HSF), 1, 8, and $9$ belong to Hang Seng Commerce \& Industry Sub-index (HSC), 2, 6, and $7$ belong to Hang Seng Properties Sub-index (HSP), and $4$ belongs to Hang Seng Utilities Sub-index (HSU).
    
    Figure \ref{stock} shows the estimated causal structure, where the causal modules of $2, 3, 4, 5$ and $7$ are found to be time-dependent, indicated by red circles. In contrast, the PC algorithm with the KCI test (without taking into account the nonstationarity) gives five more edges, which are $2-3$, $2-5$, $3-6$, $5-7$, and $4-5$; most of these spurious edges (the former four) are in the finance sub-index and the properties sub-index, indicating the existence of pseudo confounders behind these sub-indices.  
    We found that all stock returns that have changing causal mechanisms are in HSF, HSP, and HSU; they might be affected by some unconsidered changing factors, e.g., the change of economic policies.  Furthermore, we estimated causal directions by the procedure given in Algorithm \ref{Alg: invariance} and \ref{Alg: direction}, and the inferred directions are consistent with some background knowledge of the market. For instance, the within sub-index causal directions tend to follow the owner-member relationship. Examples include $5 \rightarrow 3$ (note that $3$ partially owns $5$), $9 \rightarrow 8$, and $4 \rightarrow 1$.  Those stocks in HSF are major causes for those in HSC and HSP, and the stocks in HSP and HSU impact those in HSC. These results indicate that overseas markets (e.g., the US market) may influence the local market via large banks, and that prices of stocks in commerce and industry are usually affected by certain companies in finance, properties, and utilities sectors.
    
    \begin{figure} 
    	\centering
    	\includegraphics[width=0.5\textwidth]{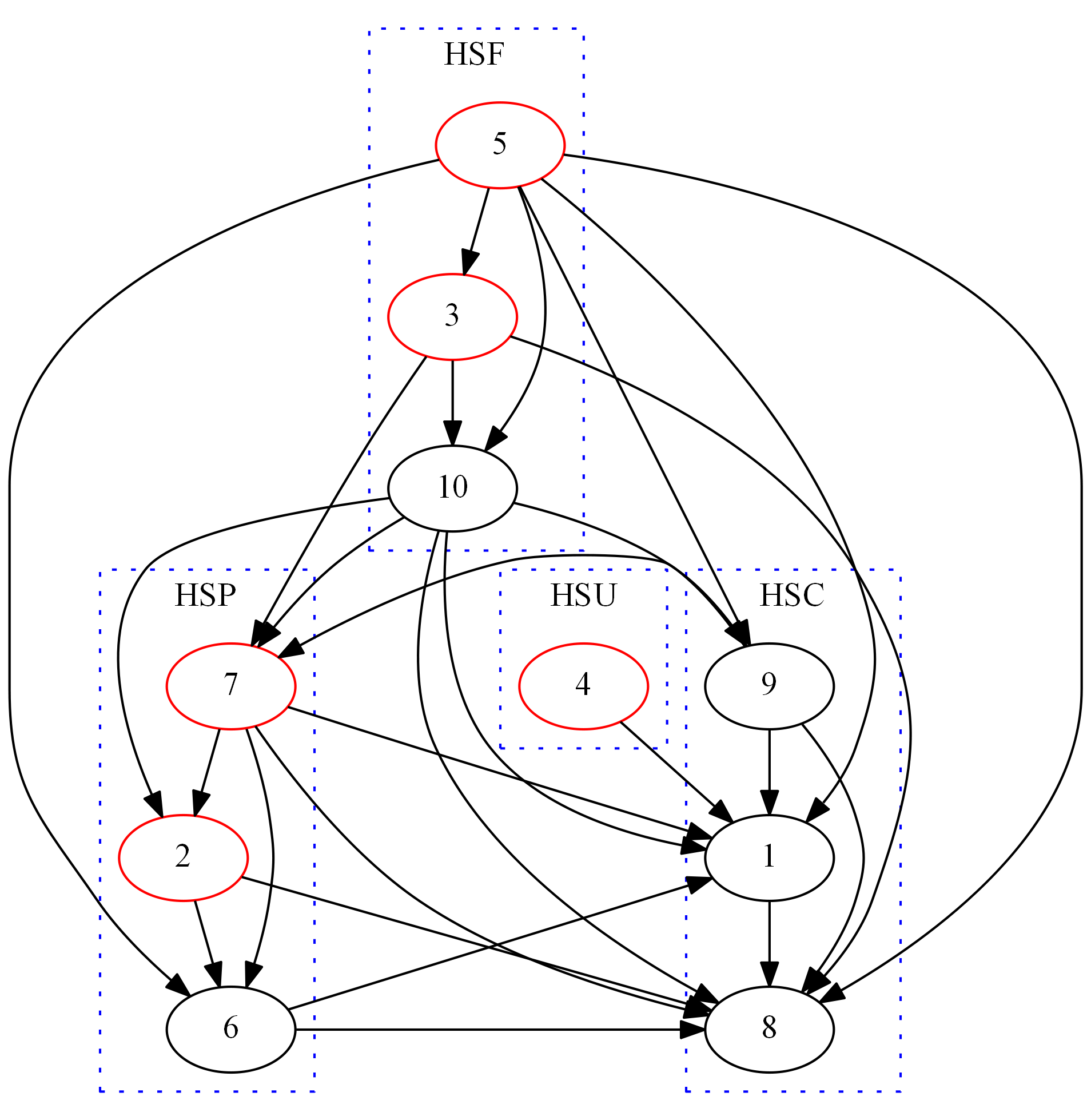}%
    	\vspace{-0.3cm}
    	\caption{Recovered causal graph over the 10 main HK stocks. Red circles indicate that causal modules of corresponding stocks change over time. }\label{stock}
    \end{figure}

    \begin{figure} 
    	\centering
    	\hspace{0.75cm}\includegraphics[width=0.3\textwidth,height = 1.2cm,trim={0.65cm 11.6cm 0.86cm 12cm},clip]{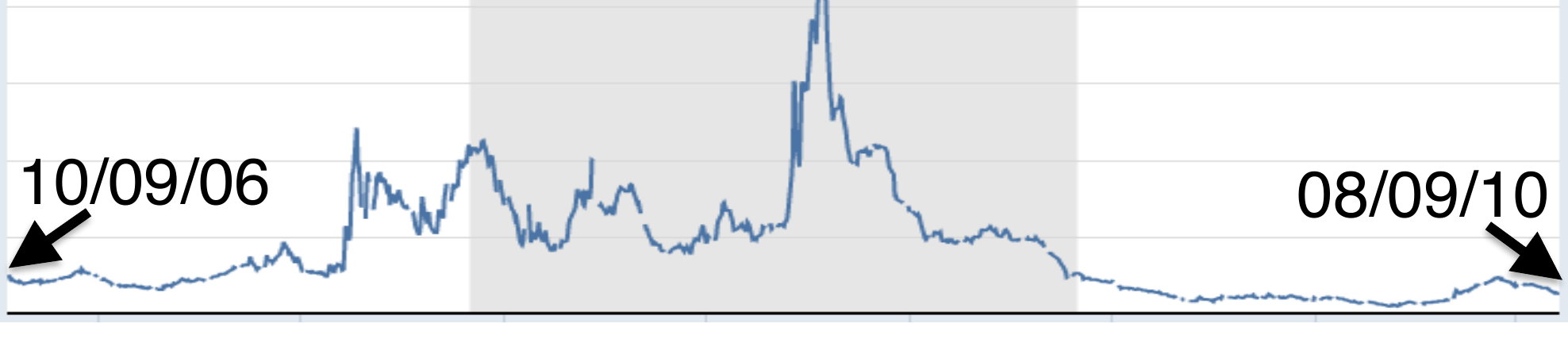}{\footnotesize~~~}~~~~~~~~~~~~~~~~~~~~~~~~~~~~~~~~~~~~~\\
    	\includegraphics[width=0.7\textwidth,trim={2.2cm 0cm 2.2cm 0cm},clip]{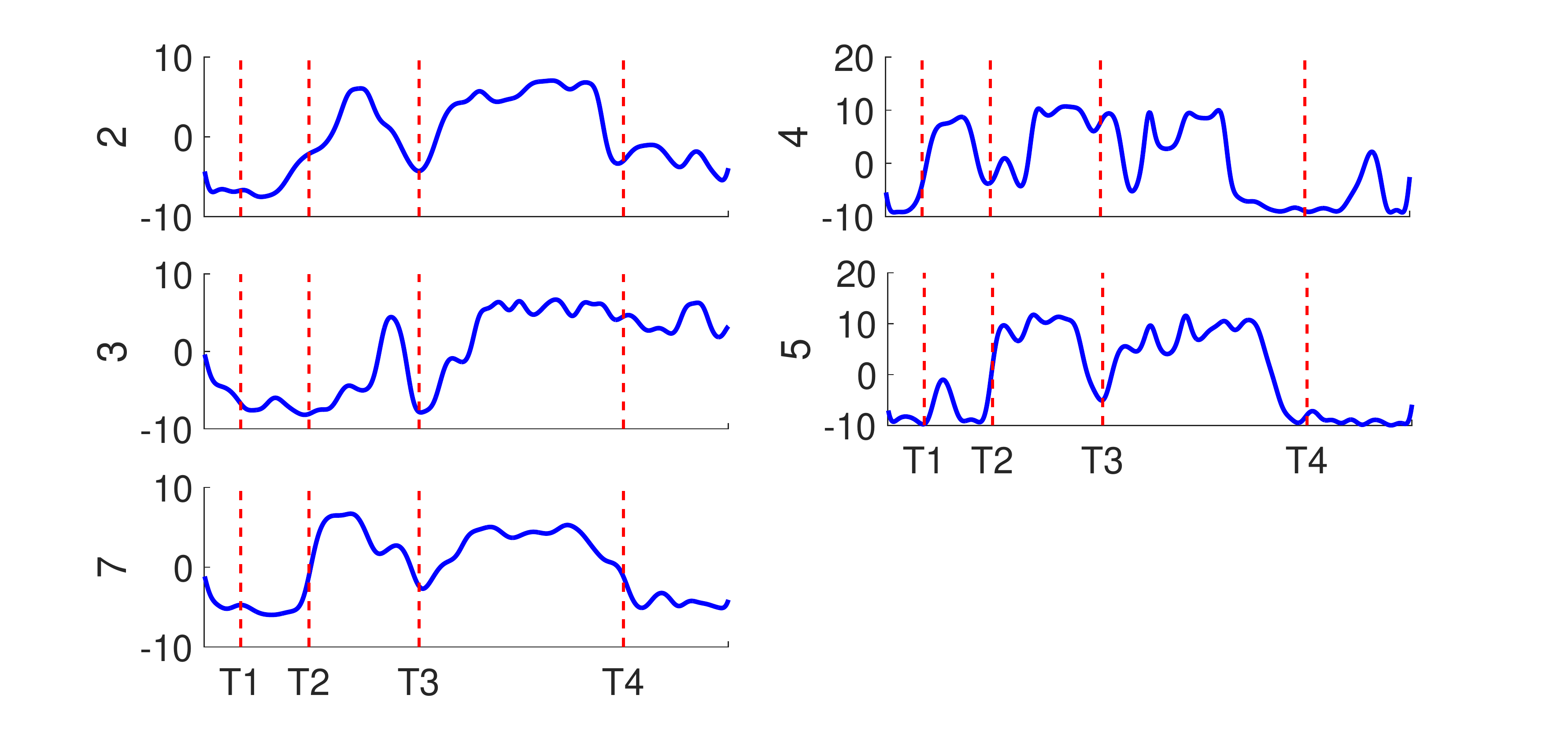}%
    	\caption{The estimated nonstationary driving forces of time-dependent stock returns, as well as the curve of the TED spread, over the period from $10/09/2006$ to $08/09/2010$. Top: Curve of the TED spread shown for comparison.
    		Bottom: Estimated nonstationary driving force of stocks $2,3,4,5$, and $7$, where $T_1$, $T_2$, $T_3$, and $T_4$ stand for  $01/22/2007$, $07/16/2007$, $05/03/2008$, and $11/02/2009$, respectively.
    		We can see that the nonstationary components of root causes, $4$ and $5$, share the similar variability with change points around $T_1$, $T_2$, $T_3$, and $T_4 $. The nonstationary components of $2$, $3$, and $7$ have change points only around $T_2$, $T_3$, and $T_4$.}\label{stock_visual}
    \end{figure}
    
    Figure \ref{stock_visual} (bottom panels) visualizes the estimated driving forces of stocks $2,3,4,5$, and $7$ (the scale of y-axis has been adjusted).  We can see that the nonstationary components of root causes, $4$ and $5$, share a similar variability; the change points are around  01/22/2007 ($T_1$),  07/16/2007 ($T_2$),  06/30/2008 ($T_3$), and  02/11/2009 ($T_4$). The nonstationary components of the causal modules of $2$, $3$, and $7$ have change points around $T_2$, $T_3$, and $T_4$. Stock $5$ has additional change points at $T_1$ (01/22/2007), which is perhaps due to the fact that Hang Seng Bank established its wholly owned subsidiary Hang Seng Bank (China) Limited in $2007$.
    The change points around $T_2$, $T_3$, and $T_4$ match with the critical time points of financial crisis around the year of $2008$.  The active phase of the crisis, which manifested as a liquidity crisis, could be dated from August, 2007,\footnote{See more information at \url{https://en.m.wikipedia.org/wiki/Financial_crisis_of_2007-08}.} around $T_2$. The estimated driving forces are consistent with the change of the TED spread,\footnote{See \url{https://en.m.wikipedia.org/wiki/TED_spread}.} which is an indicator of perceived credit risk in the general economy and is shown in Figure \ref{stock_visual} (top panel) for comparison.

     \paragraph{US Stock Market} We then applied our methods to daily returns of a number of stocks
    from New York Stock Exchange, downloaded from Yahoo Finance. We considered 80 major stocks and used their daily dividend adjusted closing prices from 07/05/2006 to 12/16/2009. They are clustered into 10 sectors: energy, public utilities, capital goods, health care, consumer service, finance, transportation, consumer nondurable goods, basic industry, and technology. 
    
    Figure \ref{USstock_graph} shows the estimated causal graph over stock returns, each color representing one sector. The size of each node reflects its degree of connections (sum of indegree and outdegree), and larger node indicates a higher degree.
    We found that intra-sector connections are more dense than inter-sector connections. The stocks in energy, finance, public utilities, and basic industries are more likely to be causes of stocks in other sectors; among these four sectors, stocks in energy and finance usually causally influence stocks in utilities and basic industries. 
    
    More specifically, GE, a stock in energy, does not have causal edges within the sector; instead, it is adjacent to stocks in other sectors, which may be due to the reason that GE is a conglomerate and has various segments, including healthcare, power, transportation, oil and gas, etc.  In regard to causal interactions within energy, CHK and KEG are root causes, which may be because KEG is an oilfield services company, and CHK is a holding company for a variety of energy enterprises; NBR and HAL are also more likely to cause others, and PBR, WMB, TLM, and NE are more likely to be influenced by other stocks. 
    Among the stocks in finance, MS, SAN, and JPM have a large number of edges with stocks from other sectors. Within finance, USB and KEY usually cause others, and WFC is the sink node.
    Among the stocks in basic industry, FCX is largely affected by stocks from energy, and it causes other stocks in basic industry, probably because it also belongs to petroleum industry; GLW and DOW have a large number of edges with stocks in consumer service and technology, which may be related to their products in glass \& ceramic materials, and they are not adjacent to other stocks in basic industry. Within basic industry, IAG is the root cause, and KGC is always influenced by others.
    Regarding the stocks from the remaining sectors, DIS, HST, and HD from consumer service, ORCL from technology, and CSX from transportation have large centrality.
    
    By considering nonstationarity, 63 edges that were produced by the original PC algorithm were removed; 19 of them are within sectors, and 44 are between sectors. 
    The causal modules of 37 out of 80 stocks change over time;
    most of them are from energy (10 out of 28), finance (7 out of 9), basic industry (6 out of 14), and consumer service (5 out of 7). It is reasonable that stocks, that are close to the root node, are more likely to be influenced by some external factors, such as changing economic environments and policies. These findings, such as stocks from finance are more likely to be causes of stocks in other sectors,  are consistent with what we discovered from HK stock market reported above.
    
    \begin{figure} 
    	\centering
    	\includegraphics[width=1\textwidth,trim={0cm 0cm 0.3cm 0cm},clip]{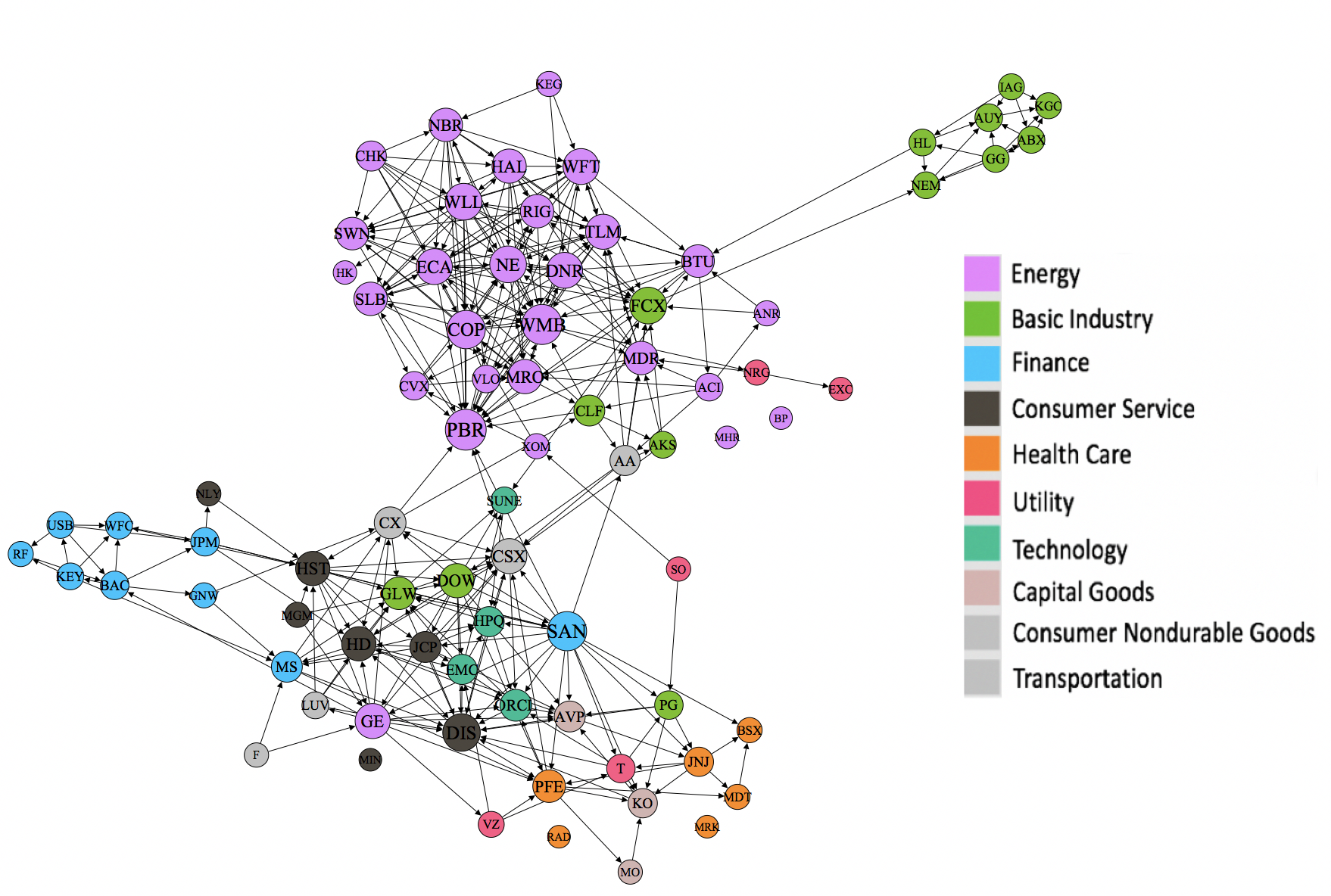}
    	\caption{Recovered causal graph over 80 NYSE stocks. Each color of nodes represents one sector.}\label{USstock_graph}
    \end{figure}
    
    Figure \ref{USstock_recov} visualizes the estimated nonstationary driving forces of stocks SAN, MS, DOW, GE, CHK, and BAC. We can see that among these six stocks, SAN, MS, and DOW have obvious change points around 07/16/2007 ($T_1$) and 05/05/2008 ($T_2$), while the stocks GE, CHK, and BAC only have changes points around 05/05/2008 ($T_2$). We found that stocks, which have change points around both time points, usually have large centrality and influence a large number of other stocks.
    The change points match the critical time of the 2008 financial crisis - those in the TED spread and parts of the change points ($T_2$ and $T_3$) in HK stock data.
    
    \begin{figure} 
    	\centering
    	\includegraphics[width=0.7\textwidth,trim={3.5cm 0cm 3.5cm 0cm},clip]{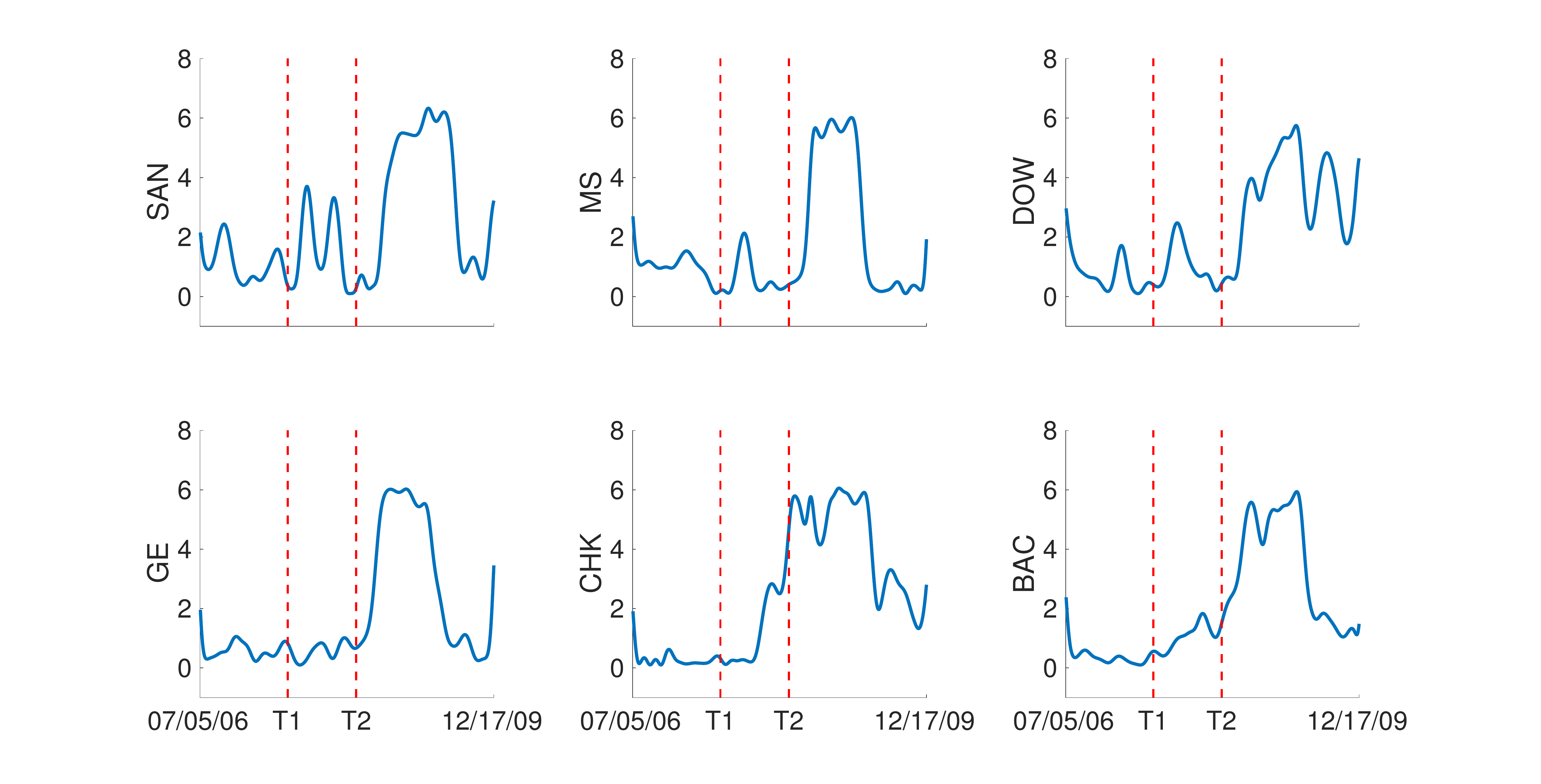}
    	\caption{The estimated nonstationary driving force of six stock returns from 07/05/2006 $\sim$ 12/16/2009. The stocks SAN, MS, and DOW have obvious change points at 07/16/2007 ($T_1$) and 05/05/2008 ($T_2$) . The stocks GE, CHK, and BAC have change points only at 05/05/2008 ($T_2$). The change points match the critical time of financial crisis. }\label{USstock_recov}
    \end{figure}

\section{Discussion and Conclusions}
This paper is concerned with causal discovery from heterogeneous/nonstationary data and visualization of changing causal modules. We assume a pseudo causal sufficiency condition, which states that all confounders can be represented by functions of domain or smooth functions of time index. We proposed (1) an enhanced constraint-based method for locating variables whose causal modules change and estimating the skeleton of the causal structure over observed variables, (2) a method for causal direction determination that takes advantage of the information carried by distribution shifts, and (3) a procedure for estimating a low-dimensional representation of changing causal modules. We then showed that distribution shifts also help for causal discovery when there are stationary confounders. In summary, to identify the causal structure, we take advantage of two types of independence constraints: (a) conditional independence (with constraint-based approach) and (b) independent changes of changing causal modules.

The performance evaluated on synthetic data sets showed largely improved F$_1$ score and precision with our proposed CD-NOD, compared to other methods for causal discovery. The results over task-fMRI data well revealed information flows between brain regions and how causal influences change across resting state and task states. The causal relationships learned from stock data (both HK and US) are consistent with background knowledge, e.g., stocks in finance and energy are causes to others, and the estimated nonstationary driving forces match the critical time of financial crisis around the year of 2008.

We showed that distribution shift contains useful information for causal discovery, since causal model provides a compact description of how the joint distribution changes. Moreover, it has been recently noticed that this view help understand and solve some machine learning problems, e.g., domain adaptation and prediction in nonstationary environments. Since distribution shifts in heterogeneous/nonstationary data are usually constrained - it might be due to the changes in the data-generating processes of only a few variables, and thus it is only necessary to adapt a small proportion of the joint distribution in domain adaptation problems \citep{Zhang13_targetshift,Scholkopf12}.

There are several open questions that we aim to answer in future work. First, in the paper we assumed that causal directions do not flip despite distribution shifts. What if some causal directions also change across domains or over time? It is important to develop a general approach to detect such direction flips and do causal discovery, because of the ubiquity of feedback loops. Second, the issue of distribution shift may decrease the power of statistical (conditional) independence tests. It is essential to develop a reliable statistical test to mitigate this problem. Third, in practice, there may also exist other types of confounders, e.g., nonstationary confounders. We will extend our methods to cover general types of confounders.
  
\acks{We are grateful to the anonymous reviewers whose careful comments and suggestions helped improve this manuscript. 
We would like to acknowledge the support by National Institutes of Health under Contract No. NIH-1R01EB022858-01, FAIN-R01EB022858, NIH-1R01LM012087, NIH5U54HG008540-02, and FAIN-U54HG008540, and by the United States Air Force under Contract No. FA8650-17-C7715. The National Institutes of Health and the U.S. Air Force are not responsible for the views reported in this article. JZ's research was supported in part by the Research Grants Council of Hong Kong under the General Research Fund LU342213.}
	
	
	
	\appendix

    \section*{Appendix A. Proof of Theorem 1} \label{Appendix: Theorem 1}

    \begin{proof}
    	Before getting to the main argument, let us establish some implications of the SEMs in Eq.~\ref{Eq:FCM} and the assumptions in Section~\ref{Sec:assumptions}. Since the structure is assumed to be acyclic or recursive, according to Eq.~\ref{Eq:FCM}, all variables $V_i$ can be written as a function of $\{g_l(C)\}_{l=1}^L \cup \{\theta_m(C)\}_{m=1}^n$ and $\{\epsilon_m\}_{m=1}^n$. As a consequence, the probability distribution of $\mathbf{V}$ at each value of $C$ is determined by the distribution of $\epsilon_1, ..., \epsilon_n$, and the values of $\{g_l(C)\}_{l=1}^L \cup \{\theta_m(C)\}_{m=1}^n$. In other words, $P(\mathbf{V}|C)$ is determined by $\prod_{i=1}^n P(\epsilon_i)$ (for $\epsilon_1, ..., \epsilon_n$ are mutually independent), and $\{g_l(C)\}_{l=1}^L \cup \{\theta_m(C)\}_{m=1}^n$, where $P(\cdot)$ denotes the probability density or mass function. For any $V_i$, $V_j$, and $\mathbf{V}^{ij}\subseteq \{V_k\,|\,k\neq i, k\neq j\}$, because $P(V_i, V_j \,|\,  \mathbf{V}^{ij}, C)$ is determined by $P(\mathbf{V}|C)$, it is also determined by $\prod_{i=1}^n P(\epsilon_i)$ and $\{g_l(C)\}_{l=1}^L \cup \{\theta_m(C)\}_{m=1}^n$. Since $\prod_{i=1}^n P(\epsilon_i)$ does not change with $C$, we have
    	\begin{flalign} \nonumber
    	&P(V_i, V_j \,|\,  \mathbf{V}^{ij} \cup \{g_l(C)\}_{l=1}^L \cup \{\theta_m(C)\}_{m=1}^n \cup \{C\})  \\ \label{Eq:Cind1}
    	=& P(V_i, V_j \,|\,  \mathbf{V}^{ij} \cup \{g_l(C)\}_{l=1}^L \cup \{\theta_m(C)\}_{m=1}^n).
    	\end{flalign}
    	That is,
    	\begin{equation} \label{Eq:Cind}
    	C \independent (V_i, V_j) \,|\,  \mathbf{V}^{ij} \cup \{g_l(C)\}_{l=1}^L \cup \{\theta_m(C)\}_{m=1}^n.
    	\end{equation}
    	By the weak union property of conditional independence, it follows that
    	\begin{equation} \label{Eqc}
    	C \independent V_j \,|\,\{V_i\}\cup \mathbf{V}^{ij} \cup \{g_l(C)\}_{l=1}^L \cup \{\theta_m(C)\}_{m=1}^n.
    	\end{equation}

    	We are now ready to prove the theorem. Let $V_i, V_j$ be any two variables in $\mathbf{V}$.
    	First, suppose that $V_i$ and $V_j$ are not adjacent in $G$. Then they are not adjacent in $G^{aug}$, which recall is the graph that incorporates $\{g_l(C)\}_{l=1}^L \cup \{\theta_m(C)\}_{m=1}^n$. It follows that there is a set $\mathbf{V}^{ij}\subseteq \{V_k\,|\,k\neq i, k\neq j\}$ such that $\mathbf{V}^{ij} \cup \{g_l(C)\}_{l=1}^L \cup \{\theta_m(C)\}_{m=1}^n$ d-separates $V_i$ from $V_j$. Since the joint distribution over $\mathbf{V} \cup \{g_l(C)\}_{l=1}^L \cup \{\theta_m(C)\}_{m=1}^n$ is assumed to be Markov to $G^{aug}$, we have
    	\begin{equation} \label{Eqd}
    	V_i \independent V_j \,|\,  \mathbf{V}^{ij} \cup \{g_l(C)\}_{l=1}^L \cup \{\theta_m(C)\}_{m=1}^n.
    	\end{equation}
    	Because all $g_l(c)$ and $\theta_m(C)$ are deterministic functions of $C$, we have $P(V_i, V_j \,|\,  \mathbf{V}^{ij} \cup \{C\}) = P(V_i, V_j \,|\,  \mathbf{V}^{ij} \cup \{g_l(C)\}_{l=1}^L \cup \{\theta_m(C)\}_{m=1}^n \cup \{C\})$.

    	According to the properties of mutual information given in~\citet{Madiman08} , Eqs.~\ref{Eqd} and~\ref{Eq:Cind} imply
    	$V_i \independent (C, V_j) \,|\,  \mathbf{V}^{ij} \cup \{g_l(C)\}_{l=1}^L \cup \{\theta_m(C)\}_{m=1}^n$.
    	By the weak union property of conditional independence, it follows that
    	$V_i \independent V_j \,|\,  \mathbf{V}^{ij} \cup \{g_l(C)\}_{l=1}^L \cup \{\theta_m(C)\}_{m=1}^n \cup \{C\}$.
    	As all $g_l(C)$ and $\theta_m(C)$ are deterministic functions of $C$, it follows that
    	$V_i \independent V_j \,|\,  \mathbf{V}^{ij} \cup \{C\}$.
    	In other words, $V_i$ and $V_j$ are conditionally independent given a subset of $\{V_k\,|\,k\neq i, k\neq j\}\cup \{C\}$.
    	
    	Conversely, suppose $V_i$ and $V_j$ are conditionally independent given a subset $\mathbf{S}$ of $\{V_k\,|\,k\neq i, k\neq j\}\cup \{C\}$. We show that $V_i$ and $V_j$ are not adjacent in $G$, or equivalently, that they are not adjacent in $G^{aug}$. There are two possible cases to consider:
    	
    	\begin{itemize}
    		\item Suppose $\mathbf{S}$ does not contain $C$. Then since the joint distribution over $\mathbf{V} \cup \{g_l(C)\}_{l=1}^L \cup \{\theta_m(C)\}_{m=1}^n$ is assumed to be faithful to $G^{aug}$, $V_i$ and $V_j$ are not adjacent in $G^{aug}$, and hence not adjacent in $G$.
    		
    		\item Otherwise, $\mathbf{S}=\mathbf{V}^{ij}\cup \{C\}$ for some $\mathbf{V}^{ij}\subseteq \{V_k\,|\,k\neq i, k\neq j\}$. That is,
    		\begin{flalign} \label{tmp0}
    		&V_i \independent V_j \,|\,  \mathbf{V}^{ij} \cup \{C\},\textrm{ or } \\ \nonumber
    		& P(V_i, V_j \,|\,  \mathbf{V}^{ij} \cup \{C\}) = P(V_i \,|\,  \mathbf{V}^{ij} \cup \{C\}) P(V_j \,|\,  \mathbf{V}^{ij} \cup \{C\}).
    		\end{flalign}
    		
    		According to Eq.~\ref{Eq:Cind1}, and also noting that $\{g_l(C)\}_{l=1}^L \cup \{\theta_m(C)\}_{m=1}^n$ is a deterministic function of $C$, we have
    		\begin{equation}\label{tmp1}
    		P(V_i, V_j \,|\,  \mathbf{V}^{ij}\cup \{C\}) = P(V_i, V_j \,|\,  \mathbf{V}^{ij} \cup \{g_l(C)\}_{l=1}^L \cup \{\theta_m(C)\}_{m=1}^n),
    		\end{equation}
    		which also implies
    		\begin{flalign} \label{tmp2}
    		&P(V_i \,|\, \mathbf{V}^{ij}\cup \{C\}) = P(V_i \,|\,  \mathbf{V}^{ij} \cup \{g_l(C)\}_{l=1}^L \cup \{\theta_m(C)\}_{m=1}^n), \\ \label{tmp3}
    		& P(V_j \,|\,  \mathbf{V}^{ij}\cup  \{C\}) = P(V_j \,|\,  \mathbf{V}^{ij} \cup \{g_l(C)\}_{l=1}^L \cup \{\theta_m(C)\}_{m=1}^n).
    		\end{flalign}
    		Substituting Eqs.~\ref{tmp1} -~\ref{tmp3} into Eq.~\ref{tmp0} gives
    		\begin{flalign}
    		&P(V_i, V_j \,|\,  \mathbf{V}^{ij} \cup \{g_l(C)\}_{l=1}^L \cup \{\theta_m(C)\}_{m=1}^n) \\ \nonumber
    		=& P(V_i \,|\,  \mathbf{V}^{ij} \cup \{g_l(C)\}_{l=1}^L \cup \{\theta_m(C)\}_{m=1}^n) P(V_j \,|\,  \mathbf{V}^{ij} \cup \{g_l(C)\}_{l=1}^L \cup \{\theta_m(C)\}_{m=1}^n).
    		\end{flalign}
    		That is,
    		$$V_i \independent V_j \,|\,  \mathbf{V}^{ij} \cup \{g_l(C)\}_{l=1}^L \cup \{\theta_m(C)\}_{m=1}^n.$$
    		Again, by the Faithfulness assumption on $G^{aug}$, this implies that $V_i$ and $V_j$ are not adjacent in $G^{aug}$ and hence are not adjacent in $G$.
    	\end{itemize}
    	
    	Therefore, $V_i$ are $V_j$ are not adjacent in $G$ if and only if they are conditionally independent given some subset of $\{V_k\,|\,k\neq i, k\neq j\}\cup \{C\}$.

    \end{proof}

	\section*{Appendix B. Proof of Proposition 1} \label{Appendix: Propo 1}

	\begin{proof}
			\begin{flalign} \nonumber
			&\tilde{\mu}_{\underline{Y},X|{C=c_n}} \triangleq \mathbb{E}_{(Y,X)\sim \tilde{P}(\underline{Y}, X \,|\,C=c_n) }[\phi(Y) \otimes \phi(X)]  \\ \nonumber
			=& \tilde{\mathcal{C}}_{(Y,X),C} \mathcal{C}_{C,C}^{-1}\phi(c_n) \\ \nonumber
			= &\mathbb{E}_{(Y,X,C)\sim \tilde{P}(\underline{Y}, X,C) } [\phi(Y)\otimes \phi(X) \otimes \phi(C)] \mathcal{C}_{C,C}^{-1} \phi(c_n) \\ \nonumber
			= &\mathbb{E}_{(X,C)\sim \tilde{P}(X,C) } \{ \mathbb{E}_{Y\sim P(Y|X,C)} [\phi(Y)] \otimes  \phi(X) \otimes \phi(C)\} \mathcal{C}_{C,C}^{-1} \phi(c_n) \\ \nonumber
			= &\mathcal{C}_{Y,(X,C)} \mathcal{C}_{(X,C),(X,C)}^{-1} \cdot \\
			& \mathbb{E}_{X\sim {P}(X) } \mathbb{E}_{C\sim P(C) } \{ [ \phi(X) \otimes \phi(C) ] \otimes  \phi(X) \otimes \phi(C)\} \mathcal{C}_{C,C}^{-1} \phi(c_n)
			\label{Eq_pf0}
			\end{flalign}
		Furthermore, 
			\begin{flalign} \nonumber
			& \mathbb{E}_{X\sim {P}(X) } \mathbb{E}_{C\sim P(C) } \{ [ \phi(X) \otimes \phi(C) ] \otimes  \phi(X) \otimes \phi(C)\} \mathcal{C}_{C,C}^{-1} \phi(c_n) \\ \nonumber
			= &\mathbb{E}_{X\sim {P}(X) } \mathbb{E}_{C\sim P(C) } \{ [ \phi(X) \otimes \phi(C) ] \otimes  [\phi(X) \cdot \phi^\intercal (C)  \mathcal{C}_{C,C}^{-1} \phi(c_n)]\}   \\ \label{Eq_p1}
			= &\mathbb{E}_{X\sim {P}(X) } \mathbb{E}_{C\sim P(C) } \{  \phi(X) \otimes [\phi(C) \cdot  (\phi(X) \phi^\intercal (C)  \mathcal{C}_{C,C}^{-1} \phi(c_n))^\intercal ] \}   \\ \nonumber
			= &\mathbb{E}_{X\sim {P}(X) } \mathbb{E}_{C\sim P(C) } \{  \phi(X) \otimes [\phi(C) \cdot  \phi^\intercal (c_n) \mathcal{C}_{C,C}^{-1} \phi (C) \phi^\intercal(X)]\} \\ \label{Eq_p2}
			= &\mathbb{E}_{X\sim {P}(X) } \mathbb{E}_{C\sim P(C) } \{  \phi(X) \otimes [\mathcal{C}_{C,C}  \mathcal{C}_{C,C}^{-1} \phi (c_n) \phi^\intercal(X)]\} \\ \nonumber
			= &\mathbb{E}_{X\sim {P}(X) }  \{  \phi(X) \otimes \phi (c_n) \otimes \phi(X)\} 
			\end{flalign}
		where (\ref{Eq_p1}) holds because tensor product is associative, and (\ref{Eq_p2}) holds because $\phi^\intercal (c_n) \mathcal{C}_{C,C}^{-1} \phi (C)$ is a scaler, implying $\phi^\intercal (c_n) \mathcal{C}_{C,C}^{-1} \phi (C) = [\phi^\intercal (c_n) \mathcal{C}_{C,C}^{-1} \phi (C)]^\intercal = \phi^\intercal (C)  \mathcal{C}_{C,C}^{-1} \phi (c_n)$.

		Therefore, equation (\ref{Eq_pf0}) becomes
			\begin{equation*} \label{Eq:embedding_cond}
			\tilde{\mu}_{\underline{Y},X|{C=c_n}}  = \mathcal{C}_{Y,(X,C)} \mathcal{C}_{(X,C),(X,C)}^{-1}\mathbb{E}_{X\sim {P}(X) }  \{  \phi(X) \otimes \phi (c_n) \otimes \phi(X)\}. 
			\end{equation*}
		Let $\phi^{\otimes}(X,C) :=  \phi(X)\otimes \phi(C)$,  $\boldsymbol{\Phi}_\mathbf{y} := [\phi(y_1), ..., \phi(y_N)]$,  $\boldsymbol{\Phi}_\mathbf{x} := [\phi(x_1), ..., \phi(x_N)]$, 
		$\boldsymbol{\Phi}_\mathbf{x,c} := [\phi^{\otimes}(x_1,c_1), ..., \phi^{\otimes}(x_N,c_n)]$, $\boldsymbol{\Phi}_{\mathbf{x},c_n} := [\phi^{\otimes}(x_1,c_n), ..., \phi^{\otimes}(x_N,c_n)]$.  $\tilde{\mu}_{\underline{Y},X|{C=c_n}}$ can be estimated as
		\begin{flalign} \nonumber
		\hat{\tilde{\mu}}_{\underline{Y},X|{C=c_n}} &= \frac{1}{N}\boldsymbol{\Phi}_\mathbf{y} \boldsymbol{\Phi}_\mathbf{x,c}^\intercal  ( \frac{1}{N} \boldsymbol{\Phi}_\mathbf{x,c} \boldsymbol{\Phi}_\mathbf{x,c}^\intercal + \lambda I)^{-1}  (\frac{1}{n}\boldsymbol{\Phi}_{\mathbf{x},c_n} \boldsymbol{\Phi}_{\mathbf{x}}^\intercal) \\ \nonumber
		&= \frac{1}{N} \boldsymbol{\Phi}_\mathbf{y} (\mathbf{K}_\mathbf{x}\odot \mathbf{K}_\mathbf{c} + \lambda I)^{-1} \textrm{diag}(\mathbf{k}_{\mathbf{c},c_n}) \mathbf{K}_\mathbf{x} \boldsymbol{\Phi}_{\mathbf{x}}^\intercal.
		\end{flalign}
	\end{proof}

	\section*{Appendix C.  Proof of Theorem 2}  \label{Appendix: Theorem 2}

    \begin{proof} 
      Let $V_i$ and $V_j$ be an unoriented pair of adjacent variables.
      
   	  If condition (1) is satisfied, then there are two possibilities: 
   	   \begin{itemize}[itemsep=0.2pt,topsep=0.2pt]
   	   	\item $V_j$ is the collider, and there exists another variable $V_k$ which influences $V_j$ and does not have an edge with $V_i$, so that $V_i \rightarrow V_j \leftarrow V_k$. In such a case, $V_i \independent V_k |\mathbf{S}$, with $\mathbf{S} \subseteq \mathbf{V}$ and $\mathbf{S} \land V_j = \ \emptyset$.
   	   	\item $V_i$ is the collider, and there exists another variable $V_k$ which influences $V_i$ and does not have an edge with $V_j$, so that $V_j \rightarrow V_i \leftarrow V_k$. In such a case, $V_j \independent V_k |\mathbf{S}$, with $\mathbf{S} \subseteq \mathbf{V}$ and $\mathbf{S} \land V_i = \ \emptyset$.
   	   \end{itemize} 
      Thus, we can identify the direction between $V_i$ and $V_j$ depending on whether $V_i \independent V_k |\mathbf{S}$ or $V_j \independent V_{k'} |\mathbf{S'}$.
       
      If condition (2) is satisfied, and there is only one change influences $V_i$ or $V_j$. Then we have the following two cases:
       \begin{itemize}[itemsep=0.2pt,topsep=0.2pt]
       	   \item The change influences $V_j$, i.e., $V_i - V_j \leftarrow C$:
       	  \begin{itemize}[itemsep=0.2pt,topsep=0.2pt]
       		 \item if $V_i \independent C |\mathbf{S}$, with $\mathbf{S} \subseteq \mathbf{V}$ and $\mathbf{S} \land V_j = \ \emptyset$, orient $V_i \rightarrow V_j$;
       		 \item  if $V_i \independent C |\mathbf{S}$, with $\mathbf{S} \subseteq \mathbf{V}$ and $V_j \subseteq \mathbf{S}$, orient $V_i \leftarrow V_j$.
       	  \end{itemize} 
       	   \item The change influences $V_i$, i.e., $V_j - V_i \leftarrow C$:
       	 \begin{itemize}[itemsep=0.2pt,topsep=0.2pt]
       		\item  if $V_j \independent C |\mathbf{S}$, with $\mathbf{S} \subseteq \mathbf{V}$ and $\mathbf{S} \land V_i = \ \emptyset$, orient $V_j \rightarrow V_i$;
       		\item  if $V_j \independent C |\mathbf{S}$, with $\mathbf{S} \subseteq \mathbf{V}$ and $V_i \subseteq \mathbf{S}$, orient $V_j \leftarrow V_i$.
         \end{itemize} 
       \end{itemize}     
   
      If condition (2) is satisfied, and there are independent changes to both $V_i$ and $V_j$. Then we have the following two cases:
      \begin{itemize}[itemsep=0.2pt,topsep=0.2pt]
           \item  if $P(V_i) \independent P(V_j|V_i,\mathbf{Z})$ and $P(V_j) \dependent P(V_i|V_j,\mathbf{Z})$, where $\mathbf{Z}$ is the deconfounding set of $(V_i,V_j)$, then orient $V_i \rightarrow V_j$;
          \item  if $P(V_j) \independent P(V_i|V_j,\mathbf{Z})$ and $P(V_i) \dependent P(V_j|V_i,\mathbf{Z})$, where $\mathbf{Z}$ is the deconfounding set of $(V_i,V_j)$, then orient  $V_j \rightarrow V_i$.
      \end{itemize} 
 
       If condition (3) is satisfied, then we have the following two cases:
       \begin{itemize}[itemsep=0.2pt,topsep=0.2pt]
       	\item There is an edge incident to $V_j$ but not to $V_i$, i.e., $V_i - V_j \leftarrow V_l$: 
       	   \begin{itemize}[itemsep=0.2pt,topsep=0.2pt]
       		\item if $V_i \independent V_l |\mathbf{S}$, with $\mathbf{S} \subseteq \mathbf{V}$ and $\mathbf{S} \land V_j = \ \emptyset$, orient $V_i \rightarrow V_j$;
       		\item  if $V_i \independent V_l |\mathbf{S}$, with $\mathbf{S} \subseteq \mathbf{V}$ and $V_j \subseteq \mathbf{S}$, orient $V_i \leftarrow V_j$.
       	   \end{itemize} 
       	\item There is an edge incident to $V_i$ but not to $V_j$, i.e., $V_j - V_i \leftarrow V_l$:
       	   \begin{itemize}[itemsep=0.2pt,topsep=0.2pt]
       	  	\item  if $V_j \independent V_l |\mathbf{S}$, with $\mathbf{S} \subseteq \mathbf{V}$ and $\mathbf{S} \land V_i = \ \emptyset$, orient $V_j \rightarrow V_i$;
       	  	\item  if $V_j \independent V_l|\mathbf{S}$, with $\mathbf{S} \subseteq \mathbf{V}$ and $V_i \subseteq \mathbf{S}$, orient $V_j \leftarrow V_i$.
       	  \end{itemize} 
       \end{itemize}
  
     Therefore, it is easy to see that the direction between $V_i$ and $V_j$ is identifiable, if at least one of the conditions is satisfied. 
   	\end{proof}

	\bibliography{term_fin1}
\end{document}